\documentclass{article}

     \usepackage[final, nonatbib]{neurips_2019}

\usepackage[utf8]{inputenc} 
\usepackage[T1]{fontenc}    
\usepackage{hyperref}       
\usepackage{url}            
\usepackage{booktabs}       
\usepackage{amsfonts}       
\usepackage{nicefrac}       
\usepackage{microtype}      
\usepackage{wrapfig}

\title{State Aggregation Learning from Markov Transition Data}

\author{
	Yaqi Duan \\
	Princeton University \\
	\texttt{yaqid@princeton.edu} \\
	\And
	Zheng Tracy Ke \\
	Harvard University \\
	\texttt{zke@fas.harvard.edu} \\
	\And
	Mengdi Wang \\
	Princeton University \\
	\texttt{mengdiw@princeton.edu} \\
}

\usepackage{amsfonts}
\usepackage{wrapfig,amsgen,amsmath,amsthm,amstext,amsbsy,amsopn,amssymb,soul,subfigure}
\usepackage[dvips]{graphicx}
\usepackage{url}   
\usepackage{graphicx}
\usepackage[dvipsnames,table,dvipsnames*, svgnames*]{xcolor}
\usepackage{comment}
\usepackage{blkarray}
\usepackage{algorithm}
\usepackage{enumitem}

\definecolor{blue}{RGB}{000,000,200}
\definecolor{green}{RGB}{000,150,100}
\definecolor{purple}{RGB}{120,000,250}

\def\red#1{{\color{red}{#1}}}

\newtheorem{Definition}{Definition}
\newtheorem{Theorem}{Theorem}
\newtheorem{Assumption}{Assumption}
{
	\theoremstyle{definition}

}

\newtheorem{Proposition}{Proposition}

\newcommand{\G}{{\mathbf{G}}}

\newcommand{\bbP}{{\mathbb{P}}}

\newcommand{\V}{{\rm V}}

\newcommand{\bOmega}{\boldsymbol{\Omega}}

\newcommand{\diag}{{\rm diag}}

\def\mybox#1{\vskip1mm \begin{center} \bf \red
		\hspace{.0\textwidth}\vbox{\hrule\hbox{\vrule\kern6pt
				\parbox{.95\textwidth}{\kern6pt#1\vskip6pt}\kern6pt\vrule}\hrule}
	\end{center} \vskip-5mm}

\def\bP{\mathbf{P}}
\def\bF{\mathbf{F}}
\def\bG{\mathbf{G}}
\def\bQ{\mathbf{Q}}
\def\bX{\mathbf{X}}
\def\bZ{\mathbf{Z}}	
\def\bU{\mathbf{U}}	
\def\bV{\mathbf{V}}	
\def\bA{\mathbf{A}}	
\def\bB{\mathbf{B}}	

\def\bT{\mathbf{T}}	
\def\bW{\mathbf{W}}	
\def\bI{\mathbf{I}}	
\def\bR{\mathbf{R}}

\def\bSigma{\boldsymbol{\Sigma}}	
\def\mu{\pi}

\def\V{\mathbf{V}}

\usepackage{color,amsmath}
\newcommand{\beq}{\begin{equation}}
\newcommand{\eeq}{\end{equation}}
\newcommand{\bN}{{\bf N}}
\newcommand{\bH}{{\bf H}}
\newcommand{\bD}{{\bf D}}
\newcommand{\bL}{{\bf L}}
\newcommand{\bh}{{\bf h}}
\newcommand{\hbh}{\widehat{\bf h}}
\newcommand{\hbb}{\widehat{\bf b}}
\newcommand{\bd}{{\bf d}}

\makeatletter
\newcommand{\xRightarrow}[2][]{\ext@arrow 0359\Rightarrowfill@{#1}{#2}}
\makeatother

\newcommand{\bu}{{\bf u}}

\newcommand{\bpi}{{\boldsymbol{\pi}}}
\newcommand{\bm}{{\bf m}}

\newcommand{\bb}{{\bf b}}

\newcommand{\bw}{{\bf w}}

\newcommand{\bvarOmega}{{\boldsymbol{\varOmega}}}
\newcommand{\be}{{\bf e}}
\newcommand{\bg}{{\bf g}}
\newcommand{\bY}{{\bf Y}}
\newcommand{\bxi}{{\boldsymbol{\xi}}}
\newcommand{\bXi}{{\bf \Xi}}
\newcommand{\bvarXi}{{\boldsymbol{\varXi}}}
\newcommand{\bvarG}{{\boldsymbol{G}}}
\newcommand{\bvarH}{{\boldsymbol{H}}}
\newcommand{\bx}{{\bf x}}

\newcommand{\bvarSigma}{{\boldsymbol{\varSigma}}}
\newcommand{\bmu}{\boldsymbol{\eta}}
\newcommand{\bK}{{\bf K}}
\newcommand{\bS}{{\bf S}}
\newcommand{\bupsilon}{{\boldsymbol{\upsilon}}}
\newcommand{\bl}{{\bf l}}
\newcommand{\dist}{{\rm dist}}
\newtheorem{remark}{Remark}
\newtheorem{lemma}{Lemma}[section]
\newtheorem{theorem}{Theorem}[section]
\newtheorem{corollary}{Corollary}[section]

\addtocontents{toc}{\protect\setcounter{tocdepth}{-1}}

\begin{document}
	
\maketitle
	

\vspace{-0.5cm}

\begin{abstract}
\vspace{-0.2cm}
State aggregation is a popular model reduction method rooted in optimal control. It reduces the complexity of engineering systems by mapping the system's states into a small number of meta-states. The choice of aggregation map often depends on the data analysts' knowledge and is largely ad hoc. In this paper, we propose a tractable algorithm that estimates the probabilistic aggregation map from the system's trajectory. We adopt a soft-aggregation model, where each meta-state has a signature raw state, called an {\it anchor state}. This model includes several common state aggregation models as special cases. Our proposed method is a simple two-step algorithm: The first step is spectral decomposition of empirical transition matrix, and the second step conducts a linear transformation of singular vectors to find their approximate convex hull. It outputs the aggregation distributions and disaggregation distributions for each meta-state in explicit forms, which are not obtainable by classical spectral methods. On the theoretical side, we prove sharp error bounds for estimating the aggregation and disaggregation distributions and for identifying anchor states. The analysis relies on a new entry-wise deviation bound for singular vectors of the empirical transition matrix of a Markov process, which is of independent interest and cannot be deduced from existing literature. The application of our method to Manhattan traffic data successfully generates a data-driven state aggregation map with nice interpretations.  
\end{abstract}

\vspace{-0.25cm}

\begin{center}
	\begin{tabular}{p{0.9\linewidth}}
	\end{tabular}	
\end{center}

\section{Introduction}
\vspace{-0.15cm}
State aggregation is a long-existing approach for model reduction of complicated systems. It is widely used as a heuristic to reduce the complexity of control systems and reinforcement learning (RL). The earliest idea of state aggregation is to aggregate ``similar'' states into a small number of  subsets through a partition map.  However, the partition is often handpicked by practitioners based on domain-specific knowledge or exact knowledge about the dynamical system \cite{rogers1991aggregation,bertsekas1995neuro}. Alternatively, the partition can be chosen via discretization of the state space in accordance with some priorly known similarity metric or feature functions \cite{tsitsiklis1996feature}. 
Prior knowledge of the dynamical system is often required in order to handpick the aggregation without deteriorating its performance. There lacks a principled approach to find the best state aggregation structure from data. 

In this paper, we propose a model-based approach to learning probabilistic aggregation structure. We adopt the {\it soft state aggregation} model, a flexible model for Markov systems. It allows one to represent each state using a membership distribution over latent variables (see Section~\ref{sec:model} for details). Such models have been used for modeling large Markov decision processes, where the membership can be used as state features to significantly reduce its complexity \cite{singh1995reinforcement,yang2019sample}. When the membership distributions are degenerate, it reduces to the more conventional {\it hard state aggregation} model, and so our method is also applicable to finding a hard partition of the state space. 

The soft aggregation model is parameterized by $p$ aggregation distributions and $r$ disaggregation distributions, where $p$ is the total number of states in a Markov chain and $r$ is the number of (latent) meta-states. Each aggregation distribution contains the probabilities of transiting from one raw state to  different meta-states, and each disaggregation distribution contains the probabilities of transiting from one meta-state to different raw states. Our goal is to use sample trajectories of a Markov process to estimate these aggregation/disaggregation distributions. The obtained aggregation/disaggregation distributions can be used to estimate the transition kernel, sample from the Markov process, and plug into downstream tasks in optimal control and reinforcement learning (see Section~\ref{sec:application} for an example). 
In the special case when the system admits a hard aggregation, these distributions naturally produce a partition map of states.

Our method is a two-step algorithm. The first step is the same as the vanilla spectral method, where we extract the first $r$ left and right singular vectors of the empirical transition matrix. The second step is a novel linear transformation of singular vectors. The rationale of the second step is as follows: Although the left (right) singular vectors are not valid estimates of the aggregation (disaggregation) distributions, their linear span is a valid estimate of the linear span of aggregation (disaggregation) distributions. Consequently, the left (right) singular vectors differ from the targeted aggregation (disaggregation) distributions only by a linear transformation. We estimate this linear transformation by leveraging a geometrical structure associated with singular vectors. Our method requires no prior knowledge of meta-states and provides a data-driven approach to learning the aggregation map.

{\bf Our contributions.} 

\vspace{-0.1cm}

\begin{enumerate}[leftmargin=0pt,itemindent=1.25em,itemsep=-0.01cm]
\vspace{-.5em}
\item We introduce a notion of ``anchor state'' for the identifiability of soft state aggregation. It creates an analog to the notions of ``separable feature'' in nonnegative matrix factorization \cite{donoho2004does} and ``anchor word'' in topic modeling \cite{Ge}. The introduction of ``anchor states'' not only ensures model identifiability but also greatly improves the interpretability of meta-states (see Section~\ref{sec:model} and Section~\ref{sec:application}). Interestingly, hard-aggregation indeed assumes {\it all} states are anchor states. Our framework instead assumes there exists {\it one} anchor state for each meta-state. 
 
\item We propose an efficient method for estimating the aggregation/disaggregation distributions of a soft state aggregation model from Markov transition data. In contrast, classical spectral methods are not able to estimate these distributions directly.

\item We prove statistical error bounds for the total variation divergence between the estimated aggregation/disaggregation distributions and the ground truth. The estimation errors depend on the size of state space, number of meta-states, and mixing time of the process. We also prove a sample complexity bound for accurately recovering all anchor states. To our best knowledge, this is the first result of statistical guarantees for soft state aggregation learning. 

\item At the core of our analysis is an entry-wise large deviation bound for singular vectors of the empirical transition matrix of a Markov process. This connects to the recent interests of entry-wise analysis of empirical eigenvectors \cite{abbe2017entrywise, koltchinskii2016asymptotics, koltchinskii2016perturbation, zhong2018near, chen2017spectral,eldridge2017unperturbed}. Such analysis is known to be challenging, and techniques that work for one type of random matrices often do not work for another. Unfortunately, our desired results cannot be deduced from any existing literature, and we have to derive everything from scratch (see Section~\ref{sec:theory}). Our large-deviation bound provides a convenient technical tool for the analysis of spectral methods on Markov data, and is of independent interest.

\item We applied our method to a Manhattan taxi-trip dataset, with interesting discoveries. The estimated state aggregation model extracts meaningful traffic modes, and the output {\it anchor regions} capture popular landmarks of the Manhattan city, such as Times square and WTC-Exchange. A plug-in of the aggregation map in a reinforcement learning (RL) experiment, taxi-driving policy optimization, significantly improves the performance. 
These validate that our method is practically useful. 
\end{enumerate}

{\bf Connection to literature.}
While classical spectral methods have been used for aggregating states in Markov processes \cite{weber2005robust,zhang2018spectral}, these methods do not directly estimate a state aggregation model. It was shown in \cite{zhang2018spectral} that spectral decomposition can reliably recover the principal subspace of a Markov transition kernel. Unfortunately, singular vectors themselves are {\it not} valid estimates of the aggregation/disaggregation distributions: the population quantity that singular vectors are trying to estimate is strictly different from the targeted aggregation/disaggregation distributions (see Section~\ref{sec:method}). 

Our method is inspired by the connection between soft state aggregation, nonnegative matrix factorization (NMF) \cite{lee1999learning,gillis2014and}, and topic modeling \cite{blei2003latent}. 
Our algorithm is connected to the spectral method in \cite{Topic-SCORE} for topic modeling. The method in \cite{Topic-SCORE} is a general approach about using spectral decomposition for nonnegative matrix factorization. 
Whether or not it can be adapted to state aggregation learning and how accurately it estimates the soft-aggregation model was unknown. In particular, \cite{Topic-SCORE} worked on a topic model, where the data matrix has column-wise independence and their analysis heavily relies on this property. Unfortunately, our data matrix has column-wise dependence, so we are unable to use their techniques. We build our analysis from ground up.  

There are recent works on statistical guarantees of learning a Markov model \cite{falahatgar2016learning, NIPS2018_7345, li2018estimation, zhang2018spectral, pmlr-v98-wolfer19a}. They target on estimating the transition matrix, but our focus is to estimate the aggregation/disaggregation distributions. Given an estimator of the transition matrix, how to obtain the aggregation/disaggregation distributions is non-trivial (for example, simply performing a vanilla PCA on the estimator doesn't work). To our best knowledge, our result is the first statistical guarantee for estimating the aggregation/disaggregation distributions. We can also use our estimator of aggregation/disaggregation distributions to form an estimator of the transition matrix, and it can achieve the known minimax rate for a range of parameter settings. See Section~\ref{sec:theory} for more discussions.  



\section{Soft State Aggregation Model with Anchor States} \label{sec:model}
We say that a Markov chain $X_0,X_1,\ldots, X_n$ admits a {\it soft state aggregation} with $r$ meta-states, if there exist random variables $Z_0, Z_1,\ldots,Z_{n-1}\in\{1,2,\ldots,r\}$ such that 
{
\vspace{-0.1cm}
\begin{equation} \label{mod0}
\bbP(X_{t+1} \mid X_t ) = \sum_{k=1}^r \mathbb{P} (Z_t=k\mid X_t )\cdot \mathbb{P} (X_{t+1}\mid Z_t=k),
\end{equation}}
for all $t$ with probability $1$. Here, $\mathbb{P} (Z_t\mid X_{t})$ and $\mathbb{P} (X_{t+1}\mid Z_t)$ are independent of $t$ and referred to as the {\it aggregation distributions} and {\it disaggregation distributions}, respectively. 
The soft state aggregation model has been discussed in various literatures (e.g., \cite{singh1995reinforcement, zhang2018spectral}), where $r$ is presumably much smaller than $p$. See \cite[Section 6.3.7]{bertsekas1995dynamic} for a textbook review.
This decomposition means that one can map the states into meta-states while preserving the system's dynamics. In the special case where each aggregation distribution 
is degenerate (we say a discrete distribution is {\it degenerate} if only one outcome is possible), it reduces to the hard state aggregation model or lumpable Markov model. 

The soft state aggregation model has a matrix form.  
Let $\bP\in\mathbb{R}^{p\times p}$ be the transition matrix, where $P_{ij}=\bbP(X_{t+1}=j \mid X_t=i )$. We introduce $\bU\in\mathbb{R}^{p\times r}$ and $\bV\in\mathbb{R}^{p\times r}$, where $U_{ik}=\mathbb{P} (Z_t=k\mid X_t =i)$ and $V_{jk}=\mathbb{P}(X_{t+1}=j\mid Z_t=k)$. 
Each row of $\bU$ is an aggregation distribution, and each column of $\bV$ is a disaggregation distribution. Then,  \eqref{mod0} is equivalent to (${\bf 1}_s$: the vector of $1$'s)
{\setlength\abovedisplayskip{0.7pt}
\setlength\belowdisplayskip{10pt}
\vspace{0.15cm}
\beq \label{P-decompose}
\bP = \bU \bV^\top,  \qquad \text{where $\bU{\bf 1}_r = {\bf1}_p$ and $\bV^\top{\bf1}_p = {\bf1}_r$}. 
\eeq}Here, $\bU$ and $\bV$ are not identifiable, unless with additional conditions. We assume that each meta-state has a signature raw state, defined either through the aggregation process or the disaggregation process.

\begin{Definition}[Anchor State] \label{def:anchor}
A state $i$ is called an ``aggregation anchor state'' of the meta-state $k$ if $U_{ik}=1$ and $U_{is}=0$ for all $s\neq k$. A state $j$ is called a ``disaggregation anchor state'' of the meta-state $k$, if $V_{jk}>0$ and $V_{js}=0$ for all $s\neq k$. 
\end{Definition} 

An aggregation anchor state transits to only one meta-state, and a disaggregation anchor state can be transited from only one meta-state. Since \eqref{P-decompose} is indeed a nonnegative matrix factorization (NMF), the definition of anchor states are natural analogs of ``separable features'' in NMF \cite{donoho2004does}. They are also natural analogs of ``pure documents'' and ``anchor words'' in topic modeling. We note that in a hard state aggregation model, every state is an anchor state by default. 

{Throughout this paper, we take the following assumption: 
\begin{Assumption} There exists at least one disaggregation anchor state for each meta-state. \end{Assumption}} By well-known results in NMF \cite{donoho2004does}, this assumption guarantees that $\bU$ and $\bV$ are uniquely defined by \eqref{P-decompose}, provided that $\bP$ has a rank of $r$. Our results can be extended to the case where each meta-state has an aggregation anchor state (see the remark in Section~\ref{sec:method}). For simplicity, from now on, we call a disaggregation anchor state an {\it anchor state} for short. 

{\it The introduction of anchor states not only guarantees identifiability but also enhances interpretability of the model.} 
This is demonstrated in an application to New York city taxi traffic data. We model the taxi traffic by a finite-state Markov chain, where each state is a pick-up/drop-off location in the city. Figure~\ref{fig:interpretability} illustrates the estimated soft-state aggregation model (see Section~\ref{sec:application} for details). The estimated anchor states coincide with notable landmarks in Manhattan, such as Times square area, the museum area on Park avenue, etc. Hence, each meta-state (whose disaggregation distribution is plotted via a heat map over Manhattan in (c)) can be nicely interpreted as a representative traffic mode with exclusive destinations (e.g., the traffic to Times square, the traffic to WTC Exchange, the traffic to museum park, etc.). In contrast, if we don't explore anchor states but simply use PCA to conduct state aggregation,  the obtained meta-states have no clear association with notable landmarks, so are hard to interpret. 
The interpretability of our model also translates to better performance in downstream tasks in reinforcement learning (see Section~\ref{sec:application}). 

\begin{figure*}[ht]
	\centering
\includegraphics[height=1.78in]{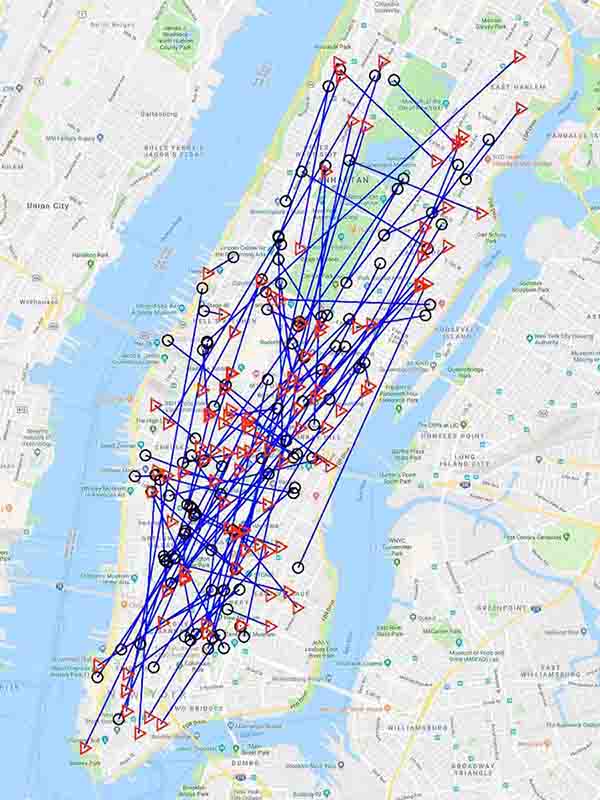}
\includegraphics[height=1.78in]{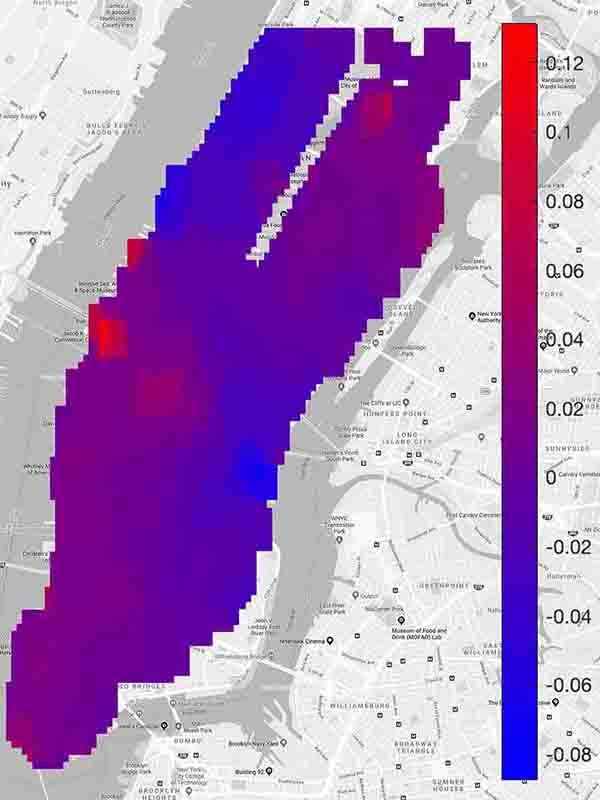}
\includegraphics[height=1.78in]{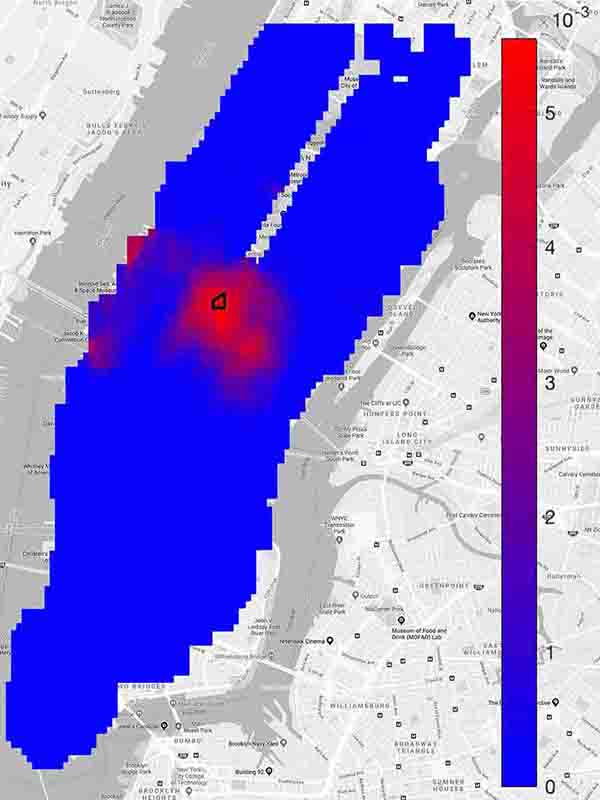}
		\includegraphics[height=1.78in]{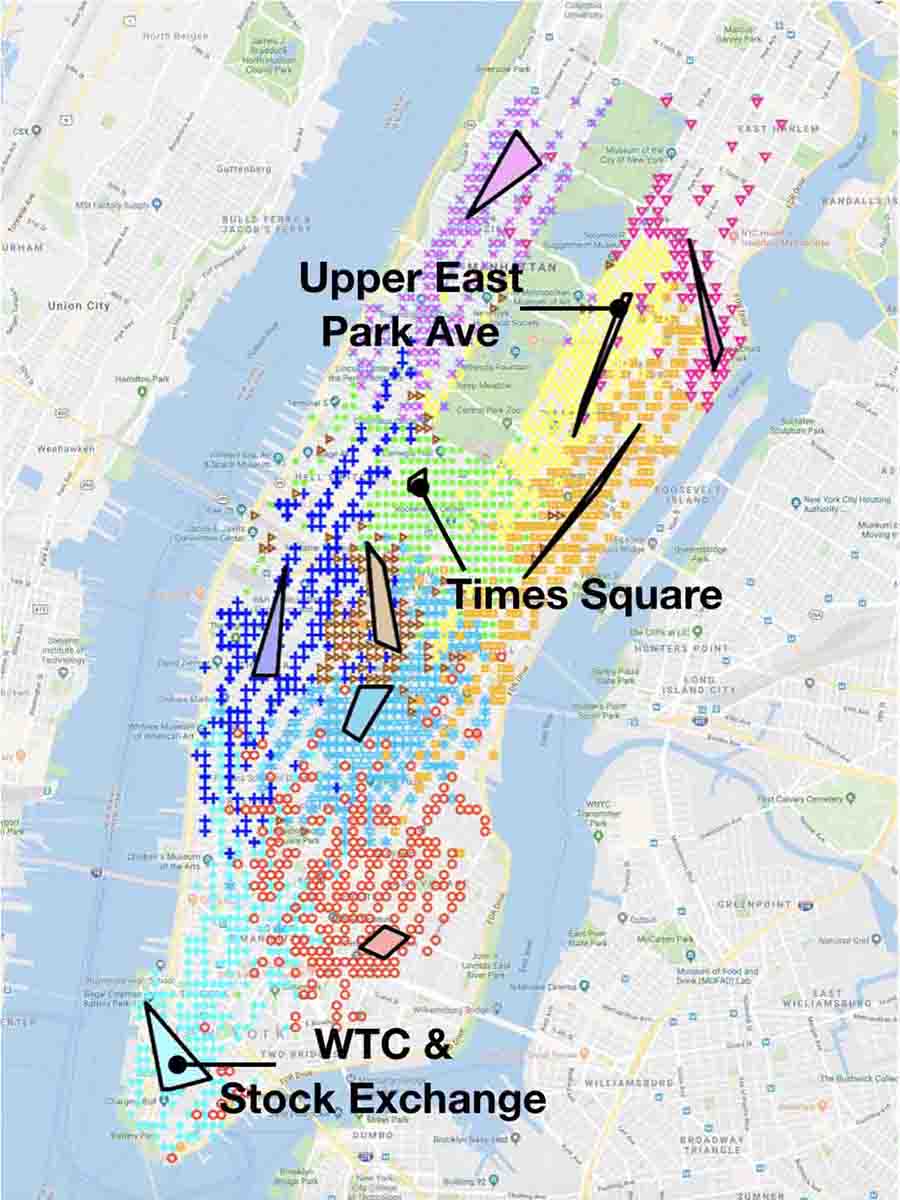}
	
	\vspace{-0.2cm}
	\caption{\footnotesize {\bf Soft state aggregation learned from NYC taxi data.} Left to right: (a) Illustration of 100 taxi trips (O: pick-up location, \textcolor{red}{$\triangle$}: drop-off location). (b) A principal component of $\mathbf{P}$ (heat map), lacking interpretability. (c) Disaggregation distribution of $\mathbf{P}$ corresponding to the Times square anchor region. (d) Ten anchor regions identified by Alg.\ 1, coinciding with landmarks of NYC.} \label{fig:interpretability}
	\vspace{-15pt}
\end{figure*}

\section{An Unsupervised State Aggregation Algorithm} \label{sec:method}
Given a Markov trajectory $\{X_0,X_1,\ldots,X_n\}$ from a state-transition system, let $\bN\in\mathbb{R}^{p\times p}$ be the matrix consisting of empirical state-transition counts, i.e.,
$\bN_{ij} = \sum_{t=0}^{n-1} \mathbf1\{X_t = i, X_{t+1} = j\}$. Our algorithm takes as input the matrix ${\bf N}$ and the number of meta-states $r$, and it estimates (a) the disaggregation distributions $\bV$, (b) the aggregation distributions $\bU$, and (c) the anchor states. See Algorithm 1. Part (a) is the core of the algorithm, which we explain below. 

{\it Insight 1: Disaggregation distributions are linear transformations of right singular vectors.}

We consider an oracle case, where the transition matrix $\bP$ is given and we hope to retrieve $\bV$ from $\bP$. Let $\bH=[\bh_1,\ldots,\bh_r]\in\mathbb{R}^{p\times r}$ be the matrix containing first $r$ right singular vectors of $\bP$. Let $\mathrm{Span}(\cdot)$ denote the column space of a matrix. By linear algebra, $\mathrm{Span}(\bH)=\mathrm{Span}(\bP^{\top})=\mathrm{Span}(\bV\bU^{\top}) = \mathrm{Span}(\bV)$. Hence, the columns of $\bH$ and the columns of $\bV$  are two different bases of the same $r$-dimensional subspace.
It follows immediately that there exists $\bL\in\mathbb{R}^{r\times r}$ such that $
\bH = \bV\bL$.
On the one hand, this indicates that singular vectors are not valid estimates of disaggregation distributions, as each singular vector $\bh_k$ is a linear combination of multiple disaggregation distributions.  On the other hand, it suggests a promising two-step procedure for recovering $\bV$ from $\bP$: (i) obtain the right singular vectors $\bH$, (ii) identify the matrix $\bL$ and retrieve $\bV=\bH\bL^{-1}$. 

{\it Insight 2: The linear transformation of $\bL$ is estimable given the existence of anchor states.}

The estimation of $\bL$ hinges on a geometrical structure induced by the anchor state assumption \cite{donoho2004does}: Let ${\cal C}$ be a simplicial cone with $r$ supporting rays, where the directions of the supporting rays are specified by the $r$ rows of the matrix $\bL$. If $j$ is an anchor state, then the $j$-th row of $\bH$ lies on one supporting ray of this simplicial cone. If $j$ is not an anchor state, then the $j$-th row of $\bH$ lies in the interior of the simplicial cone. See the left panel of Figure~\ref{fig:geometry} for illustration with $r=3$. 
Once we identify the $r$ supporting rays of this simplicial cone, we immediately obtain the desired matrix $\bL$.

{\it Insight 3: Normalization on eigenvectors is the key to estimation of $\bL$ under noise corruption.}

In the real case where $\bN$, instead of $\bP$, is given, we can only obtain a noisy version of $\bH$. 
With noise corruption, to estimate supporting rays of a simplicial cone is very challenging. \cite{Topic-SCORE} discovered that a particular row-wise normalization on $\bH$ manages to ``project'' the simplicial cone to a simplex with $r$ vertices. Then, for all anchor states of one meta-state, their corresponding rows collapse to one single point in the noiseless case (and a tight cluster in the noisy case). The task reduces to estimating the vertices of a simplex, which is much easier to handle under noise corruption. This particular normalization is called SCORE \cite{SCORE}. It re-scales each row of $\bH$ by the first coordinate of this row. After re-scaling, the first coordinate is always $1$, so it is eliminated; the normalized rows then have $(r-1)$ coordinates. See the right panel of Figure~\ref{fig:geometry}. Once we identify the $r$ vertices of this simplex, we can use them to re-construct $\bL$ in a closed form \cite{Topic-SCORE}.

\begin{figure}[t]
	\centering
	\includegraphics[width=.75\linewidth]{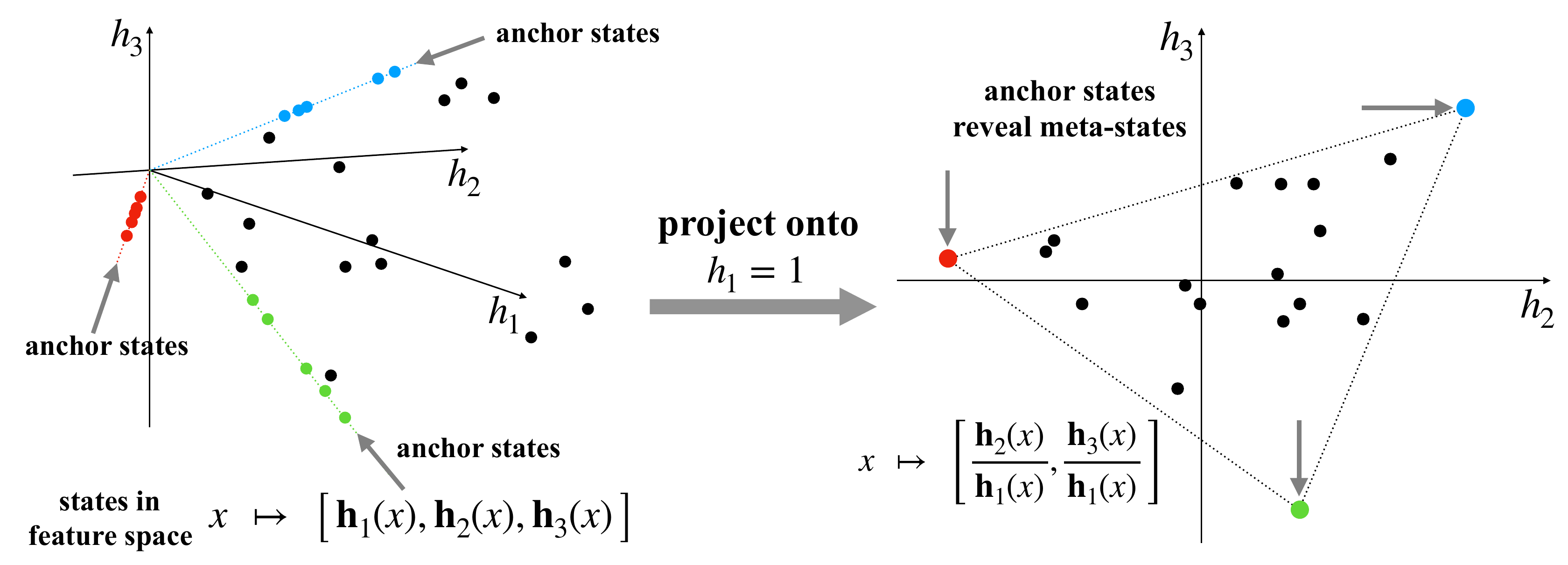}
	\caption{\footnotesize {\bf Geometrical structure of anchor states}. Left: Each dot is a row of the matrix $\bH=[\bh_1,\bh_2,\bh_3]$. The data points are contained in a simplicial cone with three supporting rays. Right: Each dot is a re-scaled row of $\bH$ by SCORE, where the first coordinate is dropped and so every row is in a plane. The data points are contained in a simplex with three vertices. } \label{fig:geometry}
	\vspace{-1em}
\end{figure}

These insights have casted estimation of $\bV$ to a simplex finding problem: Given data $\widehat{\bd}_1,\ldots,\widehat{\bd}_p\in\mathbb{R}^{r-1}$, suppose they are noisy observations of non-random vectors $\bd_1,\ldots,\bd_p\in\mathbb{R}^{r-1}$ (in our case, each $\bd_j$ is a row of the matrix $\bH$ after SCORE normalization), where these non-random vectors are contained in a simplex with $r$ vertices $\bb_1,\ldots,\bb_r$ with at least one $\bd_j$ located on each vertex. We aim to estimate the vertices $\bb_1,\ldots,\bb_r$. This is a well-studied problem in the literature, also known as {\it linear unmixing} or {\it archetypal analysis}. There are a number of existing algorithms \cite{boardman1995mapping, araujo2001successive, nascimento2005vertex, chang2006new, mixed-SCORE}. For example, the successive projection algorithm \cite{araujo2001successive} has a complexity of $O(pr^4)$.

Algorithm 1 follows the insights above but is more sophisticated. Due to space limit, we relegate the detailed explanation to Appendix. It has a complexity of $O(p^3)$, where main cost is from SVD.




{\bf Remark} {\it (The case with aggregation anchor states)}. The anchor states here are the disaggregation anchor states in Defintion~\ref{def:anchor}. Instead, if each meta-state has an aggregation anchor state, there is a similar geometrical structure associated with the left singular vectors. We can modify our algorithm to first use left singular vectors to estimate $\bU$ and then estimate $\bV$ and anchor states.

\begin{algorithm}[!t]
	{\small 
	\caption{Learning the Soft State Aggregation Model.} \label{alg:main}
	{\bf Input}: empirical state-transition counts $\bN$, number of meta-states $r$, {anchor state threshold $\delta_0$}
	\vspace{-1pt}
	
	\begin{enumerate} \itemsep = -1pt
	\item Estimate the matrix of disaggregation distributions $\bV$. 
	\begin{enumerate} \itemsep = -1pt
		\item[(i)] Conduct SVD on $\widetilde{\bN} = \bN [\mathrm{diag}(\bN^\top{\bf 1}_p)]^{-1/2}\in\mathbb{R}^{p\times p}$, and let $\widehat{\bh}_1,...,\widehat{\bh}_r$ denote the first $r$ right singular vectors. Obtain a matrix $\widehat{\bD} = [\mathrm{diag}(\hbh_1)]^{-1}[\hbh_2,\ldots,\hbh_r]\in\mathbb{R}^{p\times (r-1)}$. 
	        \item[(ii)] Run an existing vertex finding algorithm to rows of $\widehat{\bD}$. Let $\hbb_1,\ldots,\hbb_r$ be the output vertices. (In our numerical experiments, we use the vertex hunting algorithm in \cite{mixed-SCORE}). 
	        		\item[(iii)] For $1\leq j\leq p$, compute 
		\vspace{-10pt}
		\[
		\widehat{\bf w}_j^*=\mathrm{argmin}_{{\bf q}\in \mathbb{R}^r} \Bigl\|\widehat{\bf d}_j-\sum_{k=1}^r q_k\widehat{\bf b}_k\Bigr\|^2 + \Bigl(1 - \sum_{k=1}^r q_k\Bigr)^2.
		\vspace{-6pt}
		\]
		Set the negative entries in $\widehat\bw_j^*$ to zero and renormalize it to have a unit $\ell^1$-norm. The resulting vector is denoted as $\widehat\bw_j$. Let $\widehat{\bf W}=[\widehat{\bf w}_1,\widehat{\bf w}_2,\ldots,\widehat{\bf w}_p]^{\top}\in\mathbb{R}^{p\times r}$. Obtain the matrix $[\mathrm{diag}(\widehat{\bh}_1)][\mathrm{diag}(\bN^{\top} \mathbf{1}_p)]^{1/2}\widehat{\bf W}$ and re-normalize each column of it to have a unit $\ell^1$-norm. The resulting matrix is $\widehat{\bf V}$. 
		\end{enumerate}
		\item Estimate matrix of aggregation distributions $\bU$. Let $\widehat{\bP}=[\mathrm{\diag}(\bN{\bf 1}_p)]^{-1}\bN$ be the empirical transition probability matrix. Estimate $\bU$ by
		\vspace{-3pt}
		$$\widehat{\bU} = \widehat{\bP} \widehat{\bV}(\widehat{\bV}^\top\widehat{\bV})^{-1}, 
		\vspace{-4pt}
		$$ 
		\item Estimate the set of anchor states. Let $\widehat{\bf w}_j$ be as in Step (iii). Let
		\vspace{-3pt}
		$$
		{\cal A} = \big\{1\leq j\leq p: \max_{1\leq k\leq r}\widehat{\bf w}_j(k)\geq 1-\delta_0\big\}. 
		\vspace{-14pt}
		$$
	\end{enumerate}
	{\bf Output}: estimates $\widehat{\bV}$ and $\widehat{\bU}$, set of anchor states ${\mathcal{A}}$
}
\end{algorithm}

\vspace{-1em}

\section{Main Statistical Results} \label{sec:theory}
Our main results are twofold. The first is row-wise large-deviation bounds for empirical singular vectors. The second is statistical guarantees of learning the soft state aggregation model. 
Throughout this section, suppose we observe a trajectory $\{X_0,X_1,\ldots,X_n\}$ from an ergodic Markov chain with $p$ states, where the transition matrix $\bP$ satisfies \eqref{P-decompose} with $r$ meta-states. Let $\bpi\in\mathbb{R}^p$ denote the stationary distribution. Define a mixing time
$\tau_* = \min \bigl\{ k \geq 1: \max_{1 \leq i \leq p} \| (\bP^k)_{i,\cdot} - \bpi^{\top} \|_1 \leq 1/2 \bigr\}$. 
We assume there are constants $c_1$, $C_1$, $\bar{C}_1$, $c_2$, $c_3$, $C_4 > 0$ such that the following conditions hold:
\vspace{-2pt}
\begin{itemize} \itemsep -2pt
\item[(a)] The stationary distribution $\bpi$ satisfies $c_1p^{-1} \leq \bpi_j \leq C_1p^{-1}$, for $j=1,2,\ldots,p$. 
\item[(b)] The stationary distribution on meta-states satisfies $\bigl(\bU^{\top}\bpi\bigr)_k \leq \bar{C}_1 r^{-1}$, for $k=1,2,\ldots,r$. 
\item[(c)] $\lambda_{\min}(\bU^{\top} \bU) \geq c_2 pr^{-1}$ and $
		\lambda_{\min}(\bV^{\top} \bV) \geq c_2 p^{-1}r$. 
\item[(d)] The first two singular values of $[\diag(\bpi)]\bP[\diag(\bpi)]^{-1/2}$ satisfy $\sigma_1 - \sigma_2 \geq c_3 p^{-\frac{1}{2}}$.
\item[(e)] The entries of the $r$-by-$r$ matrix 
	$\bU^{\top}\bP\bV$ satisfy $\frac{\max_{k,l}(\bU^{\top}\bP\bV)_{kl}}{\min_{k,l}(\bU^{\top}\bP\bV)_{kl}} \leq C_4$. 
\end{itemize}
Conditions (a)-(b) require that the Markov chain has a balanced number of visits to each state and to each meta-state when reaching stationarity. Such conditions are often imposed for learning a Markov model \cite{zhang2018spectral,li2018estimation}. Condition (c) eliminates the aggregation (disaggregation) distributions from being highly collinear, so that each of them can be accurately identified from the remaining. Condition (d) is a mild eigen-gap condition, which is necessary for consistent estimation of eigenvectors from PCA \cite{abbe2017entrywise,Topic-SCORE}. Condition (e) says that the meta-states are reachable from one-another and that meta-state transitions cannot be overly imbalanced. 

\vspace{-.8em}

\subsection{Row-wise large-deviation bounds for singular vectors}
\vspace{-.3em}
At the core of the analysis of any spectral method (the classical PCA or our unsupervised algorithm) is characterization of the errors of approximating population eigenvectors by empirical eigenvectors. If we choose the loss function to be the Euclidean norm between two vectors, it reduces to deriving a bound for the spectral norm of noise matrix \cite{DavisKahan} and is often  manageable. However, the Euclidean norm bound is useless for our problem: In order to obtain the total-variation bounds for estimating aggregation/disaggregation distributions, we need sharp error bounds for {\it each entry} of the eigenvector. Recently, there has been active research on entry-wise analysis of eigenvectors \cite{abbe2017entrywise, koltchinskii2016asymptotics, koltchinskii2016perturbation, zhong2018near, chen2017spectral,eldridge2017unperturbed}.
Since eigenvectors depend on data matrix in a complicated and highly nonlinear form, such analysis is well-known to be challenging. More importantly, there is no universal technique that works for all problems, and such bounds are obtained in a problem-by-problem manner (e.g., for Gaussian covariance matrix \cite{koltchinskii2016asymptotics, koltchinskii2016perturbation}, network adjacency matrix \cite{abbe2017entrywise, mixed-SCORE}, and topic matrix \cite{Topic-SCORE}). As an addition to the nascent literature, we develop such results for transition count matrices of a Markov chain. 
The analysis is challenging due to that the entries of the count matrix are dependent of each other. 

Recall that $\widetilde{\bN}$ is the re-scaled transition count matrix introduced in Algorithm 1 and $\widehat{\bh}_1,\ldots,\widehat{\bh}_r$ are its first $r$ right singular vectors (our technique also applies to the original count matrix $\bN$ and the empirical transition matrix $\widehat{\bP}$).  
Theorem~\ref{thm:Eig1} and Theorem~\ref{thm:Eig2} deal with the leading singular vector and the remaining ones, respectively. ($\widetilde{O}$ means ``bounded up to a logarithmic factor of $n,p$''). 
\begin{Theorem}[Entry-wise perturbation bounds for $\widehat{\bf h}_1$] \label{thm:Eig1}
Suppose the regularity conditions (a)-(e) hold. There exists a parameter $\omega \in \{\pm1\}$ such that if $n = \widetilde{\Omega}\big(\tau_*p\big)$, then with probability at least $1-n^{-1}$,
$\max_{1 \leq j \leq p} | \omega \widehat{\bf h}_1(j) - {\bf h}_1(j) | = \widetilde{O}\Big( (\sigma_2 - \sigma_1)^{-1}( 1 + \sqrt{\tau_*p/n}) \sqrt{\frac{\tau_*}{np}} \Big)$. 
\end{Theorem}


\begin{Theorem}[Row-wise perturbation bounds for $\widehat{\bf H}$] \label{thm:Eig2}
Suppose the regularity conditions (a)-(e) hold. 
For $1 \leq s \leq t \leq r$, let ${\bf H}_* = \big[ {\bf h}_s, \ldots, {\bf h}_t \big] $, $\widehat{\bf H}_* = \big[ \widehat{\bf h}_s, \ldots, \widehat{\bf h}_t \big] $, and 
$\Delta_* = \min\big\{ \sigma_{s-1} - \sigma_s, \sigma_t - \sigma_{t+1} \big\}$, where $\sigma_0 = +\infty$, $\sigma_{r+1} = 0$.  If $n = \widetilde{\Omega}\big(\tau_*p\big)$, then with probability $1-n^{-1}$, there is an orthogonal matrix ${\bf \Omega}_*$ such that
$ \max_{1 \leq j \leq p} \big\| {\bf e}_j^{\top} \big( \widehat{\bf H}_* {\bf \Omega}_* - {\bf H}_* \big) \big\|_2 = \widetilde{O}\Big(\Delta_*^{-1}\Big(1+\sqrt{\tau_*p/nr}\Big)\sqrt{\frac{\tau_*r}{np}}\Big)$. 
\end{Theorem}

\vspace{-1em}

\subsection{Statistical guarantees of soft state aggregation}
\vspace{-.5em}
We study the error of estimating $\bU$ and $\bV$, as well as recovering the set of anchor states. Algorithm 1 plugs in some existing algorithm for the simplex finding problem. We make the following assumption:

\begin{Assumption}[Efficiency of simplex finding]
Given data $\widehat{\bd}_1,...,\widehat{\bd}_p\in\mathbb{R}^{r-1}$, suppose they are noisy observations of non-random vectors $\bd_1,...,\bd_p$, where these non-random vectors are contained in a simplex with $r$ vertices $\bb_1,\ldots,\bb_r$ with at least one $\bd_j$ located on each vertex. The simplex finding algorithm outputs $\widehat{\bb}_1,...,\widehat{\bb}_r$ such that $\max_{1\leq k\leq r}\|\widehat{\bb}_k-\bb_k\|\leq C\max_{1\leq j\leq p}\|\widehat{\bd}_j-\widehat{\bd}_j\|$. 
\end{Assumption}

Several existing simplex finding algorithms satisfy this assumption, such as the successive projection algorithm \cite{araujo2001successive}, the vertex hunting algorithm \cite{mixed-SCORE, Topic-SCORE}, and the algorithm of archetypal analysis \cite{javadi2019non}.  
Since this is not the main contribution of this paper, we refer the readers to the above references for details. In our numerical experiments, we use the vertex hunting algorithm in \cite{mixed-SCORE, Topic-SCORE}. 

First, we provide total-variation bounds between estimated individual aggregation/disaggregation distributions and the ground truth. Write $\bV=[\bV_1,\ldots,\bV_r]$ and $\bU=[\bu_1,\ldots,\bu_p]^{\top}$, where each $\bV_k\in\mathbb{R}^p$ is a disaggregation distribution and each $\bu_i\in\mathbb{R}^r$ is an aggregation distribution. 

\begin{Theorem}[Error bounds for estimating $\bV$] \label{StatisticalError}
Suppose the regularity conditions (a)-(e) hold and Assumptions 1 and 2 are satisfied. 
When $n = \widetilde{\Omega}\bigl( \tau_* p^{\frac{3}{2}}r \bigr)$, with probability at least $1-n^{-1}$, the estimate $\widehat{\bV}$ given by Algorithm 1 satisfies
$\frac{1}{r} \sum_{k=1}^r \bigl\|\widehat{\bV}_k - \bV_k\bigr\|_1 = \widetilde{O}\Bigl( \bigl(1 + p\sqrt{\tau_*/n}\bigr)\sqrt{\frac{\tau_*pr}{n}} \Bigr)$. 
\end{Theorem}

\begin{Theorem}[Error bounds for estimating $\bU$] \label{StatisticalError2}
Suppose the regularity conditions (a)-(e) hold and Assumptions 1 and 2 are satisfied. When $n = \widetilde{\Omega}\bigl( \tau_* p^{\frac{3}{2}} r\bigr)$, with probability at least $1-n^{-1}$, the estimate $\widehat{\bU}$ given by Algorithm 1 satisfies
$\frac{1}{p} \sum_{j=1}^p  \bigl\|\widehat{\bu}_j - \bu_j\bigr\|_1 = \widetilde{O}\Bigl( r^{\frac{3}{2}}\bigl(1 + p\sqrt{\tau_*/n}\bigr)\sqrt{\frac{\tau_*pr}{n}} \Bigr)$. 
\end{Theorem}

Second, we provide sample complexity guarantee for the exact recovery of anchor states. To eliminate false positives, we need a condition that the non-anchor states are not too `close' to an anchor state; this is captured by the quantity $\delta$ below. (Note $\delta_j = 0$ for anchor states $j \in \mathcal{A}^*$.)

\begin{Theorem}[Exact recovery of anchor states]\label{thm-anchor}
Suppose the regularity conditions (a)-(e) hold and Assumptions 1 and 2 are satisfied. Let ${\cal A}^*$ be the set of (disaggregation) anchor states. Define $\delta_j=1-\max_{1 \leq k \leq r} \mathbb{P}_{X_0 \sim \boldsymbol{\pi}}( Z_0 = k \, | \, X_1 = j )
$ and $\delta = \min_{j \notin \mathcal{A}^*} \delta_j$. Suppose the threshold $\delta_0$ in Algorithm 1 satisfies $\delta_0 = O(\delta)$. If $n = \widetilde{\Omega}\bigl(\delta_0^{-2}\tau_*p^{\frac{3}{2}}r\bigr)$, then $\mathbb{P}(\mathcal{A} = \mathcal{A}^*) \geq 1-n^{-1}$. 
\end{Theorem}


We connect our results to several lines of works in the literature. First, in the special case of $r=1$, our problem reduces to learning a discrete distribution with $p$ outcomes, where the minimax rate of total-variation distance is $O(\sqrt{p/n})$ \cite{han2015minimax}. Our bound matches with this rate when $p=O(\sqrt{n})$. However, our problem is much harder: each row of $\bP$ is a mixture of $r$ discrete distributions. Second, our setting is connected to the setting of learning a mixture of discrete distributions \cite{rabani2014learning,li2015learning} but is different in important ways. Those works consider learning {\it one} mixture distribution, and the data are $iid$ observations. Our problem is to estimate $p$ mixture distributions, which share the same basis distributions but have different mixing proportions, and our data are a single trajectory of a Markov chain. Third, our problem is connected to topic modeling \cite{Ge,Topic-SCORE}, where we may view the empirical transition profile of each raw state as a `document'. However, in topic modeling, the documents are independent of each other, but the `documents' here are highly dependent as they are generated from a single trajectory of a Markov chain. Last, we compare with the literature of estimating the transition matrix $\bP$ of a Markov model. Without low-rank assumptions on $\bP$, the minimax rate of the total variation error is $O(p/\sqrt{n})$ \cite{pmlr-v98-wolfer19a} (also, see \cite{kontorovich2013learning} and reference therein for related settings in hidden Markov models); with a low-rank structure on $\bP$, the minimax rate becomes $O(\sqrt{rp/n})$ \cite{zhang2018spectral}. To compare, we use our estimator of $(\bU,\bV)$ to construct an estimator of $\bP$ by $\widehat{\bP}=\widehat{\bU}\widehat{\bV}^{\top}$. When $r$ is bounded and $p=O(\sqrt{n})$, this estimator achieves a total-variation error of $O(\sqrt{rp/n})$, which is optimal. At the same time, we want to emphasize that estimating $(\bU,\bV)$ is a more challenging problem than estimating $\bP$, and we are not aware of any existing theoretical results of the former.

%

\vspace{-.3cm}

\paragraph{Simulation.}
We test our method on simulations (settings are in the appendix). The results are summarized in Figure~\ref{fig:ErrorComparison}. It suggests: (a) the rate of convergence in Theorem~\ref{StatisticalError} is confirmed by numerical evidence, and (b) our method compares favorably with existing methods on estimating $\bP$ (to our best knowledge, there is no other method that directly estimates $\bU$ and $\bV$; so we instead compare the estimation on $\bP$).

\begin{figure}[ht]
	\centering
		\includegraphics[width = 1.6in,height = 1.3in]{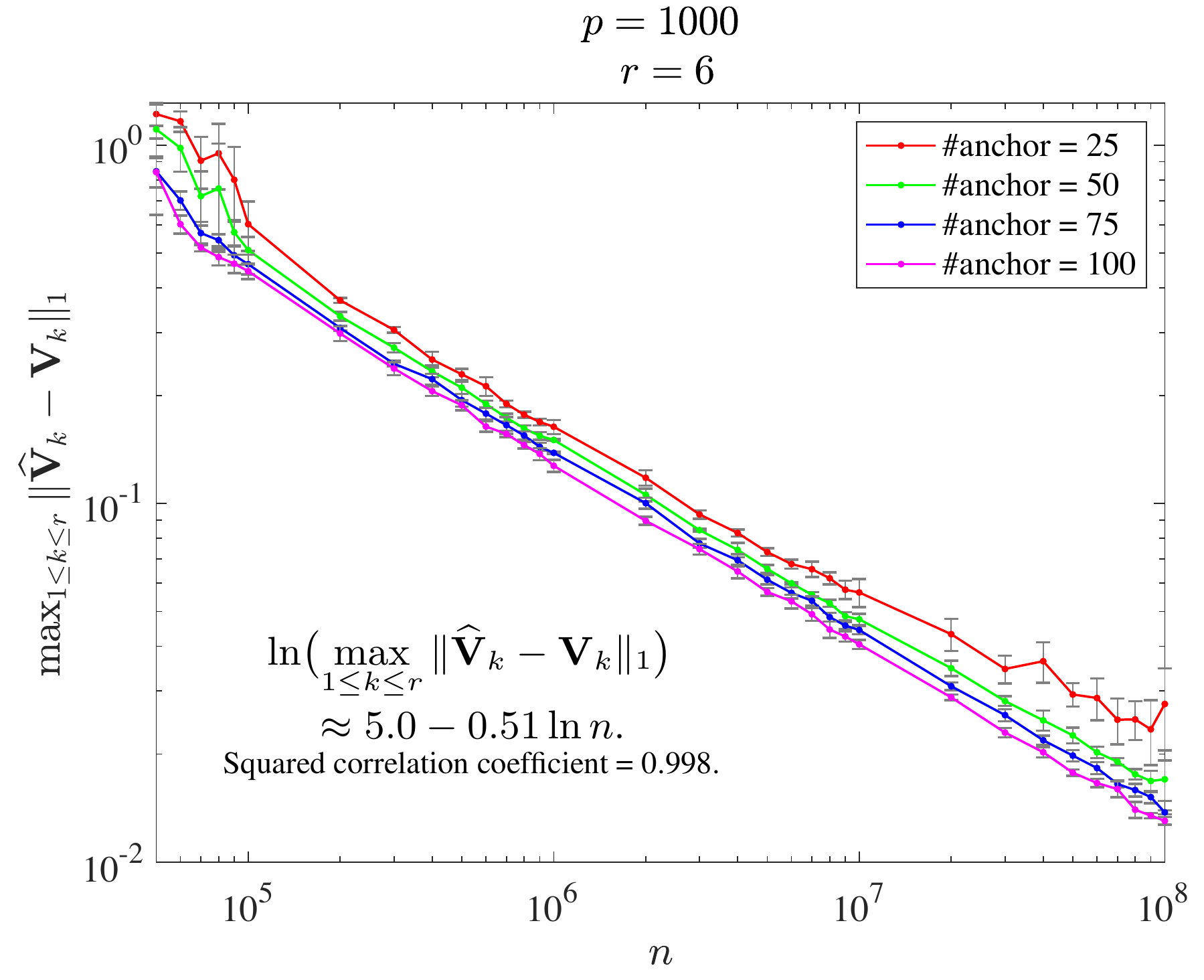}  \hspace{0.3cm}
		\includegraphics[width = 1.6in,height = 1.3in]{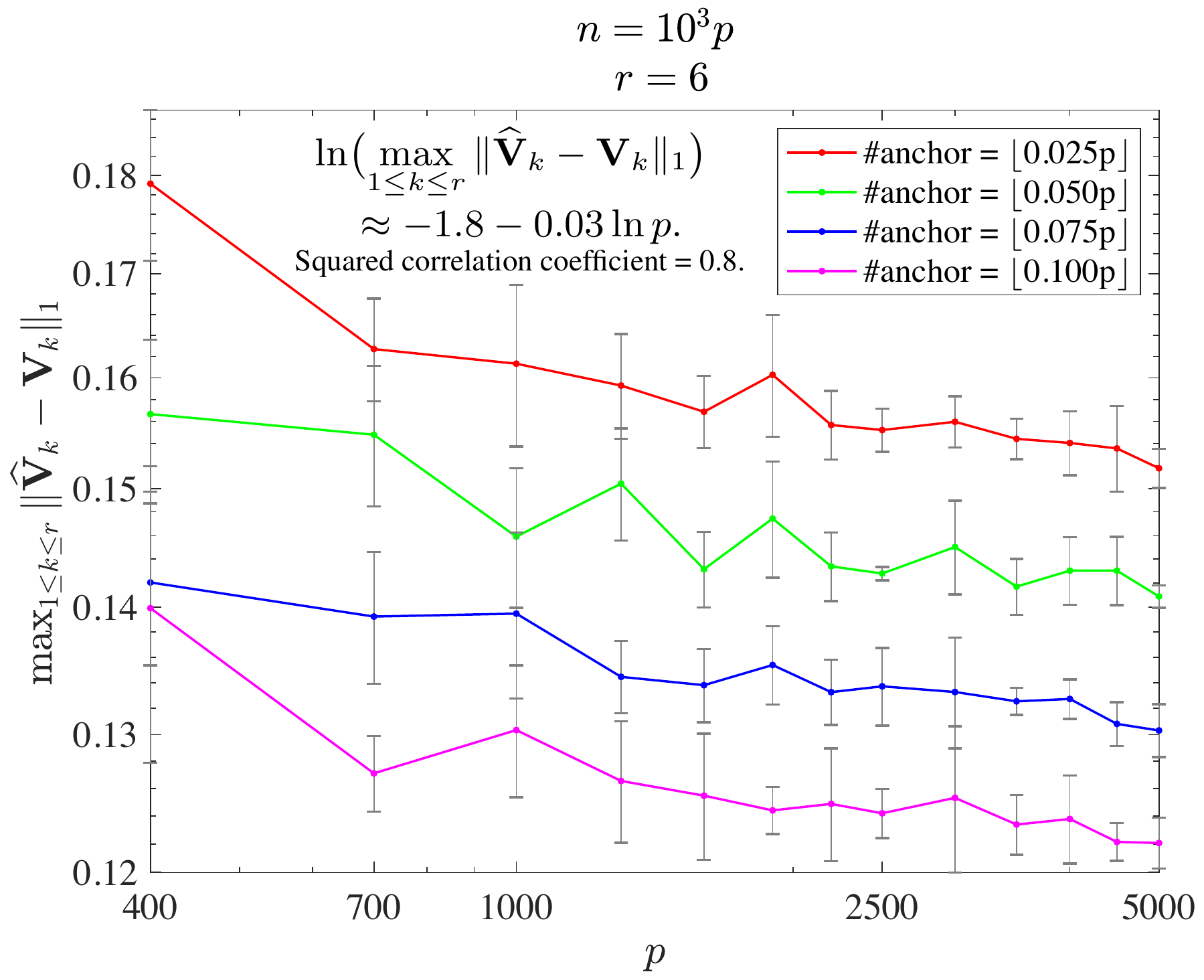}  \hspace{0.3cm}
	    \includegraphics[width = 1.6in,height = 1.3in]{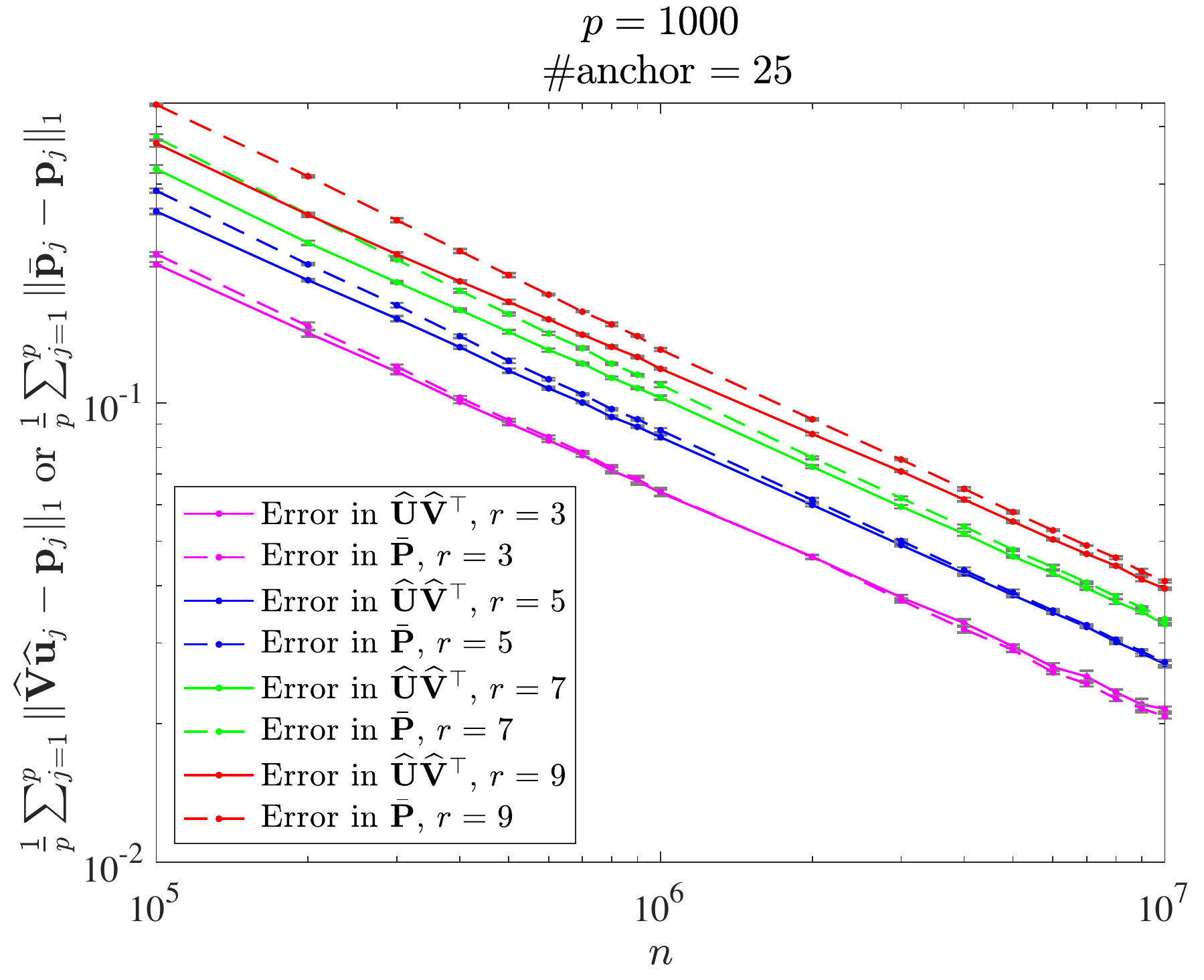} 
	        \vspace{-1em}
	\caption{\footnotesize {\bf Simulation Results.} Left: Total variation error on $\bV$ ($p$ is fixed and $n$ varies). Middle:   Total variation error on $\bV$ ($p/n$ is fixed). Both panels validate the scaling of $\sqrt{p/n}$ in Theorem~\ref{StatisticalError}. Right: Recovery error of $\bP$, where $\widehat{\bf U}\widehat{\bf V}^{\top}$ is our method and $\bar{\bf P}$ is the spectral estimator \cite{zhang2018spectral} (note: this method cannot estimate $\bU$,$\bV$ ).} \label{fig:ErrorComparison}
	\vspace{-0.2cm}
\end{figure}

\vspace{-.8em}

\section{Analysis of NYC Taxi Data and Application to Reinforcement Learning} \label{sec:application}
\vspace{-.5em}




\begin{figure}[ht]
	\centering
	
	\begin{minipage}{0.32\linewidth}
		\centering
		\hspace{-.7cm}
		{\footnotesize \begin{tabular}{c}(a)\\ \\ \end{tabular}}
		\hspace{-0.3cm}
		\begin{minipage}{0.4\linewidth}
			\centering
			\includegraphics[width = 0.95\linewidth]{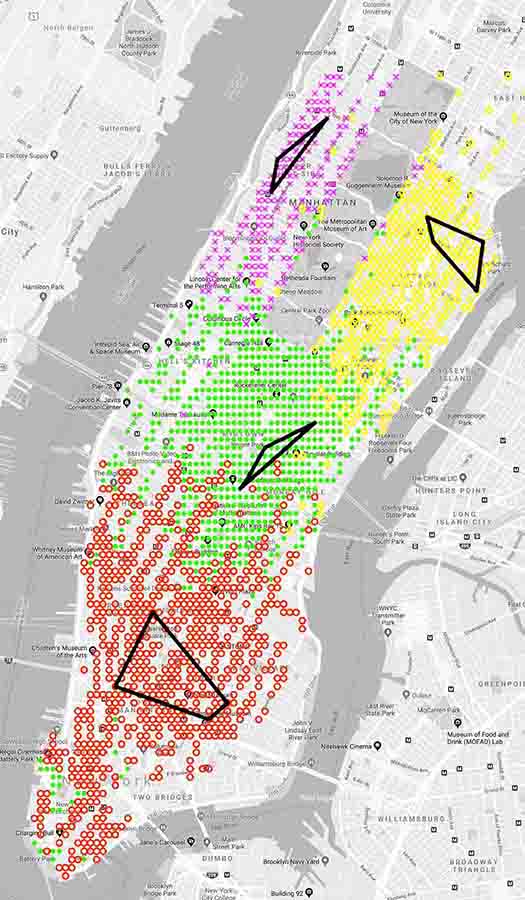} \vspace{-0.05cm}\\
			{\footnotesize $r = 4$  }
		\end{minipage}
		\hspace{0.02cm}
		\begin{minipage}{0.4\linewidth}
			\centering
			\includegraphics[width = 0.95\linewidth]{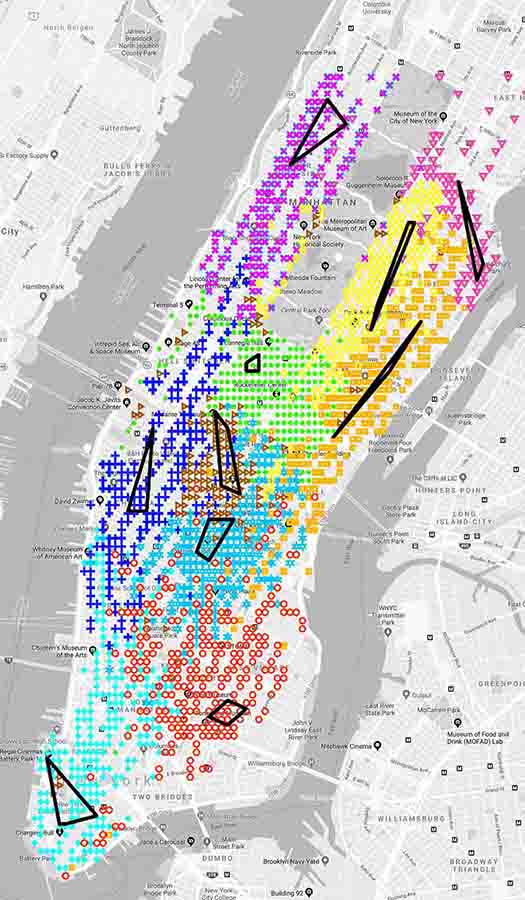} \vspace{-0.05cm}\\
			{\footnotesize $r = 10$  }
		\end{minipage} \vspace{-0.15cm}
	\end{minipage}
	\hspace{0.02cm}
	\begin{minipage}{0.65\linewidth}
		\centering
		\hspace{-1cm}
		{\footnotesize \begin{tabular}{c}(b)\\ \\ \end{tabular}}
		\hspace{-0.3cm}
		\begin{minipage}{0.23\linewidth}
			\centering
			\includegraphics[width = \linewidth]{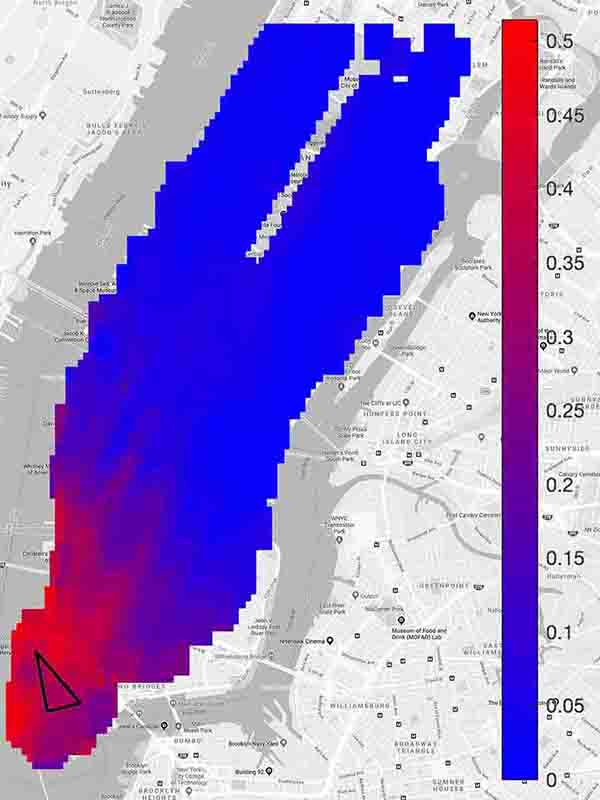} \\
			{\footnotesize $\widehat{\bf U}_k$ (downtown)}
		\end{minipage}
		\hspace{0.01cm}
		\begin{minipage}{0.23\linewidth}
			\centering
			\includegraphics[width = \linewidth]{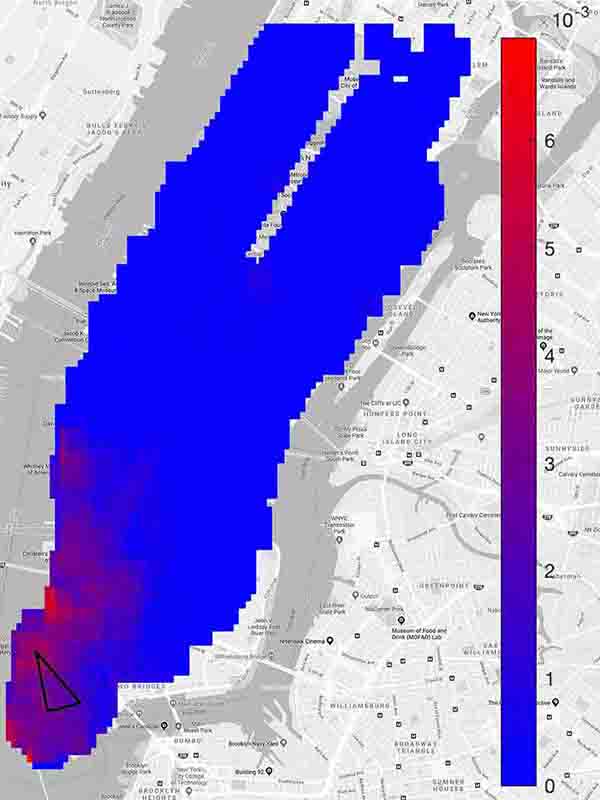}\\
			{\footnotesize $\widehat{\bf V}_k$ (downtown)}
		\end{minipage}
		\hspace{0.15cm}
		\begin{minipage}{0.23\linewidth}
			\centering
			\includegraphics[width = \linewidth]{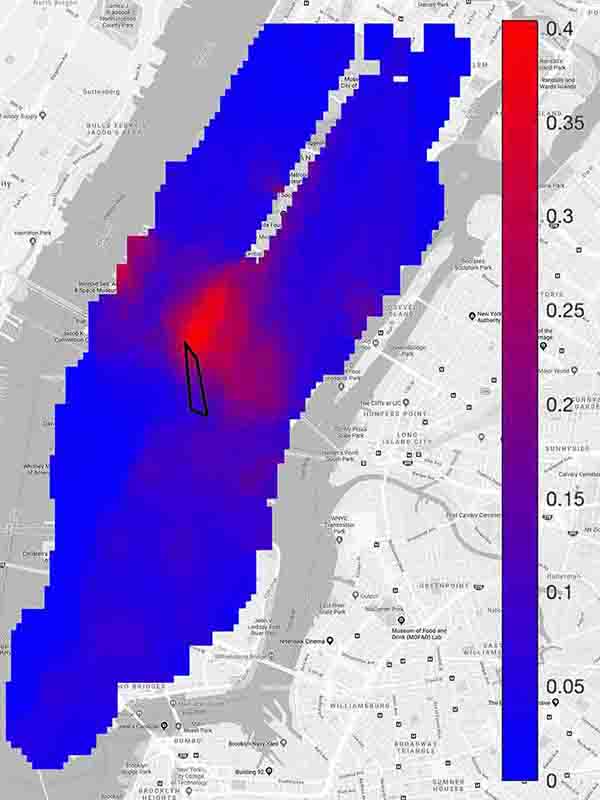}\\
			{\footnotesize $\widehat{\bf U}_k$ (midtown)}
		\end{minipage}
		\hspace{0.01cm}
		\begin{minipage}{0.23\linewidth}
			\centering
			\includegraphics[width = \linewidth]{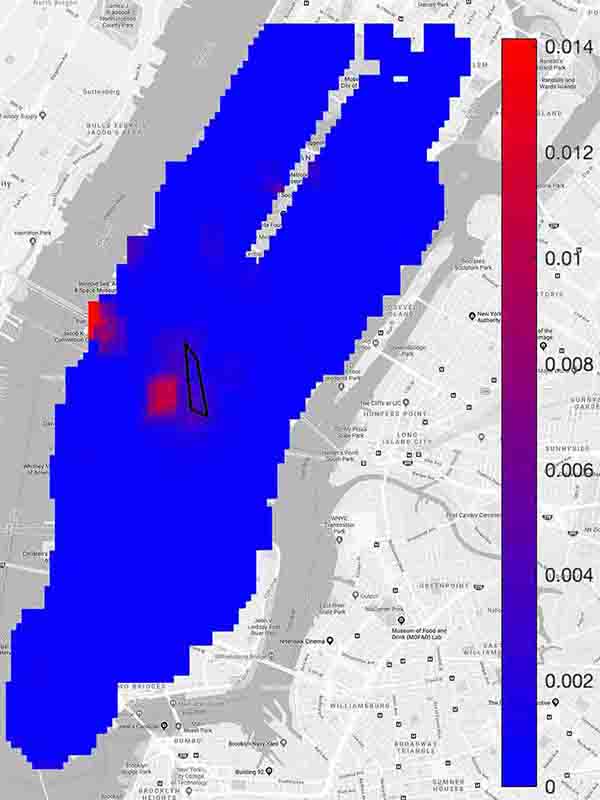} \\
			{\footnotesize $\widehat{\bf V}_k$ (midtown)}
		\end{minipage}
		\\ \vspace{-0.1cm}
	\end{minipage}
	\caption{\footnotesize {\bf State aggregation results.} (a) Estimated anchor regions and partition of NYC. (b) Estimated disaggregation distribution ($\widehat{\bV}_k$) and aggregation likelihood ($\widehat{\bf U}_k$) for two meta-states.}\label{HeatMap}
	\vspace{-0.5cm}
\end{figure}

We analyze a dataset of $1.1 \times 10^7$ New York city yellow cab trips that were collected in January 2016 \cite{NYCyellowcabJan2016}. We treat each taxi trip as a sample transition generated by a city-wide Markov process over NYC, where the transition is from a pick-up location to some drop-off location. We  discretize the map into $p=1922$ cells so that the Markov process becomes a finite-state one.

\vspace{-0.1cm}
{\bf Anchor regions and partition.}
We apply Alg.\ 1 to the taxi-trip data with $r=4,10$. The algorithm identifies sets of anchor states that are close to the vertices, as well as columns of $\widehat{\bU}$ and $\widehat{\bV}$ corresponding to each vertex (anchor region).
We further use the estimated $\widehat{\bf U}$, $\widehat{\bf V}$ to find a partition of the city. Recall that in Algorithm 1, each state is projected onto a simplex, which can be represented as a convex combination of simplex's vertices  (see Figure~\ref{fig:geometry}). 
For each state, we assign the state to a cluster that corresponds to largest value in the weights of convex combination. In this way, we can cluster the 1922 locations into a small number of regions.
The partition results are shown in Figure \ref{HeatMap} (a), where anchor regions are marked in each cluster.

\vspace{-0.1cm}
{\bf Estimated aggregation and disaggregation distributions.}
Let $\widehat{\bf U}$, $\widehat{\bf V}$ be the estimated aggregation and disaggregation matrices.
We use heat maps to visualize their columns. Take $r=10$ for example. We pick two meta-states, with anchor states in the downtown and midtown areas, respectively.
We plot in Figure \ref{HeatMap} (b) the corresponding columns of $\widehat{\bf U}$ and $\widehat{\bf V}$. Each column of $\widehat{\bf V}$ is a disaggregation distribution, and each column of $\widehat{\bf U}$ can be thought of as a likelihood function for transiting to corresponding meta-states. The heat maps reveal the leading ``modes'' of the traffic-dynamics.

\vspace{-0.3cm}
\paragraph{Aggregation distributions used as features for RL.}
\begin{wrapfigure}{r}{0.3\linewidth} \vspace{-0.1cm}
	\begin{minipage}{\linewidth}
		\centering
		\includegraphics[width = 0.75\linewidth]{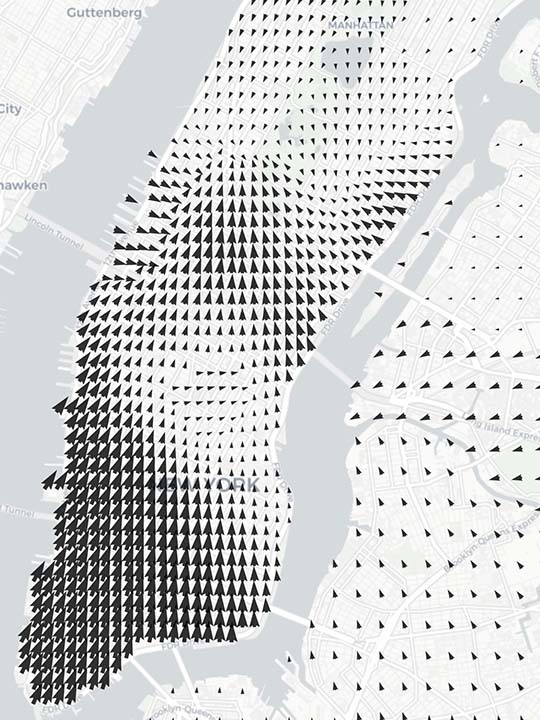} \\ \vspace{-0.2cm}
		\caption{\footnotesize {\bf The optimal driving policy} learnt by feature-based RL with estimated aggregation distributions as state features. Arrows point out the most favorable directions and the thickness is proportional to favorability of the direction.}\label{Policy}
		\vspace{-0.2cm}
	\end{minipage}
\end{wrapfigure}
\hspace{-0.35cm}
Soft state aggregation can be used to reduce the complexity of reinforcement learning (RL) \cite{singh1995reinforcement}.
Aggregation/disaggregation distributions provide features to parameterize high-dimensional policies, in conjunction with feature-based RL methods \cite{chen2018scalable,yang2019sample}. Next we experiment with using the aggregation distributions as features for RL.

Consider the taxi-driving policy optimization problem. The driver's objective is to maximize {the} daily revenue - a Markov decision process where the driver chooses driving directions in realtime based on locations. 
We compute the optimal policy using feature-based RL \cite{chen2018scalable} and simulated NYC traffic. The algorithm takes as input 27 estimated aggregation distributions as {\it state features}. For comparison, we also use a hard partition of the city which is handpicked according to 27 NYC districts. 
RL using aggregation distributions as features achieves a daily revenue of $\$230.57$, while the method using handpicked partition achieves $\$209.14$. 
Figure \ref{Policy} plots the optimal driving policy. 
This experiment suggests that (1) state aggregation learning provides features for RL automatically; (2) using aggregation distributions as features leads to better performance of RL than using handpicked features. 

%
%
%

%


\newpage

\bibliography{reference}
\bibliographystyle{plain}

\newpage

	\appendix

	\addtocontents{toc}{\protect\setcounter{tocdepth}{2}}
	\numberwithin{equation}{section}
	
	{ \huge \bf Appendices }
	
	\vspace{1.5cm}
	
	\tableofcontents
	
	\newpage
	\section{Preliminaries}

	%
	\paragraph{Notations.}~
	
	In the following, we denote the frequency matrix $[\diag(\bpi)]\bP$ by $\bF$.
	The algorithm deals with $\widetilde{\bN} = \bN[\diag(\bm)]^{-\frac{1}{2}}$ where $\bm = \bN^{\top}{\bf 1}_p$. We normalize $\widetilde{\bN}$ with $\sqrt{n}$ and define
	$\widehat\bQ = n^{-\frac{1}{2}}\bN[\diag(\bm)]^{-\frac{1}{2}}$. Martrix $\widehat{\bQ}$ has a population counterpart $\bQ = \bF[\diag(\bpi)]^{-1/2}$. The soft state aggregation assumption implies that $\bQ$ has a singular value decomposition $\bQ = \bG \bSigma \bH^{\top}$, where $\bG = \bigl[\bg_1,\bg_2,\ldots,\bg_r\bigr] \in \mathbb{R}^{p \times r}$, $\bH = \bigl[ \bh_1,\bh_2, \ldots, \bh_r \bigr] \in \mathbb{R}^{p \times r}$, $\bSigma = \diag(\sigma_1,\sigma_2,\ldots,\sigma_r)$, $\sigma_1 \geq \sigma_2 \geq \ldots \geq \sigma_r > 0$. Analogously, suppose that $\widehat\bQ$ has singular values $\widehat\sigma_1 \geq \widehat\sigma_2 \geq \ldots \geq \widehat\sigma_p \geq 0$. For $i=1,2,\ldots,p$, $\widehat\bg_i$ and $\widehat\bh_i$ are respectively the left and right singular vectors associated with $\widehat\sigma_i$.
	
	For any $k \in \mathbb{N}$, let $\mathbb{S}^{k-1}$ denote the unit sphere in $\mathbb{R}^k$ and ${\bf I}_k \in \mathbb{R}^{k \times k}$ denote the identity matrix.
	
	
	
	\paragraph{Assumptions.}~
	
	In the paper, we propose the following regularity conditions:
	\begin{Assumption}[Regularity conditions]\label{assumption-reg}
		There exists constants $c_1$, $C_1$, $\bar{C}_1$, $c_2$, $c_3$, $C_4 > 0$ such that
		\begin{enumerate} \itemsep = -0.3cm
			\item the stationary distribution $\bpi$ satisfies \beq c_1p^{-1} \leq\label{Assumption1} \bpi_j \leq C_1p^{-1}, \quad \text{for $j=1,2,\ldots,p$}; \eeq
			\item the stationary distribution on meta-states satisfies \beq \label{Assumption5} \bigl(\bU^{\top}\bpi\bigr)_k \leq \bar{C}_1 r^{-1}, \quad \text{for $k=1,2,\ldots,r$}; \eeq
			\item the aggregation and disaggregation distributions satisfy \beq\label{Assumption2'} \lambda_{\min}\bigl(\bU^{\top} \bU\bigr) \geq c_2 pr^{-1}, \
			\lambda_{\min}\bigl(\bV^{\top} \bV \bigr) \geq c_2 p^{-1}r; \eeq	
			\item the first and second largest singular values of $[\diag(\bpi)]\bP[\diag(\bpi)]^{-1/2}$ satisfy \beq \label{Assumption3} \sigma_1 - \sigma_2 \geq c_3 p^{-\frac{1}{2}}; \eeq
			\item the ratio between the largest and smallest entries in an $r$-by-$r$ matrix $\bU^{\top}\bP\bV$ satisfies  \beq\label{Assumption4} \frac{\max_{k,l}\big(\bU^{\top}\bP\bV\bigr)_{kl}}{\min_{k,l}\big(\bU^{\top}\bP\bV\bigr)_{kl}} \leq C_4. \eeq
		\end{enumerate}
	\end{Assumption}
	
	We can rewrite \eqref{Assumption2'} into the following form, which is more convenient for our proofs:
	\beq\label{Assumption2} \begin{aligned} & \lambda_{\min}\bigl(\bU^{\top} [\diag(\bpi)]^2 \bU\bigr) \geq c_2' p^{-1}r^{-1}, \\
		& \lambda_{\min}\bigl(\bV^{\top} [\diag(\bpi)]^{-1} \bV \bigr) \geq c_2' r. 
	\end{aligned} \eeq	
	To see the equivalence between \eqref{Assumption2'} and \eqref{Assumption2}, we note that under assumption \eqref{Assumption1},
	\[ \bU^{\top}\bigl( [\diag(\bpi)]^2 - c_1^2p^{-2} \bI_p \bigr) \bU \succeq 0, \]
	thus by \eqref{Assumption2'},
	\[ \bU^{\top}[\diag(\bpi)]^2\bU \succeq c_1^2 p^{-2}\bU^{\top}\bU \succeq c_1^2c_2p^{-1}r^{-1} {\bf I}_r. \]
	Similarly,
	\[ \bV^{\top}\bigl([\diag(\bpi)]^{-1} - C_1^{-1}p \bI_p \bigr)\bV \succeq 0, \]
	which implies
	\[ \bV^{\top}[\diag(\bpi)]^{-1}\bV \succeq C_1^{-1}p \bV^{\top}\bV \succeq C_1^{-1}c_2 r {\bf I}_r. \]
	For simplicity, we replace the parameter $c_2'$ in \eqref{Assumption2} by $c_2$ in the subsequent discussions.

	\paragraph{Preliminary lemmas.}~
	
	We list some preliminary lemmas that will be used later, of which the proofs can be found in Appendix \ref{TechnicalProofs}.
	
	\begin{lemma} \label{sigma_r}
		Under assumptions \eqref{Assumption1} and \eqref{Assumption2},
		\beq \label{normQ} c_2 p^{-\frac{1}{2}} \leq \sigma_r \leq \sigma_1 \leq C_1c_1^{-\frac{1}{2}}p^{-\frac{1}{2}}.\eeq
	\end{lemma}

	\begin{lemma} \label{RowNorm}
		Under assumptions \eqref{Assumption1} and \eqref{Assumption2}, we have
		\begin{eqnarray} 
		& \bigl\|\be_j^{\top}\bG\bigr\|_2 \leq c_2^{-\frac{1}{2}}\pi_j\sqrt{pr}, \label{rowG} \\ 
		& \bigl\|\be_j^{\top}\bH\bigr\|_2 \leq C_1 c_2^{-\frac{3}{2}} \sqrt{\pi_j r}. \label{rowH} 
		\end{eqnarray}
	\end{lemma}

	\begin{lemma} \label{B}
		Let $\bL \in \mathbb{R}^{r \times r}$ be the matrix defined by \eqref{Def_L} and $\bl_1 \in \mathbb{R}^{r}$ be the first column of $\bL$.
		Under assumptions \eqref{Assumption1}, \eqref{Assumption2} and \eqref{Assumption4}, there exist constants $c>0$ and $C>0$ such that
		\[ cr^{-1} \leq \bl_1(k) \leq Cr^{-1} \quad \text{for } k = 1,2,\ldots,r, \qquad
		\text{and} 
		\qquad c_2^{\frac{1}{2}}\sqrt{r} \leq \bigl\| \bL^{-1} \bigr\|_2 \leq C_1^{\frac{1}{2}}c_2^{-\frac{1}{2}}\sqrt{r}. \]
	\end{lemma}

	\begin{lemma} \label{Pre}
		Under assumptions \eqref{Assumption1}, \eqref{Assumption5}, \eqref{Assumption2} and \eqref{Assumption4}, there exist constants $c>0$ and $C>0$ such that for $j = 1,2,\ldots,p$,
		\beq \label{h1} \begin{aligned}  c\sqrt{\pi_j} \leq \bh_1(j) \leq C\sqrt{\pi_j}, \qquad 0 \leq \bg_1(j) \leq C \pi_j\sqrt{p}. \end{aligned} \eeq
	\end{lemma}

	\section{Proof of Entry-wise Eigenvector Bounds}
	
	A building block of our method is to get a sharp error bound for each row of $\widehat\bH=[\widehat\bh_1,\ldots,\widehat\bh_r]$ and each entry of $\widehat\bh_1$. 

	\subsection{Deterministic Analysis}
	
	\begin{lemma}[A deterministic row-wise perturbation bound for singular vectors] \label{Deterministic Row-wise Error}
		For $1 \leq s \leq t \leq r$, denote $\bvarH = \bigl[\bh_{s},\bh_{s+1},\ldots,\bh_{t}\bigr]$ and $\widehat\bvarH = \bigl[\widehat\bh_s,\widehat\bh_{s+1},\ldots,\widehat\bh_t\bigr]$. Let \[\Delta = \min\bigl\{ \sigma_{s-1} - \sigma_s, \sigma_t - \sigma_{t+1}, \sigma_t \bigr\},\] where $\sigma_0 = +\infty$ and $\sigma_{r+1} = 0$ for simplicity. Suppose that $\bigl\| \widehat\bQ-\bQ\bigr\|_2 \leq \frac{\Delta}{2}$, $\Delta > 0$. Then there exists an orthogonal matrix $\bvarOmega \in \mathbb{R}^{(t-s+1) \times (t-s+1)}$ such that
		\beq \label{Row-wise Error} \begin{aligned}
			\bigl\|\be_j^{\top}(\widehat\bvarH\bOmega-\bvarH)\bigr\|_2 \leq & \frac{2}{\Delta}\bigl\|\bvarG^{\top}(\widehat\bQ-\bQ)\be_j\bigr\|_2 \\ & + \frac{4(1+\sqrt{2})}{\Delta}\bigl\|\be_j^{\top}\bH\bigr\|_2\bigl\|\widehat\bQ-\bQ\bigr\|_2 \\ & + \frac{8}{\Delta^2}\bigl\|\widehat\bQ-\bQ\bigr\|_2^2, \end{aligned} \eeq
		where $\bH = [\bh_1,\bh_2,\ldots,\bh_r]$ and $\bvarG = \bigl[\bg_s,\bg_{s+1},\ldots,\bg_t\bigr]$.
	\end{lemma}
	
	\begin{proof}
		We apply a symmetric dilation to matrices $\widehat\bQ$ and $\bQ$, so as to relate the singular vectors $\widehat\bvarH = \bigl[\widehat\bh_s,\widehat\bh_{s+1},\ldots,\widehat\bh_t\bigr]$ and $\bvarH = \bigl[\bh_s,\bh_{s+1},\ldots,\bh_t\bigr]$ to the eigen vectors of some symmetric matrices.
		Define
		\beq \label{SymmetricDilation} \widehat\bY \equiv \left[ \begin{array}{cc} {\bf 0} & \widehat\bQ \\ \widehat\bQ^{\top} & {\bf 0} \end{array} \right] \quad \text{and} \quad \bY \equiv \left[ \begin{array}{cc} {\bf 0} & \bQ \\ \bQ^{\top} & {\bf 0} \end{array} \right]. \eeq
		The symmetric matrix $\widehat\bY$ has eigen pairs $(\widehat\sigma_i, \widehat\bxi_i)$ and $(-\widehat\sigma_i, \widehat\bxi_{-i})$ with \[ \widehat\bxi_i = \frac{1}{\sqrt{2}}\left[ \begin{array}{c} \widehat\bg_i \\ \widehat\bh_i \end{array} \right] \quad \text{and} \quad \widehat\bxi_{-i} = \frac{1}{\sqrt{2}}\left[ \begin{array}{c} \widehat\bg_i \\ -\widehat\bh_i \end{array} \right] \] for $i=1,2,\ldots,p$. Analogously, \[ \bxi_i = \frac{1}{\sqrt{2}}\left[ \begin{array}{c} \bg_i \\ \bh_i \end{array} \right] \quad \text{and} \quad \bxi_{-i} = \frac{1}{\sqrt{2}}\left[ \begin{array}{c} \bg_i \\ -\bh_i \end{array} \right]\] are eigen vectors of $\bY$ associated with $\sigma_i$ and $-\sigma_i$ for $i=1,2,\ldots,r$. Define \[ \bXi = \bigl[ \bxi_1, \bxi_2, \ldots, \bxi_r, \bxi_{-1}, \bxi_{-2}, \ldots, \bxi_{-r} \bigr], \] then $\bY$ adimits an eigen decomposition \[ \bY = \bXi[\diag(\sigma_1,\sigma_2,\ldots,\sigma_r,-\sigma_1,-\sigma_2,\ldots,-\sigma_r)]\bXi^{\top}. \]
		
		Let $\widehat\bvarXi = \bigl[\widehat\bxi_s, \ldots, \widehat\bxi_t\bigr]$, $\widehat\bvarSigma = \diag(\widehat\sigma_s,\ldots,\widehat\sigma_t)$ and $\bvarXi = \bigl[\bxi_s, \ldots, \bxi_t\bigr]$. Based on the Davis-Kahan $\sin\theta$ theorem \cite{DavisKahan}, we can estimate the difference between the subspaces spaned by the columns of $\widehat\bvarXi$ and $\bvarXi$. Denote $k=t-s+1$. Suppose that the singular values of $\bvarXi^{\top}\widehat\bvarXi$ are $\sigma_1\bigl(\bvarXi^{\top}\widehat\bvarXi\bigr) \geq \sigma_2\bigl(\bvarXi^{\top}\widehat\bvarXi\bigr) \geq \ldots \geq \sigma_k\bigl(\bvarXi^{\top}\widehat\bvarXi\bigr)$. Then we call \[ \begin{aligned} \Theta(\widehat\bvarXi,\bvarXi) =  \diag\Bigl(\arccos\bigl(\sigma_1\bigl(\bvarXi^{\top}\widehat\bvarXi\bigr)\bigr), \ldots, \arccos\bigl(\sigma_k\bigl(\bvarXi^{\top}\widehat\bvarXi\bigr)\bigr)\Bigr) \end{aligned} \]
		the principal angles.
		According to (\cite{DavisKahan}, Proposition 4.1),
		\beq\label{DavisKahan1} \min_{\bvarOmega \in \mathbb{O}^{k\times k}} \bigl\| \widehat\bvarXi\bvarOmega - \bvarXi \bigr\|_2 = 2 \Bigl\| \sin \Bigl( \frac{1}{2}{\bf \Theta}(\widehat\bvarXi,\bvarXi) \Bigr) \Bigr\|_2. \eeq
		Denote by $\bvarOmega$ the orthogonal matrix that achieves the minimum in \eqref{DavisKahan1}.
		Since $\sin\bigl(\frac{\theta}{2}\bigr) \leq \frac{\sin\theta}{\sqrt{2}}$ for all $\theta \in [0,\frac{\pi}{2}]$, we have \[ \Bigl\| \sin \Bigl( \frac{1}{2}{\bf \Theta}(\widehat\bvarXi,\bvarXi) \Bigr) \Bigr\|_2 \leq \frac{\sqrt{2}}{2} \bigl\| \sin \bigl({\bf \Theta}(\widehat\bvarXi,\bvarXi) \bigr) \bigr\|_2.\] The Davis-Kahan $\sin\theta$ theorem further implies that \[\bigl\| \sin \bigl({\bf \Theta}(\widehat\bvarXi,\bvarXi) \bigr) \bigr\|_2 \leq \widehat\Delta^{-1} \bigl\| (\widehat\bY-\bY)\bvarXi \bigr\|_2,\] where $\widehat\Delta = \min\bigl\{\widehat\sigma_{s-1}-\sigma_s,\sigma_t - \widehat\sigma_{t+1}\bigr\}$. By Weyl's inequality, $\widehat\Delta \geq \Delta - \bigl\|\widehat\bQ-\bQ\bigr\|_2 \geq \frac{\Delta}{2}$. Hence,
		\beq\label{DavisKahan2} \begin{aligned} \bigl\| \widehat\bvarXi \bvarOmega - \bvarXi \bigr\|_2 \leq 2\sqrt{2} \Delta^{-1}\bigl\| (\widehat\bY - \bY) \bvarXi \bigr\|_2 \leq 2\sqrt{2} \Delta^{-1}\bigl\| \widehat\bY - \bY \bigr\|_2. \end{aligned} \eeq
		
		We next analyze row-wise errors $\bigl\| \be_j^{\top}(\widehat\bvarXi\bvarOmega-\bvarXi) \bigr\|_2$ for $j=1,2,\ldots,p$. Following the proof idea in \cite{Topic-SCORE}, Lemma 3.2, we propose a matrix  \[ \widetilde\bvarXi = \bigl[ \widetilde\bxi_s, \widetilde\bxi_{s+1}, \ldots, \widetilde\bxi_k \bigr] \quad \text{with} \quad \widetilde\bxi_i = \widehat\sigma_i^{-1}\bY \widehat\bxi_i,\] and decompose the row-wise error as
		\beq\label{Row-wise Error0} \begin{aligned} \bigl\| \be_j^{\top}(\widehat\bvarXi\bvarOmega-\bvarXi) \bigr\|_2 \leq & \bigl\| \be_j^{\top}(\widehat\bvarXi\bvarOmega-\widetilde\bvarXi\bvarOmega) \bigr\|_2 + \bigl\| \be_j^{\top}(\widetilde\bvarXi\bvarOmega-\bvarXi) \bigr\|_2 \\ = & \bigl\| \be_j^{\top}(\widehat\bvarXi-\widetilde\bvarXi) \bigr\|_2 + \bigl\| \be_j^{\top}(\widetilde\bvarXi\bvarOmega-\bvarXi) \bigr\|_2. \end{aligned} \eeq
		
		By Weyl's inequality, $\min_{s \leq i \leq t} \widehat\sigma_i \geq \sigma_t - \bigl\|\widehat\bQ-\bQ\bigr\|_2 \geq \frac{\Delta}{2}$, thus $\Bigl\|\widehat\bvarSigma^{-1}\Bigr\|_2\leq2\Delta^{-1}$. The first term in \eqref{Row-wise Error0} satisfies
		\beq\label{Row-wise Error11} \begin{aligned} & \bigl\| \be_j^{\top}(\widehat\bvarXi - \widetilde\bvarXi) \bigr\|_2 =  \Bigl\|\be_j^{\top}(\widehat\bY-\bY)\widehat\bvarXi\widehat\bvarSigma^{-1} \Bigr\|_2 \\ \leq & \bigl\| \be_j^{\top} (\widehat\bY-\bY)\widehat\bvarXi \bigr\|_2 \Bigl\| \widehat\bvarSigma^{-1} \Bigr\|_2 \leq \frac{2}{\Delta} \bigl\| \be_j^{\top} (\widehat\bY-\bY)\widehat\bvarXi \bigr\|_2, \end{aligned} \eeq
		where we used $\widehat\bvarXi = \widehat\bY\widehat\bvarXi\widehat\bvarSigma^{-1}$. By \eqref{DavisKahan2},
		\beq\label{Row-wise Error12} \begin{aligned}
			& \bigl\| \be_j^{\top} (\widehat\bY-\bY)\widehat\bvarXi \bigr\|_2 = \bigl\| \be_j^{\top} (\widehat\bY-\bY)\widehat\bvarXi\bvarOmega \bigr\|_2 \\ \leq & \bigl\| \be_j^{\top} (\widehat\bY-\bY)\bvarXi \bigr\|_2 + \bigl\| \be_j^{\top} (\widehat\bY-\bY)\bigl(\widehat\bvarXi\bvarOmega-\bvarXi\bigr) \bigr\|_2 \\
			\leq & \bigl\| \be_j^{\top} (\widehat\bY-\bY)\bvarXi \bigr\|_2 + \bigl\|\widehat\bY-\bY\bigr\|_2\bigl\|\widehat\bvarXi\bvarOmega-\bvarXi \bigr\|_2 \\ \leq & \bigl\| \be_j^{\top} (\widehat\bY-\bY)\bvarXi \bigr\|_2 + \frac{2\sqrt{2}}{\Delta}\bigl\|\widehat\bY-\bY\bigr\|_2^2.
		\end{aligned} \eeq
		Combining \eqref{Row-wise Error11} and \eqref{Row-wise Error12}, we have
		\beq\label{Row-wise Error1} \bigl\| \be_j^{\top}(\widehat\bvarXi - \widetilde\bvarXi) \bigr\|_2 \leq \frac{2}{\Delta}\bigl\| \be_j^{\top} (\widehat\bY-\bY)\bvarXi \bigr\|_2 + \frac{4\sqrt{2}}{\Delta^2}\bigl\|\widehat\bY-\bY\bigr\|_2^2. \eeq
		
		Considering the second term in \eqref{Row-wise Error0}, we find that
		\[ \begin{aligned} \widetilde\bvarXi = \bY\widehat\bvarXi\widehat\bvarSigma^{-1} = \bigl(\bXi \bXi^{\top}\bY\bigr)\widehat\bvarXi\widehat\bvarSigma^{-1} = \bXi\bigl( \bXi^{\top}\bY\widehat\bvarXi\widehat\bvarSigma^{-1}\bigr)\end{aligned} \] and $\bvarXi = \bXi(\bXi^{\top}\bvarXi)$. Therefore,
		\beq\label{Row-wise Error21} \begin{aligned} & \bigl\|\be_j^{\top}\bigl(\widetilde\bvarXi\bvarOmega-\bvarXi\bigr)\bigr\|_2 = \Bigl\|\be_j^{\top}\bXi\bigl( \bXi^{\top}\bY\widehat\bvarXi\widehat\bvarSigma^{-1}\bvarOmega - \bXi^{\top}\bvarXi\bigr)\Bigr\|_2 \\ \leq & \bigl\|\be_j^{\top}\bXi\bigr\|_2\Bigl\| \bXi^{\top}\bigl(\bY\widehat\bvarXi\widehat\bvarSigma^{-1}\bvarOmega - \bvarXi\bigr)\Bigr\|_2 \leq \bigl\|\be_j^{\top}\bXi\bigr\|_2\Bigl\|\bY\widehat\bvarXi\widehat\bvarSigma^{-1}\bvarOmega - \bvarXi\Bigr\|_2, \end{aligned} \eeq
		where
		\[ \begin{aligned}
		& \Bigl\|\bY\widehat\bvarXi\widehat\bvarSigma^{-1}\bvarOmega - \bvarXi\Bigr\|_2 
		= \Bigl\|\bigl(\widehat\bY-(\widehat\bY-\bY)\bigr)\widehat\bvarXi\widehat\bvarSigma^{-1}\bvarOmega - \bvarXi\Bigr\|_2 \\
		= & \Bigl\|\widehat\bvarXi\bvarOmega-(\widehat\bY-\bY)\widehat\bvarXi\widehat\bvarSigma^{-1}\bvarOmega - \bvarXi\Bigr\|_2
		\leq \bigl\|\widehat\bvarXi\bvarOmega-\bvarXi\bigr\|_2 + \bigl\|(\widehat\bY-\bY)\widehat\bvarXi\widehat\bvarSigma^{-1}\bvarOmega\bigr\|_2 \\ \leq & \bigl\|\widehat\bvarXi\bvarOmega-\bvarXi\bigr\|_2 + \bigl\|(\widehat\bY-\bY)\widehat\bvarXi\bigr\|_2\Bigl\|\widehat\bvarSigma^{-1}\Bigr\|_2 \overset{\eqref{DavisKahan2}}{\leq} \frac{2(1+\sqrt{2})}{\Delta}\bigl\|\widehat\bY-\bY\bigr\|_2.
		\end{aligned} \]
		
		Plugging \eqref{Row-wise Error1} and \eqref{Row-wise Error21} into \eqref{Row-wise Error0}, we have
		\beq \label{Row-wise Error00} \begin{aligned} \bigl\|\be_j^{\top}(\widehat\bvarXi\bvarOmega-\bvarXi)\bigr\|_2 \leq & \frac{2}{\Delta}\bigl\|\be_j^{\top}(\widehat\bY-\bY)\bvarXi\bigr\|_2  + \frac{2(1+\sqrt{2})}{\Delta}\bigl\|\be_j^{\top}\bXi\bigr\|_2\bigl\|\widehat\bY-\bY\bigr\|_2 \\ & + \frac{4\sqrt{2}}{\Delta^2}\bigl\|\widehat\bY-\bY\bigr\|_2^2. \end{aligned} \eeq
		Recall the definitions of $\bXi$, $\widehat\bY$, $\bY$, $\widehat\bvarXi$ and $\bvarXi$. For $j=1,2,\ldots,p$, 
		\[ \begin{aligned}
		\be_{p+j}^{\top}(\widehat\bvarXi\bvarOmega-\bvarXi) = & \frac{1}{\sqrt{2}}\bigl[ \widehat\bh_s(j), \widehat\bh_{s+1}(j), \ldots, \widehat\bh_t(j) \bigr]\bvarOmega \\ & - \frac{1}{\sqrt{2}}\bigl[ \bh_s(j), \bh_{s+1}(j), \ldots, \bh_t(j) \bigr] = \frac{1}{\sqrt{2}}\be_j^{\top}(\widehat\bvarH\bvarOmega-\bvarH), \end{aligned} \]
		\[ \begin{aligned}
		\be_{p+j}^{\top}(\widehat\bY - \bY)\bvarXi = & \frac{1}{\sqrt{2}}\be_j^{\top}(\widehat\bQ - \bQ)^{\top}\bvarG, \text{ where } \bvarG = \bigl[\bg_s,\bg_{s+1},\ldots,\bg_t\bigr],\end{aligned}\]
		\[ \bigl\|\widehat\bY-\bY\bigr\|_2 = \bigl\|\widehat\bQ-\bQ\bigr\|_2,\]
		\[ \begin{aligned}
		\bigl\| \be_{p+j}\bXi \bigr\|_2 = &
		\Bigl\|\frac{1}{\sqrt{2}}\bigl[\bh_1(j),\ldots,\bh_r(j),-\bh_1(j),\ldots,-\bh_r(j)\bigr]\Bigr\|_2 = \sqrt{2}\bigl\|\be_j^{\top}\bH\bigr\|_2, \\ & \qquad \qquad \qquad \qquad \qquad \qquad \text{where } \bH = \bigl[\bh_1,\bh_2,\ldots,\bh_r\bigr].
		\end{aligned} \]
		Hence, \eqref{Row-wise Error00} can be reduced to \eqref{Row-wise Error}.
		
	\end{proof}
	
	\begin{lemma}[A deterministic entry-wise perturbation bound for the leading singular vector] \label{Deterministic_Leading} Let $ \delta = \sigma_1 - \sigma_2$. Suppose that $\bigl\| \widehat\bQ-\bQ\bigr\|_2 \leq \frac{\delta}{2K}$, $K \geq 1$, $\delta > 0$. Then there exists $\omega \in \bigl\{\pm1\bigr\}$ such that
		\beq \label{Leading}\begin{aligned}
			\bigl| \omega\widehat\bh_1(j) - \bh_1(j) \bigr| \leq & \sum_{k=1}^{\lceil K/2 \rceil} \Bigl(\sigma_1^{-2k}\Bigl|\be_j^{\top}\bQ^{\top}(\bQ\bQ^{\top})^{k-1}(\widehat\bQ-\bQ)\bh_1\Bigr| \\ & \quad + \sigma_1^{-2k-1}\Bigl|\bg_1^{\top}(\widehat\bQ-\bQ)(\bQ^{\top}\bQ)^k\be_j\Bigr| \Bigr) \\ & + |\bh_1(j)| \cdot (2K+1+2\sqrt{2})\delta^{-1}\bigl\| \widehat\bQ-\bQ \bigr\|_2 \\ & + 4\sqrt{2}\Bigl(\frac{2\sigma_2}{\sigma_1+\sigma_2}\Bigr)^K \bigl\| \be_j^{\top} \bH\bigr\|_2 \delta^{-1}\bigl\| \widehat\bQ - \bQ \bigr\|_2 \\
			& + \bigl(2K^2+8\sqrt{2}K\bigr)\delta^{-2}\bigl\|\widehat\bQ-\bQ\bigr\|_2^2.
		\end{aligned}\eeq
	\end{lemma}
	
	\begin{proof}
		In the following, we use the same notations $\widehat\bY$, $\bY$, $\widehat\bxi_i$, $\bxi_i$, $\widehat\bxi_{-i}$ and $\bxi_{-i}$ as in the proof of Lemma \ref{Deterministic Row-wise Error}, and denote $\delta = \sigma_1 - \sigma_2$.  Let
		\[ \begin{aligned} & \bSigma_{\perp} = \diag(\sigma_2,\sigma_3,\ldots,\sigma_r,-\sigma_2,-\sigma_3,\ldots,-\sigma_r), \\ & \bXi_{\perp} = \bigl[\bxi_2,\bxi_3, \ldots, \bxi_r, \bxi_{-2}, \bxi_{-3}, \ldots, \bxi_{-r}\bigr]. \end{aligned} \]
		Then one has
		\beq\label{EigenDecompose} \begin{aligned} & \bY = \sigma_1 \bxi_1 \bxi_1^{\top} - \sigma_1 \bxi_{-1} \bxi_{-1}^{\top} + \bXi_{\perp} \bSigma_{\perp} \bXi_{\perp}^{\top}. \end{aligned} \eeq
		
		Inspired by power iteration, we decompose the difference $\omega\widehat\bxi_1 - \bxi_1$ by
		\beq\label{PowerIteration1} \begin{aligned} & \omega\widehat\bxi_1 - \bxi_1 \\ = & \omega\bigl(\widehat\bxi_1 - \bY \widehat\bxi_1 \widehat\sigma_1^{-1}\bigr) + \bigl( \omega\bY \widehat\bxi_1 \widehat\sigma_1^{-1} - \bxi_1 \bigr) \\ = & \omega\bigl(\widehat\bY - \bY \bigr) \widehat\bxi_1 \widehat\sigma_1^{-1} + \bigl( \omega \bY \widehat\bxi_1 \widehat\sigma_1^{-1} - \bxi_1 \bigr) \\
			= & \omega\bigl(\widehat\bY - \bY \bigr) \widehat\bxi_1 \widehat\sigma_1^{-1} + \omega \bigl( \bY \widehat\bxi_1 \widehat\sigma_1^{-1} - \bY^2 \widehat\bxi_1 \widehat\sigma^{-1}\bigr) + \bigl(\omega\bY^2 \widehat\bxi_1 \widehat\sigma_1^{-1} - \bxi_1 \bigr) \\ = &  \omega\bigl(\widehat\bY - \bY \bigr)\widehat\bxi_1 \widehat\sigma_1^{-1} + \omega \bY \bigl( \widehat\bY- \bY \bigr)\widehat\bxi_1 \widehat\sigma_1^{-2} + \bigl(\omega\bY^2 \widehat\bxi_1 \widehat\sigma_1^{-2} - \bxi_1 \bigr) \\ = &  \omega\widehat\sigma_1^{-1}\sum_{k=0}^1 \bigl(\widehat\sigma_1^{-1}\bY\bigr)^k \bigl(\widehat\bY - \bY \bigr)\widehat\bxi_1 + \bigl(\omega\bY^2 \widehat\bxi_1 \widehat\sigma_1^{-2} - \bxi_1 \bigr) \\ = & \ldots \\ = & \omega\widehat\sigma_1^{-1}\sum_{k=0}^{K-1} \bigl(\widehat\sigma_1^{-1}\bY\bigr)^k  \bigl(\widehat\bY - \bY \bigr)\widehat\bxi_1 + \bigl(\omega\bY^K \widehat\bxi_1 \widehat\sigma_1^{-K} - \bxi_1 \bigr), \end{aligned} \eeq
		where we repeatedly used the fact that $\widehat\bY \widehat\bxi_1 \widehat\sigma_1^{-1} = \widehat\bxi_1$.
		
		Plugging \eqref{EigenDecompose} into the second term in \eqref{PowerIteration1} and using \[ \bxi_{-1}^{\top}\bxi_1=0, \qquad \bXi_{\perp}^{\top}\bxi_1={\bf 0}_{2r-2},\] one obtains
		\[ \begin{aligned} \omega\bY^K \widehat\bxi_1 \widehat\sigma_1^{-K} - \bxi_1
		= & \bigl((\sigma_1/\widehat\sigma_1)^K\bigl(\omega\bxi_1^{\top}\widehat\bxi_1\bigr)-1 \bigr)\bxi_1 \\ & + (-\sigma_1/\widehat\sigma_1)^K\bxi_{-1}\bxi_{-1}^{\top} (\omega\widehat\bxi_1-\bxi_1) \\ & + \bXi_{\perp} \bigl( \widehat\sigma_1^{-1}\bSigma_{\perp}\bigr)^K \bXi_{\perp}^{\top} (\omega\widehat\bxi_1-\bxi_1). \end{aligned} \]
		For a fixed $j = 1,2,\ldots, 2p$, it follows from \eqref{PowerIteration1} that
		\beq \label{NormDecompose1} \begin{aligned} 
			\bigl| \omega\widehat\bxi_1(j) - \bxi(j) \bigr| \leq & \sigma_1^{-1} \sum_{k=1}^K(\sigma_1/\widehat\sigma_1)^k\Bigl| \be_j^{\top} \bigl( \sigma_1^{-1} \bY \bigr)^{k-1} \bigl( \widehat\bY - \bY \bigr) \widehat\bxi_1 \Bigr| \\ & + |\bxi_1(j)| \cdot \bigl| (\sigma_1/\widehat\sigma_1)^K \bigl( \omega \bxi_1^{\top} \widehat\bxi_1 \bigr) - 1 \bigr|  \\
			& + (\sigma_1/\widehat\sigma_1)^K |\bxi_1(j)| \cdot \bigl\| \omega \widehat\bxi_1 - \bxi_1 \bigr\|_2 \\ & + \bigl(\sigma_2/\widehat\sigma_1\bigr)^K \bigl\| \be_j^{\top} \bXi_{\perp} \bigr\|_2 \bigl\| \omega\widehat\bxi_1 - \bxi_1 \bigr\|_2,
		\end{aligned} \eeq
		where we used $\bigl|\bxi_{-1}(j)\bigr| = \bigl|\bxi_1(j)\bigr|$.
		
		$\omega\bxi_1^{\top}\widehat\bxi_1$ and $\bigl\|\omega\widehat\bxi_1-\bxi_1\bigr\|_2$ in \eqref{NormDecompose1} represent the difference between $\widehat\bxi_1$ and $\bxi_1$.
		According to the Davis-Kahan $\sin\theta$ theorem,
		\[ \sin\bigl(\Theta(\bxi_1,\widehat\bxi_1)\bigr) \leq \delta^{-1}\bigl\|\widehat\bY-\bY\bigr\|_2 = \bigl\|\widehat\bQ-\bQ\bigr\|_2, \]
		where $\Theta(\bxi_1,\widehat\bxi_1)$ denotes the principal angle between $\bxi_1$ and $\widehat\bxi_1$. It follows that there exists $\omega \in \{\pm 1\}$ such that
		\beq \label{DavisKahan21} \begin{aligned} & \omega \bxi_1^{\top}\widehat\bxi_1 = \cos\bigl(\Theta(\bxi_1,\widehat\bxi_1)\bigr) = \sqrt{1-\sin\bigl(\Theta(\bxi_1,\widehat\bxi_1)\bigr)} \geq \sqrt{1-\delta^{-1}\bigl\|\widehat\bQ-\bQ\bigr\|_2}. \end{aligned} \eeq
		Similar as \eqref{DavisKahan2}, we also have
		\beq \label{DavisKahan22} \bigl\|\omega\widehat\bxi_1-\bxi_1\bigr\|_2 \leq 2\sqrt{2} \delta^{-1} \bigl\|\widehat\bQ-\bQ\bigr\|_2. \eeq
		
		We now focus on the terms involving $\widehat\sigma_1$.
		Note that for all $K \geq 1$ and $x \in \bigl[0,\frac{1}{2K}\bigr]$.
		\[ \bigl| (1+x)^K - 1 \bigr| \leq 2K|x|, \quad \bigl| (1+x)^{-K} - 1 \bigr| \leq 2K|x|.\]
		By Weyl's inequality,
		$ \bigl|\widehat\sigma_1 - \sigma_1\bigr| \leq \bigl\| \widehat\bQ - \bQ \bigr\|_2$.
		Therefore, under the condition $\bigl\|\widehat\bQ-\bQ\bigr\|_2 \leq \frac{\delta}{2K}$,
		\beq \label{Ratio} \begin{aligned} & (\sigma_1/\widehat\sigma_1)^k = \Bigl(1+\frac{\widehat\sigma_1-\sigma_1}{\sigma_1}\Bigr)^{-k} \leq 1 + 2k \delta^{-1} \bigl\|\widehat\bQ-\bQ\bigr\|_2 \leq 1 + 2K \delta^{-1} \bigl\|\widehat\bQ-\bQ\bigr\|_2, \\ & \qquad \qquad\qquad \qquad\qquad \qquad\qquad \qquad\qquad \qquad\qquad \qquad \text{for $k=1,2,\ldots,K$}, \\ & (\sigma_2/\widehat\sigma_1)^K \leq \biggl(\frac{\sigma_2}{\sigma_1 - \bigl\|\widehat\bQ-\bQ\bigr\|_2}\biggr)^K \leq \biggl(\frac{\sigma_2}{\sigma_1 - \delta/2}\biggr)^K \leq \Bigl(\frac{2\sigma_2}{\sigma_1+\sigma_2}\Bigr)^K, \\  & (\sigma_1/\widehat\sigma_1)^K\bigl(\omega\bxi_1^{\top}\widehat\bxi_1\bigr) - 1 \leq 2K\delta^{-1}\bigl\|\widehat\bQ-\bQ\bigr\|_2, \\
			& (\sigma_1/\widehat\sigma_1)^K\bigl(\omega\bxi_1^{\top}\widehat\bxi_1\bigr) - 1 \overset{\eqref{DavisKahan21}}{\geq} \bigl(1 - 2K\delta^{-1}\bigl\|\widehat\bQ-\bQ\bigr\|_2\bigr)\sqrt{1-\delta^{-1}\bigl\|\widehat\bQ-\bQ\bigr\|_2} - 1 \\ \geq & \bigl(1 - 2K\delta^{-1}\bigl\|\widehat\bQ-\bQ\bigr\|_2\bigr)\bigl(1-\delta^{-1}\bigl\|\widehat\bQ-\bQ\bigr\|_2\bigr) - 1 \geq -(2K+1)\delta^{-1}\bigl\|\widehat\bQ-\bQ\bigr\|_2. 
		\end{aligned} \eeq
		
		Consider the first term in the right hand side of \eqref{NormDecompose1}. For $k=1,2,\ldots,K$,
		\[  \begin{aligned} & \sigma_1^{-1}\Bigl| \be_j^{\top} \bigl( \sigma_1^{-1} \bY \bigr)^{k-1} \bigl( \widehat\bY - \bY \bigr) \widehat\bxi_1 \Bigr| \\ \leq & \sigma_1^{-1}\Bigl| \be_j^{\top} \bigl( \sigma_1^{-1} \bY \bigr)^{k-1} \bigl( \widehat\bY - \bY \bigr) \bxi_1 \Bigr| + \sigma_1^{-1}\Bigl| \be_j^{\top} \bigl( \sigma_1^{-1} \bY \bigr)^{k-1} \bigl( \widehat\bY - \bY \bigr) \bigl(\omega\widehat\bxi_1 - \bxi_1\bigr) \Bigr| \\ \leq & \sigma_1^{-1} \Bigl| \be_j^{\top} \bigl( \sigma_1^{-1} \bY \bigr)^{k-1} \bigl( \widehat\bY - \bY \bigr) \bxi_1 \Bigr| + \sigma^{-1} \bigl\| \widehat\bY - \bY \bigr\|_2 \bigl\|\omega\widehat\bxi_1 - \bxi_1\bigr\|, \end{aligned} \]
		and by \eqref{DavisKahan22} and \eqref{Ratio},
		\beq\label{FirstTerm} \begin{aligned} & (\sigma_1/\widehat\sigma_1)^k \cdot \sigma_1^{-1}\Bigl| \be_j^{\top} \bigl( \sigma_1^{-1} \bY \bigr)^{k-1} \bigl( \widehat\bY - \bY \bigr) \widehat\bxi_1 \Bigr| \\ \leq & \Bigl(1 + 2K \delta^{-1} \bigl\|\widehat\bQ-\bQ\bigr\|_2\Bigr) \cdot \sigma_1^{-1}\biggl( \Bigl| \be_j^{\top} \bigl( \sigma_1^{-1} \bY \bigr)^{k-1} \bigl( \widehat\bY - \bY \bigr) \bxi_1 \Bigr| + 2\sqrt{2}\delta^{-1}\bigl\|\widehat\bQ-\bQ\bigr\|_2^2 \biggr) \\ \leq & \sigma_1^{-1} \Bigl| \be_j^{\top} \bigl( \sigma_1^{-1} \bY \bigr)^{k-1} \bigl( \widehat\bY - \bY \bigr) \bxi_1 \Bigr| + \bigl(2K+2\sqrt{2}\bigr)\delta^{-2}\bigl\|\widehat\bQ-\bQ\bigr\|_2^2 + 4\sqrt{2}\delta^{-3}\bigl\|\widehat\bQ-\bQ\bigr\|_2^3 \\ \leq & \sigma_1^{-1} \Bigl| \be_j^{\top} \bigl( \sigma_1^{-1} \bY \bigr)^{k-1} \bigl( \widehat\bY - \bY \bigr) \bxi_1 \Bigr| + \bigl(2K+4\sqrt{2}\bigr)\delta^{-2}\bigl\|\widehat\bQ-\bQ\bigr\|_2^2, \end{aligned} \eeq
		where we used $\delta^{-1}\bigl\|\widehat\bQ-\bQ\bigr\|_2\leq \frac{1}{2K}$.
		
		Plugging \eqref{DavisKahan22}, \eqref{Ratio} and \eqref{FirstTerm} into \eqref{NormDecompose1} gives
		\beq \label{NormDecompose2} \begin{aligned} 
			& \bigl| \omega\widehat\bxi_1(j) - \bxi_1(j) \bigr| \\ \leq & \sigma_1^{-1} \sum_{k=1}^K\Bigl| \be_j^{\top} \bigl( \sigma_1^{-1} \bY \bigr)^{k-1} \bigl( \widehat\bY - \bY \bigr) \bxi_1 \Bigr| + \bigl(2K^2+4\sqrt{2}K\bigr)\delta^{-2}\bigl\|\widehat\bQ-\bQ\bigr\|_2^2 \\ & + |\bxi_1(j)| \cdot (2K+1)\delta^{-1}\bigl\|\widehat\bQ-\bQ\bigr\|_2 \\ & + 2\sqrt{2} |\bxi_1(j)| \cdot \Bigl(1 + 2K \delta^{-1} \bigl\|\widehat\bQ-\bQ\bigr\|_2\Bigr) \delta^{-1}\bigl\|\widehat\bQ - \bQ\bigr\|_2 \\ & + 2\sqrt{2}\Bigl(\frac{2\sigma_2}{\sigma_1+\sigma_2}\Bigr)^K \bigl\| \be_j^{\top} \bXi \bigr\|_2 \delta^{-1}\bigl\| \widehat\bQ - \bQ \bigr\|_2, \\
			= & \sigma_1^{-1} \sum_{k=1}^K\Bigl| \be_j^{\top} \bigl( \sigma_1^{-1} \bY \bigr)^{k-1} \bigl( \widehat\bY - \bY \bigr) \bxi_1 \Bigr| + |\bxi_1(j)| \cdot (2K+1+2\sqrt{2})\delta^{-1}\bigl\|\widehat\bQ-\bQ\bigr\|_2 \\ & + 2\sqrt{2}\Bigl(\frac{2\sigma_2}{\sigma_1+\sigma_2}\Bigr)^K \bigl\| \be_j^{\top} \bXi\bigr\|_2 \delta^{-1}\bigl\| \widehat\bQ - \bQ \bigr\|_2 + \bigl(2K^2+8\sqrt{2}K\bigr)\delta^{-2}\bigl\|\widehat\bQ-\bQ\bigr\|_2^2.
		\end{aligned} \eeq
		
		For $k=0,1,\ldots,K-1$, let
		\[ \bigl(\sigma_1^{-1}\bY\bigr)^k \be_{p+j} = \left[ \begin{array}{c} \bupsilon_{1,k} \\ \bupsilon_{2,k} \end{array} \right], \quad \bupsilon_{1,k}, \bupsilon_{2,k} \in \mathbb{R}^p. \]
		Since
		\[ \bY^k = \left\{\begin{aligned} & \left[ \begin{array}{cc} {\bf 0} & \bQ(\bQ^{\top}\bQ)^{\frac{k-1}{2}} \\ \bQ^{\top}(\bQ\bQ^{\top})^{\frac{k-1}{2}} & {\bf 0} \end{array} \right], \\ & ~~~~~~~~~~~~~~~~~~~~~~~~~~~~~~~~~~~~~~~~~~~ \text{if $k$ is odd}, \\ & \left[ \begin{array}{cc} (\bQ\bQ^{\top})^{\frac{k}{2}} & {\bf 0} \\ {\bf 0} & (\bQ^{\top}\bQ)^{\frac{k}{2}} \end{array} \right], \\ & ~~~~~~~~~~~~~~~~~~~~~~~~~ \text{if $k$ is even}, \end{aligned}\right. \]
		\[ \begin{aligned} & \bupsilon_{1,k} = \sigma_1^{-k}\bQ(\bQ^{\top}\bQ)^{\frac{k-1}{2}}\be_j, \ \bupsilon_{2,k} = {\bf 0}, && \text{when $k$ is odd}, \\ & \bupsilon_{1,k} = {\bf 0}, \ \bupsilon_{2,k} = \sigma_1^{-k}(\bQ^{\top}\bQ)^{\frac{k}{2}}\be_j, && \text{when $k$ is even}. \end{aligned} \]
		Recall the definitions of $\widehat\bY$, $\bY$, $\bXi$ and $\bxi_1$. For $j=1,2,\ldots,p$, we reduce \eqref{NormDecompose2} into \eqref{Leading}.
	\end{proof}

	\subsection{Markov Chain Concentration Inequalities}
	
	\begin{lemma}[Lemma 7 in \cite{zhang2018spectral}, Markov chain concentration inequality] \label{MC_Concentration} Suppose $\bP \in \mathbb{R}^{p \times p}$ is an ergodic Markov chain transition matrix on $p$ states $\{1,\cdots,p\}$. $\bP$ is with invariant distribution $\bpi$ and the Markov mixing time \[ \tau(\varepsilon) = \min\biggl\{k : \max_{1 \leq i \leq p} \frac{1}{2}\|\be_i^{\top}\bP^k - \bpi^{\top} \|_1 \leq \varepsilon \biggr\}.\] Recall the frequency matrix is $\bF = [\diag(\bpi)]\bP$. Give a Markov trajectory with $(n+1)$ observable states $X = \{X_0,X_1,\cdots,X_n\}$ from any initial state.
		Let $\tau_*=\tau(1/4)$. For any constant $c_0 > 0$, there exists a constant $C>0$ such that if $n \geq C \tau_* p \log^2(n)$, then \vspace{-0.05cm}
		\begin{align}
		& \mathbb{P} \Bigl( \bigl\|n^{-1}\bN - \bF\bigr\|_2 \geq C n^{-1/2} \sqrt{\pi_{\max} \tau_* \log^2(n)} \Bigr) \leq n^{-c_0}, \label{MC_Concentration_F} \\
		& \mathbb{P} \Bigl( \|n^{-1}\bm - \bpi\|_{\infty} \geq C n^{-1/2} \sqrt{\pi_{\max} \tau_* \log^2(n)} \Bigr) \leq n^{-c_0}. \label{MC_Concentration_pi}
		\end{align}
	\end{lemma}
	
	\begin{lemma}[Row-wise Markov chain concentration inequality] \label{Row-wise_Concentration}
		Suppose $\bmu \in \mathbb{R}^p$ is a fixed unit vector and $\bK \in \mathbb{O}^{p \times s}$ is a fixed orthonormal matrix satisfying \[ \max_{1 \leq i \leq p} \bigl\| \be_i^{\top}\bK\bigr\|_2 \leq \gamma \sqrt{s/p}.\] Under the same setting as in Lemma \ref{MC_Concentration}, for any $c_0 > 0$, there exists $C > 0$ such that if $n \geq Cp\tau_*\log^2(n)$, then
		\beq \begin{aligned}
			\mathbb{P}\Bigl(\bigl\|\bmu^{\top}\bigl(n^{-1}\widehat\bN - \bF\bigr)[\diag(\bpi)]^{-\frac{1}{2}}\bK\bigr\|_2  \geq C n^{-1/2}\sqrt{\tau_*(\gamma^2s/p)\log^2(n)}\Bigr) \leq n^{-c_0}, \\
			\mathbb{P}\Bigl(\bigl\|\bK^{\top}\bigl(n^{-1}\bN - \bF\bigr)[\diag(\bpi)]^{-\frac{1}{2}}\bmu\bigr\|_2 \geq C n^{-1/2}\sqrt{\tau_*(\gamma^2s/p)\log^2(n)}\Bigr) \leq n^{-c_0}.
		\end{aligned} \eeq
	\end{lemma}
	
	\begin{proof}
		We follow the proof of Lemma 8 in \cite{zhang2018spectral}. For any $t>0$, let \[ \alpha = \tau\Bigl(\pi_{\min}\wedge\frac{\sqrt{c_1}t}{2\gamma\sqrt{s}}\Bigr)+1,\] and $n_0 = \lfloor n / \alpha \rfloor$. Without loss of generality, assume $n$ is a multiple of $\alpha$. By definition, \[ n^{-1}\bK^{\top}\bN[\diag(\bpi)]^{-\frac{1}{2}}\bmu = \frac{1}{n}\sum_{l=1}^\alpha\sum_{k=1}^{n_0} \bT_k^{(l)},\] \[ \bT_k^{(l)} = \bK^{\top}\be_{X_{k\alpha+l-1}}\be_{X_{k\alpha+l}}^{\top}[\diag(\bpi)]^{-\frac{1}{2}}\bmu, \]
		and
		\[ n^{-1}\bmu^{\top}\bN[\diag(\bpi)]^{-\frac{1}{2}}\bK = \frac{1}{n}\sum_{l=1}^\alpha\sum_{k=1}^{n_0} \bS_k^{(l)},\] \[ \bS_k^{(l)} = \bmu^{\top}\be_{X_{k\alpha+l-1}}\be_{X_{k\alpha+l}}^{\top}[\diag(\bpi)]^{-\frac{1}{2}}\bK. \]
		We introduce the ``thin'' sequences as
		\[ \begin{aligned} \widetilde\bT_k^{(l)} = \bT_k^{(l)} - \mathbb{E}\bigl[ \bT_k^{(l)} \big| X_{(k-1)\alpha+l} \bigr], \qquad \widetilde\bS_k^{(l)} = \bS_k^{(l)} - \mathbb{E}\bigl[ \bS_k^{(l)} \big| X_{(k-1)\alpha+l} \bigr], \end{aligned} \]
		for $l = 1,\ldots,\alpha$, $ k = 1,\ldots,n_0$,
		and apply the matrix Freedman's inequality \cite{Freedman} to derive concentration inequalities of partial sum sequences $\sum_{k=1}^{n_0} \widetilde\bT_k^{(l)}$ and $\sum_{k=1}^{n_0} \widetilde\bS_k^{(l)}$ for a fixed $l$.
		
		Consider the predictable quadratic variantion processes of the martingales $\Bigl\{ \sum_{k=1}^t \widetilde\bT_k^{(l)} \Bigr\}_{t=1}^{n_0}$ and $\Bigl\{ \sum_{k=1}^t \widetilde\bS_k^{(l)} \Bigr\}_{t=1}^{n_0}$: \[\bX_{1,t} ^{(l)} = \sum_{k=1}^t\mathbb{E} \Bigl[ \widetilde\bT_k^{(l)} \bigl(\widetilde\bT_k^{(l)}\bigr)^{\top} \Big| X_{(k-1)\alpha+l} \Bigr], \quad \bX_{2,t}^{(l)} = \sum_{k=1}^t \mathbb{E} \Bigl[ \bigl(\widetilde\bT_k^{(l)}\bigr)^{\top} \widetilde\bT_k^{(l)}  \Big| X_{(k-1)\alpha+l} \Bigr], \]
		and
		\[\bZ_{1,t} ^{(l)} = \sum_{k=1}^t\mathbb{E} \Bigl[ \widetilde\bS_k^{(l)} \bigl(\widetilde\bS_k^{(l)}\bigr)^{\top} \Big| X_{(k-1)\alpha+l} \Bigr], \quad \bZ_{2,t}^{(l)} = \sum_{k=1}^t \mathbb{E} \Bigl[ \bigl(\widetilde\bS_k^{(l)}\bigr)^{\top} \widetilde\bS_k^{(l)}  \Big| X_{(k-1)\alpha+l} \Bigr], \]
		for $t=1,2,\ldots,n_0$. Since
		\[ \begin{aligned}
		\bigl\| \bX_{1,t}^{(l)} \bigr\|_2 \leq & \sum_{k=1}^t \Bigl\| \mathbb{E} \Bigl[ \widetilde\bT_k^{(l)} \bigl(\widetilde\bT_k^{(l)}\bigr)^{\top} \Big| X_{(t-1)\alpha+l} \Bigr] \Bigr\|_2 \leq \sum_{k=1}^t \mathbb{E} \Bigl[ \Bigl\| \widetilde\bT_k^{(l)} \bigl(\widetilde\bT_k^{(l)}\bigr)^{\top}\Bigr\|_2 \Big| X_{(k-1)\alpha+l} \Bigr] \\ = &  \sum_{k=1}^t \mathbb{E} \Bigl[ \bigl\| \widetilde\bT_k^{(l)} \bigr\|_2^2 \Big| X_{(k-1)\alpha+l} \Bigr] = \bX_{2,t}^{(l)},
		\end{aligned} \]
		and $\bigl\| \bZ_{2,t}^{(l)} \bigr\|_2 \leq \bZ_{1,t}^{(l)}$ for the same reason, we only need to focus on $\bX_{2,t}^{(l)}$ and $\bZ_{1,t}^{(l)}$. 
		
		Denote $\widetilde\bpi^{(k,l)} \equiv (\bP^{\alpha-1})^{\top}\be_{X_{(k-1)\alpha+l-1}}$. By the definition of mixing time $\alpha$, \[ \max_{1 \leq s \leq p} \frac{1}{2} \bigl\|\be_s^{\top}\bP^{\alpha-1}-\bpi^{\top}\bigr\|_1 \leq \pi_{\min} \wedge \frac{\sqrt{c_1}t}{2\gamma\sqrt{s}}.\] Hence, $\widetilde\pi_i^{(k,l)} \leq \pi_i + \pi_{\min} \leq 2\pi_i$ for $i=1,\ldots,p$. We have
		\[ \begin{aligned}
		& \mathbb{E} \Bigl[ \bigl(\widetilde\bT_k^{(l)}\bigr)^{\top} \widetilde\bT_k^{(l)}  \Big| X_{(k-1)\alpha+l} \Bigr] \leq \mathbb{E} \Bigl[ \bigl\|\bT_k^{(l)}\bigr\|_2^2  \Big| X_{(k-1)\alpha+l} \Bigr] \\ = & \mathbb{E} \Bigl[ \bigl\|\be_{X_{k\alpha+l-1}}^{\top}\bK\bigr\|_2^2 \bigl(\be_{X_{k\alpha+l}}^{\top}[\diag(\bpi)]^{-\frac{1}{2}}\bmu\bigr)^2 \Big| X_{(k-1)\alpha+l} \Bigr] \\ = & \sum_{i=1}^p\sum_{j=1}^p \mathbb{P}\bigl[ X_{k\alpha+l-1} = i \big| X_{(k-1)\alpha+l}\bigr] \cdot\mathbb{P}\bigl[ X_{k\alpha+l} = j \big| X_{k\alpha+l-1} = i \bigr] \bigl\|\be_i^{\top} \bK\bigr\|_2^2 \pi_j^{-1} \eta_j^2 \\ = & \sum_{i=1}^p \sum_{j=1}^p \widetilde{\pi}_i^{(k,l)} P_{ij}\bigl\|\be_i^{\top} \bK\bigr\|_2^2 \pi_j^{-1} \eta_j^2 \leq \sum_{j=1}^p 2\Bigl(\sum_{i=1}^p\pi_iP_{ij}\Bigr) \pi_j^{-1} \eta_j^2 \cdot \gamma^2 sp^{-1} \\ = & \sum_{j=1}^p 2 \eta_j^2 \cdot \gamma^2 sp^{-1} = 2\gamma^2sp^{-1}, \quad \text{for $k=1,\ldots,n_0$,}
		\end{aligned} \]
		where we used $\bpi^{\top}\bP = \bpi^{\top}$, $\sum_{i=1}^p\pi_iP_{ij} = \pi_j$. Similarly,
		\[ \begin{aligned}
		& \mathbb{E} \Bigl[  \widetilde\bS_k^{(l)} \bigl(\widetilde\bS_k^{(l)}\bigr)^{\top} \Big| X_{(k-1)\alpha+l} \Bigr] \leq \mathbb{E} \Bigl[ \bigl\|\bS_k^{(l)}\bigr\|_2^2  \Big| X_{(k-1)\alpha+l} \Bigr] \\ = & \mathbb{E} \Bigl[ \bigl(\bmu^{\top}\be_{X_{k\alpha+l-1}}\bigr)^2 \bigl\|\be_{X_{k\alpha+l}}^{\top}[\diag(\bpi)]^{-\frac{1}{2}}\bK\bigr\|_2^2 \Big| X_{(k-1)\alpha+l} \Bigr] \\ = & \sum_{i=1}^p\sum_{j=1}^p \mathbb{P}\bigl[ X_{k\alpha+l-1} = i \big| X_{(k-1)\alpha+l}\bigr] \cdot\mathbb{P}\bigl[ X_{k\alpha+l} = j \big| X_{k\alpha+l-1} = i \bigr] \bigl\|\be_j^{\top} \bK\bigr\|_2^2 \pi_j^{-1} \eta_i^2 \\ = & \sum_{i=1}^p \sum_{j=1}^p \widetilde{\pi}_i^{(k,l)} P_{ij}\bigl\|\be_j^{\top} \bK\bigr\|_2^2 \pi_j^{-1} \eta_i^2 \leq \sum_{i=1}^p \sum_{j=1}^p 2 (\pi_i/\pi_j)P_{ij}\eta_i^2 \cdot \gamma^2 sp^{-1} \\ \leq & 2 C_1c_1^{-1} \sum_{i=1}^p \eta_i^2 \sum_{j=1}^p P_{ij} \cdot \gamma^2 sp^{-1} = 2 C_1c_1^{-1} \gamma^2sp^{-1}, \qquad \text{for $k=1,\ldots,n_0$.}
		\end{aligned} \]
		
		Note that \[ \begin{aligned} & \begin{aligned} \bigl\|\bT_k^{(l)}\bigr\|_2 \leq \max_{1 \leq i \leq p} \pi_{\min}^{-\frac{1}{2}}\bigl\|\be_i^{\top}\bK\bigr\|_2 \leq \gamma\sqrt{s/(\pi_{\min}p)} \leq c_1^{-\frac{1}{2}}\gamma \sqrt{s},\end{aligned} \\ & \begin{aligned} \bigl\| \mathbb{E} \bigl[ \bT_k^{(l)} \big| X_{(k-1)\alpha+l} \bigr] \bigr\|_2 \leq \mathbb{E} \bigl[ \bigl\|\bT_k^{(l)}\bigr\|_2 \big| X_{(k-1)\alpha+l} \bigr] \leq c_1^{-\frac{1}{2}}\gamma \sqrt{s}. \end{aligned} \end{aligned} \] Therefore, $\bigl\| \widetilde\bT_k^{(l)} \bigr\|_2 \leq 2c_1^{-\frac{1}{2}}\gamma \sqrt{s}$. Similarly, we have $\bigl\| \widetilde\bS_k^{(l)} \bigr\|_2 \leq 2c_1^{-\frac{1}{2}}\gamma \sqrt{s}$. The matrix Freedman's inequality \cite{Freedman} implies
		\[ \begin{aligned} & \mathbb{P} \left( \Bigl\| n_0^{-1}\sum_{k=1}^{n_0} \widetilde\bT_k^{(l)} \Bigr\|_2 \geq \frac{t}{2} \right) \leq (s+1)\exp\left(-\frac{n_0^2t^2/8}{2n_0\gamma^2sp^{-1}+n_0c_1^{-\frac{1}{2}}\gamma t\sqrt{s}/3}\right) \end{aligned} \]
		and
		\[ \begin{aligned} & \mathbb{P} \left( \Bigl\| n_0^{-1}\sum_{k=1}^{n_0} \widetilde\bS_k^{(l)} \Bigr\|_2 \geq \frac{t}{2} \right) \leq (s+1)\exp\left(-\frac{n_0^2t^2/8}{2n_0C_1c_1^{-1}\gamma^2sp^{-1}+n_0c_1^{-\frac{1}{2}}\gamma t\sqrt{s}/3}\right). \end{aligned} \]
		By union bound,
		\beq \label{term1} \begin{aligned}
			& \begin{aligned} 
				\mathbb{P} \Bigg( \Bigl\| & n^{-1}\bK^{\top}\bN[\diag(\bpi)]^{-\frac{1}{2}}\bmu - n^{-1} \sum_{l=1}^{\alpha} \sum_{k=1}^{n_0} \mathbb{E}\bigl[ \bT_k^{(l)} \big| X_{(k-1)\alpha+l} \bigr] \Bigr\|_2 \geq \frac{t}{2} \Bigg) 
			\end{aligned} \\
			= & \mathbb{P} \left( \Bigl\| n^{-1} \sum_{l=1}^{\alpha} \sum_{k=1}^{n_0} \widetilde{\bT}_k^{(l)} \Bigr\|_2 \geq \frac{t}{2} \right)
			\leq \sum_{l=1}^{\alpha} \mathbb{P} \left( \Bigl\| n_0^{-1}\sum_{k=1}^{n_0} \widetilde\bT_k^{(l)} \Bigr\|_2 \geq \frac{t}{2} \right) \\
			\leq & \alpha (s+1) \exp\left(-\frac{n_0t^2/8}{2\gamma^2sp^{-1}+c_1^{-\frac{1}{2}}\gamma t\sqrt{s}/3}\right).
		\end{aligned} \eeq
		Similarly,
		\beq \label{term1'} \begin{aligned}
			& \begin{aligned} 
				\mathbb{P} \Bigg( \Bigl\| & n^{-1}\bmu^{\top}\bN[\diag(\bpi)]^{-\frac{1}{2}}\bK - n^{-1} \sum_{l=1}^{\alpha} \sum_{k=1}^{n_0} \mathbb{E}\bigl[ \bS_k^{(l)} \big| X_{(k-1)\alpha+l} \bigr] \Bigr\|_2 \geq \frac{t}{2} \Bigg) 
			\end{aligned} \\
			\leq & \alpha (s+1) \exp\left(-\frac{n_0t^2/8}{2C_1c_1^{-1}\gamma^2sp^{-1}+c_1^{-\frac{1}{2}}\gamma t\sqrt{s}/3}\right).
		\end{aligned} \eeq
		
		We next analyze the differences \[ n^{-1}\sum_{l=1}^{\alpha}\sum_{k=1}^{n_0} \mathbb{E}\bigl[ \bT_k^{(l)} \big| X_{(k-1)\alpha+l} \bigr] - \bK^{\top}\bF[\diag(\bpi)]^{-\frac{1}{2}}\bmu\]
		and
		\[ n^{-1}\sum_{l=1}^{\alpha}\sum_{k=1}^{n_0} \mathbb{E}\bigl[ \bS_k^{(l)} \big| X_{(k-1)\alpha+l} \bigr] - \bmu^{\top}\bF[\diag(\bpi)]^{-\frac{1}{2}}\bK.\]
		Since 
		\[ \begin{aligned} & \begin{aligned}  \mathbb{E}\bigl[ \bT_k^{(l)} \big| X_{(k-1)\alpha+l} \bigr] = & \sum_{i=1}^p\sum_{j=1}^p \widetilde{\pi}_i^{(k,l)}P_{ij} \bK^{\top}\be_i\be_j^{\top}[\diag(\bpi)]^{-\frac{1}{2}}\bmu \\ = & \sum_{i=1}^p\sum_{j=1}^p \widetilde{\pi}_i^{(k,l)}P_{ij} \bK^{\top}\be_i\pi_j^{-\frac{1}{2}}\eta_j, \end{aligned} \\
		& \begin{aligned} \bK^{\top}\bF[\diag(\bpi)]^{-\frac{1}{2}}\bmu = & \sum_{i=1}^p\sum_{j=1}^p \bK^{\top} \bigl( \pi_i \be_i\be_i^{\top} \bigr) \bP(\be_j\be_j^{\top})[\diag(\bpi)]^{-\frac{1}{2}}\bmu \\ = & \sum_{i=1}^p\sum_{j=1}^p \pi_i P_{ij}\bK^{\top}\be_i\pi_j^{-\frac{1}{2}}\eta_j, \end{aligned} \end{aligned} \]
		it follows that
		\[ \begin{aligned}
		& \Bigl\| \mathbb{E}\bigl[ \bT_k^{(l)} \big| X_{(k-1)\alpha+l} \bigr] - \bK^{\top}\bF[\diag(\bpi)]^{-\frac{1}{2}}\bmu \Bigr\|_2 \\ = & \Bigl\| \sum_{i=1}^p \sum_{j=1}^p (\widetilde{\pi}_i^{(k,l)} - \pi_i)P_{ij}\bK^{\top}\be_i \pi_j^{-\frac{1}{2}}\eta_j \Bigr\|_2 \\ \leq & \sum_{i=1}^p\sum_{j=1}^p \bigl| \widetilde{\pi}_i^{(k,l)} - \pi_i \bigr| P_{ij} \pi_j^{-\frac{1}{2}} \eta_j \bigl\| \be_i^{\top} \bK \bigr\|_2 \\ \leq & \sum_{i=1}^p \bigl| \widetilde{\pi}_i^{(k,l)} - \pi_i \bigr| \sum_{j=1}^p P_{ij} \pi_{\min}^{-\frac{1}{2}} \cdot \gamma\sqrt{s/p} \\ \leq & \bigl\| \widetilde{\bpi}^{(k,l)} - \bpi \bigr\|_1 \cdot c_1^{-\frac{1}{2}} \gamma \sqrt{s} \leq \frac{t}{2}.
		\end{aligned} \]
		Therefore,
		\beq \label{term2} \begin{aligned} & \begin{aligned} \Bigl\| n^{-1}\sum_{l=1}^{\alpha}\sum_{k=1}^{n_0} & \mathbb{E}\bigl[ \bT_k^{(l)} \big| X_{(k-1)\alpha+l} \bigr] - \bK^{\top}\bF[\diag(\bpi)]^{-\frac{1}{2}}\bmu \Bigr\|_2 \end{aligned} \\ \leq & n^{-1}\sum_{l=1}^{\alpha}\sum_{k=1}^{n_0} \Bigl\| \mathbb{E}\bigl[ \bT_k^{(l)} \big| X_{(k-1)\alpha+l} \bigr] - \bK^{\top}\bF[\diag(\bpi)]^{-\frac{1}{2}}\bmu \Bigr\|_2 \leq \frac{t}{2}, \end{aligned} \eeq
		Symmetrically, we have 
		\beq \label{term2'} \begin{aligned} & \begin{aligned} \Bigl\| n^{-1}\sum_{l=1}^{\alpha}\sum_{k=1}^{n_0} \mathbb{E}\bigl[ \bS_k^{(l)} \big| X_{(k-1)\alpha+l} \bigr] - \bmu^{\top}\bF[\diag(\bpi)]^{-\frac{1}{2}}\bK \Bigr\|_2 \leq \frac{t}{2}. \end{aligned} \end{aligned} \eeq
		
		Combining \eqref{term1} and \eqref{term2}, \eqref{term1'} and \eqref{term2'} yields
		\[ \begin{aligned} & \mathbb{P} \Bigl( \bigl\| n^{-1}\bK^{\top}\bN[\diag(\bpi)]^{-\frac{1}{2}}\bmu - \bK^{\top}\bF[\diag(\bpi)]^{-\frac{1}{2}}\bmu \bigr\|_2 \geq t \Bigr)
		\\ \leq & \alpha (s+1) \exp\left(-\frac{nt^2/(8\alpha)}{2\gamma^2sp^{-1}+c_1^{-\frac{1}{2}}\gamma t\sqrt{s}/3}\right), \end{aligned} \]
		and
		\[ \begin{aligned} & \mathbb{P} \Bigl( \bigl\| n^{-1}\bmu^{\top}\bN[\diag(\bpi)]^{-\frac{1}{2}}\bK - \bmu^{\top}\bF[\diag(\bpi)]^{-\frac{1}{2}}\bK \bigr\|_2 \geq t \Bigr)
		\\ \leq & \alpha (s+1) \exp\left(-\frac{nt^2/(8\alpha)}{2C_1c_1^{-1}\gamma^2sp^{-1}+c_1^{-\frac{1}{2}}\gamma t\sqrt{s}/3}\right). \end{aligned} \]
		For a fixed $c_0 > 0$, there exists a constant $C',C''>0$ such that, by taking \[\begin{aligned} & t = C' n^{-\frac{1}{2}} \sqrt{\gamma^2 s p^{-1} \alpha\log(n)} + C' n^{-1} \sqrt{s}\alpha\log(n), \\ & n \geq C'' \bigl( \alpha p \log(n) \vee \alpha s \bigr), \end{aligned} \] one has
		\[ \begin{aligned} \mathbb{P}\Bigl( & \bigl\| n^{-1}\bK^{\top}\bN[\diag(\bpi)]^{-\frac{1}{2}}\bmu - \bK^{\top}\bF[\diag(\bpi)]^{-\frac{1}{2}}\bmu \bigr\|_2 \\ & \geq C'\bigl(1+(C'')^{-\frac{1}{2}}\bigr) n^{-\frac{1}{2}} \sqrt{\gamma^2 s p^{-1} \alpha\log(n)} \Bigr) \leq n^{-c_0}. \end{aligned} \]
		According to Lemma 5 in \cite{zhang2018spectral}, \[ \alpha \leq - \tau_* \log_2\Bigl(\frac{\pi_{\min}}{2}\wedge\frac{\sqrt{c_1}t}{4\gamma\sqrt{s}}\Bigr), \] where $\tau_* = \tau(1/4)$. Fix $C'$, when $C''$ is sufficiently large, \[ n^{-2} \leq \frac{\sqrt{c_1}}{4\gamma} C'n^{-1} \leq \frac{\sqrt{c_1}t}{4\gamma\sqrt{s}} \leq \frac{\pi_{\min}}{2}.\] In this case, $\alpha \leq 2 \tau_* \log_2\bigl(n\bigr)$. Therefore, there exists a constant $C>0$ such that when $n \geq Cp\tau_*\log^2(n)$,
		\beq \label{GNe_Concentration} \begin{aligned} & \mathbb{P}\Bigl(\bigl\| n^{-1}\bK^{\top}\bN[\diag(\bpi)]^{-\frac{1}{2}}\bmu - \bK^{\top}\bF[\diag(\bpi)]^{-\frac{1}{2}}\bmu \bigr\|_2 \\ & \qquad \leq C n^{-\frac{1}{2}} \sqrt{\gamma^2 s p^{-1} \tau_*\log^2(n)}\Bigr) \leq n^{-c_0}. \end{aligned} \eeq
		Following the same analysis, one also has
		\beq \begin{aligned} & \mathbb{P}\Bigl(\bigl\| n^{-1}\bmu^{\top}\bN[\diag(\bpi)]^{-\frac{1}{2}}\bK - \bmu^{\top}\bF[\diag(\bpi)]^{-\frac{1}{2}}\bK \bigr\|_2 \\ & \qquad \leq \bar{C} n^{-\frac{1}{2}} \sqrt{\gamma^2 s p^{-1} \tau_*\log^2(n)}\Bigr) \leq n^{-c_0}, \end{aligned} \eeq
		when $n \geq \bar{C} p\tau_*\log^2(n)$ for some sufficiently large $\bar{C}>0$.
	\end{proof}
	
	\subsection{Adapting Markov Chain Concentration Inequalities to Our Settings}
	
	\begin{corollary}\label{MC_Concentration_Corollary}
		Under assumption \eqref{Assumption1}, for any $c_0 > 0$, there exists a constant $C > 0$, such that if $n \geq C\tau_*p\log^2(n)$, then
		\beq\label{Corollay1} \mathbb{P} \Bigl( \bigl\| \widehat\bQ - \bQ \bigr\|_2 \geq C n^{-1/2}\sqrt{\tau_* \log^2(n)} \Bigr) \leq n^{-c_0}. \eeq
	\end{corollary}
	
	\begin{proof}
		Based on Lemma \ref{MC_Concentration} and the union bound, we know that for any fixed $c_0 > 0$, there exists a constant $C_0 > 0$ such that with probability at least $1-2n^{-c_0}$, 
		\beq\label{MC_Concentration0} \begin{aligned} & \bigl\|n^{-1}\bN - \bF\bigr\|_2 \leq C_0 n^{-1/2} \sqrt{\pi_{\max} \tau_* \log^2(n)} \\ \text{and }& \|n^{-1}\bm - \bpi\|_{\infty} \leq C_0 n^{-1/2} \sqrt{\pi_{\max} \tau_* \log^2(n)} \end{aligned} \eeq
		hold simultaneously for $n \geq C_0 \tau_*p\log^2(n)$.
		
		We find that
		\beq \label{Q_Concentration} \begin{aligned}
			& \bigl\| \widehat\bQ - \bQ \bigr\|_2 = \bigl\| (n^{-1}\bN)[\diag(n^{-1}\bm)]^{-\frac{1}{2}}- \bF [\diag(\bpi)]^{-\frac{1}{2}} \bigr\|_2 \\ \leq &  \bigl\|\bigl(n^{-1}\bN- \bF\bigr) [\diag(\bpi)]^{-\frac{1}{2}} \bigr\|_2 + \bigl\| \bigl((n^{-1}\bN)[\diag(n^{-1}\bm)]^{-\frac{1}{2}}-[\diag(\bpi)]^{-\frac{1}{2}}\bigr)\bigr\|_2.
		\end{aligned} \eeq
		Under condition \eqref{MC_Concentration0}, the first term in \eqref{Q_Concentration} satisfies
		\beq\label{Q_Concentration1} \begin{aligned}
			& \bigl\|\bigl(n^{-1}\bN- \bF\bigr)[\diag(\bpi)]^{-\frac{1}{2}}\bigr\|_2 \leq \pi_{\min}^{-\frac{1}{2}} \bigl\| n^{-1}\bN - \bF \bigr\|_2 \\ \leq & C_0n^{-1/2}\sqrt{\frac{\pi_{\max}}{\pi_{\min}}\tau_*\log^2(n)} \leq C_0C_1^{\frac{1}{2}}c_1^{-\frac{1}{2}}n^{-1/2}\sqrt{\tau_*\log^2(n)}. 
		\end{aligned} \eeq
		We decompose the second term in \eqref{Q_Concentration} into
		\[ \begin{aligned}
		& \bigl\| (n^{-1}\bN)\bigl([\diag(n^{-1}\bm)]^{-\frac{1}{2}}-[\diag(\bpi)]^{-\frac{1}{2}}\bigr)\bigr\|_2 \\ \leq &  \bigl\| n^{-1}\bN\bigr\| \bigl\|[\diag(n^{-1}\bm)]^{-\frac{1}{2}}-[\diag(\bpi)]^{-\frac{1}{2}}\bigr\|_2 \\ \leq & \bigl(\|\bF\|_2+\bigl\|n^{-1}\bN-\bF\bigr\|_2\bigr) \cdot \max_{1 \leq j \leq p} \bigl| (n^{-1}m_j)^{-\frac{1}{2}} - \pi_j^{-\frac{1}{2}} \bigr|.
		\end{aligned} \]
		Here, $\|\bF\|_2 \leq \sqrt{\pi_{\max}}$ by \eqref{normF}. Notice that $\bigl|x^{-1/2}-1\bigr| \leq |x-1|$ for any $x \geq \frac{1}{2}$. Taking $n \geq 4C_0^2C_1c_1^{-2}\tau_*p\log^2(n)$, we have
		\beq\label{Normalize_Concentration} \begin{aligned}
			& \bigl| (n^{-1}m_j)^{-\frac{1}{2}} - \pi_j^{-\frac{1}{2}} \bigr|
			= \pi_j^{-\frac{1}{2}} \Bigl| \bigl( \pi_j^{-1}n^{-1}m_j \bigr)^{-\frac{1}{2}} - 1 \Bigr| \\ \leq & \pi_j^{-\frac{1}{2}} \bigl|\pi_j^{-1}n^{-1}m_j-1\bigr| \leq \pi_j^{-\frac{3}{2}}\|n^{-1}\bm - \bpi\|_{\infty}
		\end{aligned} \eeq
		for $j=1,2,\ldots,p$.
		It follows from condition \eqref{MC_Concentration0} that
		\[ \begin{aligned} \max_{1 \leq j \leq p} \bigl| (n^{-1}m_j)^{-\frac{1}{2}} - \pi_j^{-\frac{1}{2}} \bigr| \leq & \pi_{\min}^{-\frac{3}{2}} \cdot C_0 n^{-1/2}\sqrt{\pi_{\max}\tau_*\log^2(n)} \\ \leq & C_0C_1^{\frac{1}{2}}c_1^{-\frac{3}{2}} n^{-1/2}p\sqrt{\tau_*\log^2(n)}. \end{aligned} \]
		When $n \geq C_0\tau_*p\log^2(n)$, condition \eqref{MC_Concentration0} also implies
		\[ \begin{aligned} \bigl\|n^{-1}\bN-\bF\bigr\|_2 \leq & C_0\bigl(C_0\tau_*p\log^2(n)\bigr)^{-\frac{1}{2}} \sqrt{\pi_{\max}\tau_*\log^2(n)} \\ \leq & \sqrt{C_0\pi_{\max}p^{-1}} \leq C_0^{\frac{1}{2}}C_1^{\frac{1}{2}}p^{-1}. \end{aligned} \]
		
		The second term in \eqref{Q_Concentration} then satisfies
		\beq \label{Q_Concentration2} \begin{aligned}
			& \bigl\|(n^{-1}\bN) \bigl([\diag(n^{-1}\bm)]^{-\frac{1}{2}}-[\diag(\bpi)]^{-\frac{1}{2}}\bigr)\bigr\|_2 \\ \leq & (C_0^{\frac{1}{2}}+C_1^{\frac{1}{2}})C_0C_1c_1^{-\frac{3}{2}}n^{-1/2}\sqrt{\tau_*\log^2(n)}.
		\end{aligned} \eeq
		
		Plugging \eqref{Q_Concentration1} and \eqref{Q_Concentration2} in \eqref{Q_Concentration}, we can conclude that, when $n \geq C_0\tau_*p\log^2(n)$, 
		\[ \mathbb{P}\Bigl( \bigl\| \widehat\bQ - \bQ \bigr\|_2 \geq \widetilde{C}n^{-1/2}\sqrt{\tau\log^2(n)}\Bigr) \leq 2n^{-c_0}\] for some constant $\widetilde{C} > 0$.
	\end{proof}
	
	\begin{corollary} \label{Row-wise_Concentration_Corollary}
		Under assumptions \eqref{Assumption1} and \eqref{Assumption2}, for any $c_0 > 0$, there exists a constant $C > 0$, such that if $n \geq C \tau_* p \log^2(n)$, then
		\beq\label{Corollary2} \mathbb{P} \Bigl( \bigl\| \bG^{\top} (\widehat\bQ-\bQ)\be_j \bigr\|_2 \geq C n^{-1/2}\sqrt{\tau_*(r/p)\log^2(n)} \Bigr) \leq n^{-c_0}. \eeq
	\end{corollary}
	
	\begin{proof}
		Recall that we have proved $\bigl\|\be_j^{\top}\bG\bigr\|_2 \leq c_2^{-\frac{1}{2}}\pi_j\sqrt{pr}$ in Lemma \ref{RowNorm}.
		Let $\gamma = C_1c_2^{-\frac{1}{2}}$. Then \[ \max_{1 \leq j \leq p} \bigl\|\be_j^{\top}\bG\bigr\|_2 \leq \gamma \sqrt{r/p}.\] $\be_j$ and $\bG$ satisfy the conditions in Lemma \ref{Row-wise_Concentration}.
		For a fixed $c_0 > 0$, Lemma \ref{MC_Concentration} and \ref{Row-wise_Concentration} imply that there exists a constant $C_0>0$ such that when $n \geq C_0\tau_*p\log^2(n)$, by union bound, with probability at least $1-2n^{-c_0}$,
		\beq
		\|n^{-1}\bm - \bpi\|_{\infty} \leq C_0 n^{-1/2} \sqrt{\pi_{\max} \tau_* \log^2(n)} \label{Row-wise_Concentration01} \eeq
		and
		\beq \begin{aligned} &\max_{1 \leq j \leq p}\bigl\|\bG^{\top}(n^{-1}\bN - \bF)[\diag(\bpi)]^{-\frac{1}{2}}\be_j\bigr\|_2 \leq C_0 n^{-1/2}\sqrt{\tau_*(\gamma^2r/p)\log^2(n)}. \end{aligned} \label{Row-wise_Concentration02}
		\eeq
		
		Note that
		\beq \label{Row-wise_Concentration00}\begin{aligned}
			& \bigl\|\bG^{\top}\bigl(\widehat\bQ-\bQ)\be_j\bigr\|_2 =   \Bigl\|\bG^{\top}\bigl((n^{-1}\bN)[\diag(n^{-1}\bm)]^{-\frac{1}{2}}-\bF[\diag(\bpi)]^{-\frac{1}{2}}\bigr)\be_j\Bigr\|_2 \\
			\leq & \bigl\|\bG^{\top}\bigl(n^{-1}\bN-\bF\bigr)[\diag(\bpi)]^{-\frac{1}{2}}\be_j\bigr\|_2 + \Bigl|(n^{-1}m_j)^{-\frac{1}{2}} - \pi_j^{-\frac{1}{2}}\Bigr| \cdot \bigl\|\bG^{\top}(n^{-1}\bN)\be_j\bigr\|_2.
		\end{aligned} \eeq
		The inequality \eqref{Row-wise_Concentration02} provides an upper bound for the first term in \eqref{Row-wise_Concentration00}.  In \eqref{Normalize_Concentration}, there is an estimate of $\bigl|(n^{-1}m_j)^{-\frac{1}{2}} - \pi_j^{-\frac{1}{2}}\bigr|$.  We only need to analyze $\bigl\|\bG^{\top}(n^{-1}\bN)\be_j\bigr\|_2$ in the following.
		
		We find that \beq \label{Row-wise_Concentration21} \begin{aligned}
			& \bigl\|\G^{\top}\bF\be_j\bigr\|_2 = \bigl\|\bG^{\top}\bQ[\diag(\bpi)]^{\frac{1}{2}}\be_j\bigr\|_2 = 
			\sqrt{\pi_j}\bigl\|\bSigma\bH^{\top}\be_j\bigr\|_2 \\ \leq & \sqrt{\pi_j} \sigma_1 \bigl\|\be_j^{\top}\bH\bigr\|_2 \leq C_1^2c_1^{-\frac{1}{2}}c_2^{-\frac{3}{2}}\pi_j\sqrt{r/p} , \end{aligned} \eeq
		where we used \eqref{rowH} and \eqref{normQ}.
		
		Additionally, when $n \geq C_0\tau_*p\log^2(n)$, \eqref{Row-wise_Concentration02} shows that
		\beq\label{Row-wise_Concentration22} \begin{aligned} \bigl\|\bG^{\top}(n^{-1}\bN-\bF)\be_j\bigr\|_2  = & \sqrt{\pi_j} \bigl\|\bG^{\top}(n^{-1}\bN-\bF)[\diag(\bpi)]^{-\frac{1}{2}}\be_j\bigr\|_2\\ \leq & \sqrt{\pi_j} C_0 \bigl( C_0\tau_*p\log^2(n) \bigr)^{-\frac{1}{2}}\sqrt{\tau_*(\gamma^2r/p)\log^2(n)} \\ = & C_0^{\frac{1}{2}}\gamma p^{-1}\sqrt{\pi_jr} \leq C_0^{\frac{1}{2}}C_1^{\frac{1}{2}}c_1^{-\frac{1}{2}}\gamma p^{-\frac{3}{2}}r^{\frac{1}{2}}. \end{aligned} \eeq
		Combining \eqref{Row-wise_Concentration21} and \eqref{Row-wise_Concentration22} gives
		\beq\label{Row-wise_Concentration2} \begin{aligned}
			& \bigl\|\bG^{\top}(n^{-1}\bN)\be_j\bigr\|_2 \leq \bigl\|\bG^{\top}\bF\be_j\bigr\|_2 + \bigl\|\bG^{\top}(n^{-1}\bN-\bF)\be_j\bigr\|_2 \leq \widetilde{C} p^{-\frac{3}{2}}r^{\frac{1}{2}}
		\end{aligned} \eeq
		for some constant $\widetilde{C} > 0$.
		
		Plugging \eqref{Row-wise_Concentration01}, \eqref{Row-wise_Concentration02}, \eqref{Normalize_Concentration} and \eqref{Row-wise_Concentration2} into \eqref{Row-wise_Concentration00}, we complete the proof of Corollary \ref{Row-wise_Concentration_Corollary}.
		
		
		
		
	\end{proof}
	
	\begin{corollary}\label{MC_Leading_Corollary}
		Suppose $\bmu \in \mathbb{R}^{p}$ is a nonnegative unit vector. Under assumptions \eqref{Assumption1} and \eqref{Assumption4}, if $\bh_1^{\top}\bmu \leq \beta p^{-\frac{1}{2}}$ for some $\beta>0$, then for any $c_0 > 0$, there exists a constant $C>0$, such that when $n \geq C\tau_*p\log^2(n)$,
		\beq \label{State1}
		\mathbb{P} \Bigl( \bigl| \bg_1^{\top}(\widehat\bQ-\bQ)\bmu \bigr| \geq Cn^{-1/2} \sqrt{\tau_*p^{-1}\log^2(n)} \Bigr) \leq n^{-c_0}.
		\eeq
		Symmetrically, if $\bmu^{\top}\bg_1 \leq \beta p^{-1}$ then for any $c_0 > 0$, there exists a constant $C>0$, such that when $n \geq C\tau_*p\log^2(n)$,
		\beq \label{State2}
		\mathbb{P} \Bigl( \bigl| \bmu^{\top}(\widehat\bQ-\bQ)\bh_1 \bigr| \geq Cn^{-1/2} \sqrt{\tau_*p^{-1}\log^2(n)} \Bigr) \leq n^{-c_0}.
		\eeq
	\end{corollary}
	
	\begin{proof}
		In Lemma \ref{Pre}, we proved that $0<\bh_1(j) \leq C^{\#}\sqrt{\pi_j}$ and $0 \leq \bg_1(j) \leq c^{\#}\pi_j\sqrt{p}$ for $j = 1,2,\ldots,p$ and some $c^{\#}, C^{\#} > 0$. Let $\gamma = C_1^{\frac{1}{2}}C^{\#} \vee C_1c^{\#}$, then $\bigl|\bh_1(j)\bigr| \leq \gamma p^{-\frac{1}{2}}$ and $\bigl|\bg_1(j)\bigr| \leq \gamma p^{-\frac{1}{2}}$. Lemma \ref{MC_Concentration} and \ref{Row-wise_Concentration} show that for any fixed $c_0>0$, there exists a constant $C_0 > 0$ such that when $n \geq C_0\tau_*p\log^2(n)$, with probability at least $1-2n^{-c_0}$,
		\begin{eqnarray}
		& \begin{aligned} & \bigl| \bg_1^{\top}(n^{-1}\bN-\bF)[\diag(\bpi)]^{-\frac{1}{2}}\bmu \bigr| \leq C_0n^{-1/2} \sqrt{\tau_*\gamma^2p^{-1}\log^2(n)}, \end{aligned} \label{Quote1} \\
		& \begin{aligned} & \bigl| \bmu^{\top}(n^{-1}\bN-\bF)[\diag(\bpi)]^{-\frac{1}{2}}\bh_1 \bigr| \leq C_0n^{-1/2} \sqrt{\tau_*\gamma^2p^{-1}\log^2(n)}, \end{aligned} \label{Quote2} \\
		&\|n^{-1}\bm - \bpi\|_{\infty} \leq C_0 n^{-1/2} \sqrt{\pi_{\max} \tau_* \log^2(n)}.
		\end{eqnarray}
		
		Note that
		\beq \label{Eqn00} \begin{aligned} \bigl|\bg_1^{\top}(\widehat\bQ-\bQ)\bmu\bigr| \leq & \bigl| \bg_1^{\top}(n^{-1}\bN-\bF)[\diag(\bpi)]^{-\frac{1}{2}}\bmu \bigr| \\ & + \bigl| \bg_1^{\top}(n^{-1}\bN)\bigl([\diag(n^{-1}\bm)]^{-\frac{1}{2}}-[\diag(\bpi)]^{-\frac{1}{2}}\bigr) \bmu \bigr|. \end{aligned} \eeq
		Since $\bg_1,\bN,\bmu \geq 0$,
		\beq \label{Eqn0} \begin{aligned} & \bigl|  \bg_1^{\top}(n^{-1}\bN)\bigl([\diag(n^{-1}\bm)]^{-\frac{1}{2}}-[\diag(\bpi)]^{-\frac{1}{2}}\bigr) \bmu \bigr| \\ \leq & \bigl| \bg_1^{\top}(n^{-1}\bN)[\diag(\bpi)]^{-\frac{1}{2}} \bmu \bigr| \cdot \max_{1 \leq i \leq p}\bigl|(n^{-1}m_i/\pi_i)^{-\frac{1}{2}}-1\bigr| \\
			\leq & \Bigl( \bigl| \bg_1^{\top}(n^{-1}\bN-\bF)[\diag(\bpi)]^{-\frac{1}{2}} \bmu \bigr| + \bg_1^{\top}\bQ \bmu \Bigr) \cdot \max_{1 \leq i \leq p}\bigl|\sqrt{n^{-1}m_i/\pi_i}-1\bigr|. \end{aligned} \eeq
		By \eqref{Quote1}, when $n \geq C_0 \tau_*p\log^2(n)$,
		\beq \label{Eqn1} \begin{aligned} & \bigl| \bg_1^{\top}(n^{-1}\bN-\bF)[\diag(\bpi)]^{-\frac{1}{2}} \bmu \bigr| \\ \leq & C_0\bigl(C_0\tau_*p\log^2(n)\bigr)^{-\frac{1}{2}}\sqrt{\tau_*\gamma^2p^{-1}\log^2(n)} \leq C_0^{\frac{1}{2}} \gamma p^{-1}. \end{aligned} \eeq
		In addition,
		\beq \label{Eqn2} \bg_1^{\top}\bQ\bmu = \sigma_1\bh_1^{\top}\bmu \leq C_1c_1^{-\frac{1}{2}}\beta p^{-1}, \eeq
		where we used \eqref{normQ} and condition $\bh_1^{\top} \bmu \leq \beta p^{-\frac{1}{2}}$.
		According to \eqref{Normalize_Concentration}, we have
		\beq \label{Eqn3} \max_{1 \leq i \leq p}\bigl|\sqrt{n^{-1}m_i/\pi_i}-1\bigr| \leq \pi_{\min}^{-1}\bigl\| n^{-1}\bm_j - \bpi \bigr\|_{\infty}. \eeq
		
		Combining \eqref{Eqn00} - \eqref{Eqn3} gives \eqref{State1}. We can prove \eqref{State2} in the same way.
		
	\end{proof}

	\subsection{Plugging Concentration Inequalities into Deterministic Bounds}
	
	\begin{theorem}[Theorem 2 in paper] \label{Row-wise Bound}
		Let $\bH^* = \bigl[ \bh_2, \bh_3, \ldots, \bh_r \bigr]$ and $\widehat\bH^* = \bigl[ \widehat\bh_2, \widehat\bh_3, \ldots, \widehat\bh_r \bigr]$. Under assumptions \eqref{Assumption1}, \eqref{Assumption2} and \eqref{Assumption3}, for a fixed $c_0>0$, there exists a constant $C>0$ and an orthogonal matrix $\bOmega_* \in \mathbb{O}^{(r-1)\times(r-1)} $, such that if $n \geq C \tau_*p\log^2(n)$,
		\beq \label{Stocastic Row-wise Bound} \begin{aligned} \mathbb{P} & \left(  \max_{1 \leq j \leq p} \bigl\| \be_j^{\top}\bigl(\widehat\bH^*\bOmega^*-\bH^*\bigr) \bigr\|_2 \right. \\ & \left. \geq C \Bigl(n^{-1/2}\sqrt{\tau_*r\log^2(n)} + n^{-1}\tau_*p\log^2(n)\Bigr) \right) \leq n^{-c_0}. \end{aligned} \eeq
	\end{theorem}
	
	\begin{proof}
		According to Lemma \ref{Deterministic Row-wise Error}, $\widehat\bH^*$ has a deterministic row-wise perturbation bound
		\beq \label{Deterministic} \begin{aligned} 
			\bigl\|\be_j^{\top}(\widehat\bH_*\bOmega_*-\bH_*)\bigr\|_2 \leq & \frac{2}{\Delta^*}\bigl\|\bG_*^{\top}(\widehat\bQ-\bQ)\be_j\bigr\|_2 + \frac{4(1+\sqrt{2})}{\Delta^*}\bigl\|\be_j^{\top}\bH\bigr\|_2\bigl\|\widehat\bQ-\bQ\bigr\|_2 \\ & + \frac{8}{(\Delta^*)^2}\bigl\|\widehat\bQ-\bQ\bigr\|_2^2 
		\end{aligned} \eeq
		Here, $\bG^* = \bigl[\bg_2,\bg_3,\ldots,\bg_r\bigr]$ and $\Delta^* = \min\bigl\{\sigma_1-\sigma_2, \sigma_r\bigr\}$. We note that $\bigl\|\bG_*^{\top}(\widehat\bQ-\bQ)\be_j\bigr\|_2 \leq \bigl\|\bG^{\top}(\widehat\bQ-\bQ)\be_j\bigr\|_2$.
		Based on Corollary \ref{MC_Concentration_Corollary} and \ref{Row-wise_Concentration_Corollary}, for a fixed $c_0 > 0$, there exists a constant $C_0 > 0$ such that when $n \geq C_0\tau_*p\log^2(n)$, by union bound, with probability at least $1-2n^{-c_0}$,
		\begin{align} 
		& \bigl\| \widehat\bQ - \bQ \bigr\|_2 \leq C_0 n^{-1/2}\sqrt{\tau_* \log^2(n)}, \label{MC_Concentration_quote} \\
		& \bigl\| \bG^{\top} (\widehat\bQ-\bQ)\be_j \bigr\|_2 \leq C_0 n^{-1/2}\sqrt{\tau_*(r/p)\log^2(n)}, \label{Row-wise_Concentration_quote}
		\end{align}
		for $j = 1,2,\ldots,p$. Recall that $\sigma_1 - \sigma_2 \geq c_3 p^{-\frac{1}{2}}$ by assumption \eqref{Assumption3}, and $\sigma_r \geq c_2 p^{-\frac{1}{2}}$ by Lemma \ref{sigma_r}. We have \beq\label{EigGap} \Delta^* \geq (c_2 \wedge c_3) p^{-\frac{1}{2}}. \eeq
		Plugging \eqref{MC_Concentration_quote}, \eqref{Row-wise_Concentration_quote}, \eqref{rowG} and \eqref{EigGap} into \eqref{Deterministic}, we complete the proof of Theorem \ref{Row-wise Error}.
		
		
		
	\end{proof}

	\begin{theorem}[Theorem 1 in paper] \label{Leading_Convergence}
		Under assumptions \eqref{Assumption1}, \eqref{Assumption3} and \eqref{Assumption4}, for any constant $c_0>0$, there exists a constant $C>0$ and $\omega \in \{\pm 1\}$ such that if $n \geq C \tau_*p\bigl(\log^2(r) \vee 1\bigr)\log^2(n)$,
		\[ \begin{aligned} \mathbb{P} \biggl( \max_{1 \leq j \leq p} \bigl| \omega\widehat\bh_1(j)-\bh_1(j) \bigr| \geq & C \Bigl( n^{-1/2}\sqrt{\tau_*\bigl(\log^2(r) \vee 1\bigr)\log^2(n)} \\
		& + n^{-1}\tau_*p\bigl(\log^2(r) \vee 1\bigr)\log^2(n)\Bigr) \biggr) \leq n^{-c_0}. \end{aligned} \]
	\end{theorem}
	
	\begin{proof}
		In Lemma \ref{Deterministic_Leading}, we take
		\[ K = \Bigg\lceil \biggl[2\log\Bigl(\frac{\sigma_1+\sigma_2}{2\sigma_2}\Bigr)\biggr]^{-1}\log(r) \Biggr\rceil \vee 1, \]
		then
		\[ \Bigl(\frac{2\sigma_2}{\sigma_1+\sigma_2}\Bigr)^K \leq r^{-\frac{1}{2}}. \]
		
		Fix $c_0 > 0$.
		According to Corollary \ref{MC_Concentration_Corollary}, there exists a constant $C_0 \geq c_3^2$ such that when $n \geq C_0\tau_*p\log^2(n)$, with probability at least $1 - n^{-c_0}$,
		\beq \label{eq1} \bigl\| \widehat\bQ - \bQ \bigr\|_2 \leq C_0 n^{-1/2}\sqrt{\tau_* \log^2(n)}. \eeq
		We further take $n \geq C_0^2c_3^{-2}K^2 \tau_* p\log^2(n)$, then
		\[ \bigl\| \widehat\bQ - \bQ \bigr\|_2 \leq \frac{\sigma_2-\sigma_1}{2K}. \]
		By Lemma \ref{Deterministic_Leading}, there exists $\omega \in \{\pm 1\}$ such that
		\beq \label{eq0} \begin{aligned}
			& \bigl| \omega\widehat\bh_1(j) - \bh_1(j) \bigr| \\ \leq & \sum_{k=1}^{\lceil K/2 \rceil} \Bigl(\sigma_1^{-2k}\Bigl|\be_j^{\top}\bQ^{\top}(\bQ\bQ^{\top})^{k-1}(\widehat\bQ-\bQ)\bh_1\Bigr| + \sigma_1^{-2k-1}\Bigl|\bg_1^{\top}(\widehat\bQ-\bQ)(\bQ^{\top}\bQ)^k\be_j\Bigr| \Bigr) \\ & + |\bh_1(j)| \cdot (2K+1+2\sqrt{2})\delta^{-1}\bigl\| \widehat\bQ-\bQ \bigr\|_2 + 4\sqrt{2}r^{-\frac{1}{2}} \bigl\| \be_j^{\top} \bH\bigr\|_2 \delta^{-1}\bigl\| \widehat\bQ - \bQ \bigr\|_2 \\
			& + \bigl(2K^2+8\sqrt{2}K\bigr)\delta^{-2}\bigl\|\widehat\bQ-\bQ\bigr\|_2^2.
		\end{aligned}\eeq
		
		Note that for $k = 1,2,\ldots,\lceil K/2 \rceil$,
		\[ \begin{aligned} & \sigma_1^{-2k+1}\be_j^{\top}\bQ^{\top}(\bQ\bQ^{\top})^{k-1}\bg_1 = \sigma_1^{-2k+1}\be_j^{\top}\bQ^{\top}(\bQ\bQ^{\top})^{k-2}\bQ\bQ^{\top}\bg_1 \\
		= & \sigma_1^{-2k+2}\be_j^{\top}\bQ^{\top}(\bQ\bQ^{\top})^{k-2}\bQ\bh_1
		= \sigma_1^{-2(k-1)+1}\be_j^{\top}\bQ^{\top}(\bQ\bQ^{\top})^{k-2}\bg_1 \\
		= & \ldots \\
		= & \sigma_1^{-1}\be_j^{\top}\bQ^{\top}\bg_1
		= \be_j^{\top} \bh_1 = \bh_1(j),
		\end{aligned} \]
		where we used $\bQ^{\top}\bg_1 = \sigma_1\bh_1$ and $\bQ\bh_1 = \sigma\bg_1$ iteratively. Similarly,
		\[ \sigma_1^{-2k}\bh_1^{\top}(\bQ^{\top}\bQ)^k\be_j = \bg_1(j). \]
		Therefore, according to Lemmat \ref{Pre}, there exists some constant $C^{\#} > 0$ such that $\bupsilon_{1,2k-1} = \sigma_1^{-2k+1}\bQ(\bQ^{\top}\bQ)^{k-1}\be_j$ satisfies \[ \begin{aligned} & \bupsilon_{1,2k-1} \geq 0, \quad \|\bupsilon_{1,2k-1}\|_2 = 1 \quad \text{and} \quad \bupsilon_{1,2k-1}^{\top}\bg_1 = \bh_1(j) \leq C^{\#}p^{-\frac{1}{2}}, \end{aligned} \]
		and $\bupsilon_{2,2k} = \sigma_1^{-2k}(\bQ^{\top}\bQ)^k\be_j$ satisfies
		\[ \begin{aligned} & \bupsilon_{2,2k} \geq 0, \quad \|\bupsilon_{2,2k}\|_2 = 1 \quad \text{and} \quad \bh_1^{\top}\bupsilon_{2,2k} = \bg_1(j) \leq C^{\#}p^{-\frac{1}{2}}. \end{aligned} \]
		According to Corollary \ref{MC_Leading_Corollary}, there exists a constant $\widetilde{C}_0 \geq C_0^2c_3^{-2}$ such that when $n \geq \widetilde{C}_0K^2\tau_*p\log^2(n)$, with probability at least $1-2\lceil K/2 \rceil n^{-c_0}$,
		\begin{eqnarray} \bigl| \bupsilon_{1,2k-1}^{\top}(\widehat\bQ-\bQ)\bh_1 \bigr| \leq \widetilde{C}_0n^{-1/2} \sqrt{\tau_*p^{-1}\log^2(n)}, \label{eq2} \\ \bigl| \bg_1^{\top}(\widehat\bQ-\bQ)\bupsilon_{2,2k} \bigr| \leq \widetilde{C}_0n^{-1/2} \sqrt{\tau_*p^{-1}\log^2(n)}, \label{eq3} \end{eqnarray}
		for $k=1,2,\ldots,\lceil K/2 \rceil$.
		
		Plugging \eqref{eq1}, \eqref{eq2}, \eqref{eq3} into \eqref{eq0} and using Lemma \ref{RowNorm} and \ref{Pre}, we have
		\[\begin{aligned}
		\bigl| \omega\widehat\bh_1(j) - \bh_1(j) \bigr| \leq & \widetilde{C}\Bigl(K n^{-1/2} \sqrt{\tau_*\log^2(n)} 
		+ K^2n^{-1}\tau_*p\log^2(n)\Bigr)
		\end{aligned}\]
		for some sufficiently large $\widetilde{C} > 0$. Notice that $K$ has order $\log(r)$, we complete the proof of Theorem \ref{Leading_Convergence}.
	\end{proof}
	
	
	
	

	
	\section{Proof of Statistical Guarantees}
	
	Define
	\begin{align} err = & \min_{\omega \in \{\pm 1\}} \max_{1 \leq j \leq p} \bigl| \omega \widehat{\bh}_1(j) - \bh_1(j) \bigr|, \label{Def_err} \\ Err = & \min_{\bOmega \in \mathbb{O}^{(r-1) \times (r-1)}} \max_{1 \leq j \leq p} \bigl\| \bOmega \widehat\bd_j - \bd_j \bigr\|_2. \label{Def_Err} \end{align}
	
	\subsection{SCORE Normalization}
	
	The bound for $err$ is shown in Theorem \ref{Leading_Convergence}. It remains to estimate $Err$.
	
	\begin{theorem} \label{Perturbation} Under assumptions \eqref{Assumption1} - \eqref{Assumption4}, for any $c_0 > 0$, there exists a constant $C>0$ and an orthogonal matrix $\bOmega \in \mathbb{O}^{(r-1)\times(r-1)}$, such that if $n \geq C \tau_*p^{\frac{3}{2}}\bigl(\log^2(r) \vee 1)\log^2(n)$,
		\[ \begin{aligned} \mathbb{P} \biggl( \max_{1 \leq j \leq p}  \bigl\| \bOmega\widehat\bd_j - \bd_j \bigr\|_2 \geq C \Bigl( &
		n^{-1/2}\sqrt{\tau_*pr\bigl(\log^2(r) \vee 1\bigr)\log^2(n)} \\
		& + n^{-1}\tau_*p^{\frac{3}{2}}r^{\frac{1}{2}}\bigl(\log^2(r) \vee 1\bigr)\log^2(n) \Bigr) \biggr) \leq n^{-c_0}. \end{aligned} \]
	\end{theorem}
	
	\begin{proof}
		By definition,
		\[ \widehat\bd_j = \bigl[\widehat\bh_1(j)\bigr]^{-1} \bH_*^{\top} \be_j, \quad \text{for $j=1,2,\ldots,p$}, \]
		where $\widehat\bH_* = \bigl[\widehat\bh_2,\widehat\bh_3,\ldots,\widehat\bh_r\bigr]$.
		According to Theorem \ref{Row-wise Bound} and \ref{Leading_Convergence}, for a fixed $c_0 > 0$, there exists a constant $C_0 \geq 1$, $\omega \in \{\pm 1\}$ and $\bOmega_* \in \mathbb{R}^{(r-1) \times (r-1)}$ such that when $n \geq C_0 \tau_* p \bigl(\log^2(r) \vee 1\bigr) \log^2(n)$, with probability at least $1-2n^{-c_0}$,
		\begin{eqnarray}
		& \begin{aligned} \bigl| \omega\widehat\bh_1(j)-\bh_1(j) \bigr| \leq & C_0 \Bigl( n^{-1/2}\sqrt{\tau_*\bigl(\log^2(r) \vee 1\bigr)\log^2(n)} \\
		& + n^{-1}\tau_*p\bigl(\log^2(r) \vee 1\bigr)\log^2(n)\Bigr), \end{aligned} \\
		& \begin{aligned} \bigl\| \be_j^{\top}\bigl(\widehat\bH^*\bOmega^*-\bH^*\bigr) \bigr\|_2 \leq & C_0 \Bigl(n^{-1/2}\sqrt{\tau_*r\log^2(n)} + n^{-1}\tau_*p\log^2(n)\Bigr), \end{aligned}
		\end{eqnarray}
		for $j=1,2,\ldots,p$.
		
		Define $\bOmega = \omega\bOmega_*^{\top}$. We find that
		\beq \label{blah} \begin{aligned} & \bigl\| \bOmega \widehat\bd_j - \bd_j \bigr\|_2 \\
			= & \Bigl\| \bigl[ \omega \widehat\bh_1(j) \bigr]^{-1} \bOmega_*^{\top} \widehat\bH_*^{\top}\be_j - \bigl[\bh_1(j)\bigr]^{-1}\bH_*^{\top}\be_j \Bigr\|_2 \\
			\leq & \bigl| \bh_1(j) \bigr|^{-1}\bigl\|\be_j^{\top}\bigl(\widehat\bH_*\bOmega_* - \bH_*\bigr)\bigr\|_2 \\ & + \Bigl| \bigl[ \omega \widehat\bh_1(j) \bigr]^{-1} - \bigl[ \bh_1(j) \bigr]^{-1} \Bigr| \cdot \bigl\| \be_j^{\top}\bH_*\bigr\|_2 \\
			& +\Bigl| \bigl[ \omega\widehat\bh_1(j) \bigr]^{-1} - \bigl[ \bh_1(j) \bigr]^{-1} \Bigr| \cdot \bigl\|\be_j^{\top}\bigl(\widehat\bH_*\bOmega_* - \bH_*\bigr)\bigr\|_2.
		\end{aligned} \eeq
		
		
		Since $\bigl|x^{-1} - 1\bigr| \leq 2|x-1|$ for $x \geq \frac{1}{2}$, if we take $n \geq C^{\#}\tau_*p^{\frac{3}{2}}\bigl(\log^2(r) \vee 1\bigr)\log^2(n)$ for some large enough $C^{\#}>0$, then $\bigl| \omega\widehat\bh_1(j) - \bh_1(j) \bigr| \leq \frac{1}{2}\bh_1(j)$,
		\[ \begin{aligned} & \Bigl| \bigl[ \omega\widehat\bh_1(j) \bigr]^{-1} - \bigl[ \bh_1(j) \bigr]^{-1} \Bigr| = \bigl[\bh_1(j)\bigr]^{-1}\biggl| \Bigl[ \frac{\omega\widehat\bh_1(j)}{\bh_1(j)} \Bigr]^{-1} - 1 \biggr| \\ \leq & 2 \bigl[\bh_1(j)\bigr]^{-1}\biggl| \frac{\omega\widehat\bh_1(j)}{\bh_1(j)} - 1 \biggr| = 2 \bigl[\bh_1(j)\bigr]^{-2}\bigl| \omega\widehat\bh_1(j)-\bh_1(j) \bigr|.  \end{aligned} \]
		
		Using the fact in Lemma \ref{RowNorm} and \ref{Pre} that $ \bh_1(j) \geq c^{\#}\sqrt{\pi_j}$ for some $c^{\#}>0$ and $\bigl\| \be_j^{\top}\bH_* \bigr\|_2 \leq \bigl\|\be_j^{\top}\bH\bigr\|_2 \leq C_1c_2^{-\frac{3}{2}}\sqrt{\pi_jr}$, we reduce \eqref{blah} into
		\[ \begin{aligned} \bigl\| \bOmega\widehat\bd_j - \bd_j \bigr\|_2 \leq \widetilde{C} \Bigl( & 
		n^{-1/2}\sqrt{\tau_*pr\bigl(\log^2(r) \vee 1\bigr)\log^2(n)} \\
		& + n^{-1}\tau_*p^{\frac{3}{2}}r^{\frac{1}{2}}\bigl(\log^2(r) \vee 1\bigr)\log^2(n) \Bigr), 
		\end{aligned} \]
		where $\widetilde{C} > 0$ is a constant. 
		
		
	\end{proof}

	\subsection{Vertex Hunting}
	
	The estimated vertices
	$\big\{ \widehat{\bf b}_1, \widehat{\bf b}_2, \ldots, \widehat{\bf b}_r \big\}$ solves the following optimization problem:
	\[ {\rm minimize}_{\big\{ \widehat{\bf b}_1, \ldots, \widehat{\bf b}_r \big\} \subseteq \big\{ \widehat{\bf d}_1, \ldots, \widehat{\bf d}_p \big\}} \quad \max_{1 \leq j \leq p} \Big\| \widehat{\bf d}_j - \mathcal{P}_{\mathcal{S}}\widehat{\bf d}_j \Big\|_2, \]
	where $\mathcal{S}$ is the convex hull of $\widehat{\bf b}_1, \widehat{\bf b}_2, \ldots, \widehat{\bf b}_r$ and $\mathcal{P}$ is the projection operator induced by Euclidean norm.
	One can refer to \cite{Topic-SCORE} for further details of vertex hunting algorithms.
	
	\begin{theorem}[Vertex hunting]\label{VertexHunting}
		Suppose that for each meta-state, there exists at least one anchor state.
		Then there exist constants $\alpha, \alpha' > 0$ such that if $Err \leq \alpha' \sqrt{r}$, 
		\[ \max_{1 \leq k \leq r} \bigl\| \bOmega \widehat\bb_k - \bb_k \bigr\|_2 \leq \alpha \cdot Err,  \]
		where $\bOmega$ is the orthogonal matrix that achieves the minimum in the definition of $Err$.
	\end{theorem}
	\begin{proof}
		Inspired by the proof of Lemma 3.1 in \cite{Topic-SCORE}, we define a mapping $\mathcal{R}$ that maps a weight vector in the standard simplex $\mathcal{S}_{r-1} \subseteq \mathbb{R}^r$ to a vector in the simplex $\mathcal{S}\bigl(\bb_1,\bb_2,\ldots,\bb_r\bigr)$:
		\[ \bw \ \overset{\mathcal{R}}{\longmapsto} \ \bigl[ \bb_1, \bb_2, \ldots, \bb_r \bigr]\bw. \]
		
		To begin with, we prove that the mapping $\mathcal{R}$ has the following properties:
		\begin{enumerate}
			\item $ \mathcal{R}\bw_j = \bd_j$ for $j = 1,2,\ldots,p$.
			\item There exist constants $C^{\star}>0$, $c^{\star}>0$ such that for any two weight vectors ${\bw}, {\bw}' \in \mathcal{S}_{r-1}$, 
			\beq \label{VH0} c^{\star}\sqrt{r}\bigl\| {\bw} - {\bw}' \bigr\|_2 \leq \bigl\| \mathcal{R}\bw - \mathcal{R}\bw' \bigr\|_2 \leq C^{\star} \sqrt{r} \bigl\| {\bw} - {\bw}' \bigr\|_2. \eeq
			\item $\mathcal{R}$ is a bijection.
		\end{enumerate}
		1. \& 3. are obvious, we only need to show 2.
		Note that
		\[ \left( \begin{array}{c} 0 \\ \mathcal{R}\widetilde\bw \end{array} \right) - \left( \begin{array}{c} 0 \\ \mathcal{R}\widetilde\bw' \end{array} \right)  = \left( \begin{array}{ccc} 1 & \ldots & 1 \\ \bb_1 & \ldots & \bb_r \end{array} \right) \bigl(\bw - \bw' \bigr). \]
		According to Lemma \ref{B}, there exist constants ${c}^{\star} > 0$ and ${C}^{\star} > 0$ such that
		$\bB = \left( \begin{array}{ccc} 1 & \ldots & 1 \\ \bb_1 & \ldots & \bb_r \end{array} \right) = \bL^{\top}[\diag(\bl_1)]^{-1}$ satisfies
		\[ \begin{aligned} & \sigma_1(\bB) \leq \|\bL\|_2 \bigl(\min_{1 \leq k \leq r} \bl_1(k)\bigr)^{-1} \leq {C}^{\star} \sqrt{r}, \\ & \sigma_r(\bB) \geq \bigl\|\bL^{-1} \bigr\|_2^{-1} \bigl(\max_{1 \leq k \leq r} \bl_1(k) \bigr)^{-1} \geq {c}^{\star} \sqrt{r}. \end{aligned} \]
		Hence,
		\[ \label{VH3} {c}^{\star} \sqrt{r}\bigl\|\bw - \bw' \bigr\|_2\leq \bigl\| \mathcal{R}\bw - \mathcal{R}\bw' \bigr\|_2 \leq {C}^{\star} \sqrt{r}\bigl\| \bw - \bw' \bigr\|_2.\]
		

		
		
		It what follows, we first show that
		\beq \label{VertexHunting1} \max_{1 \leq j \leq p} \dist\Bigl( \widehat\bd_j, \mathcal{S}\bigl(\widehat\bb_1,\widehat\bb_2,\ldots,\widehat\bb_r\bigr) \Bigr) \leq 2Err. \eeq
		Here, $\dist$ denotes the distance function yielded by the Euclidean norm.
		Let $a_k \in [p]$ denote an anchor state associated with meta-state $k$. For $j=1,2,\ldots,p$,
		\[ \bd_j = \sum_{k=1}^r \bw_j(k) \bb_k = \sum_{k=1}^r \bw_j(k) \bd_{a_k}. \]
		Let $\bOmega$ be the orthogonal matrix that achieves the minimum in the definition of $Err$. Then
		\[ \begin{aligned} & \Bigl\|\widehat\bd_j - \sum_{k=1}^r \bw_j(k) \widehat\bd_{a_k}\Bigr\|_2 \leq \bigl\| \bOmega \widehat\bd_j - \bd_j \bigr\|_2 + \sum_{k=1}^r \bw_j(k) \bigl\| \bOmega \widehat\bd_{a_k} - \bd_{a_k} \bigr\|_2 \leq 2Err. \end{aligned} \]
		In other words, \[ \max_{1 \leq j \leq p} \dist\Bigl( \widehat\bd_j, \mathcal{S}\bigl(\widehat\bd_{a_1},\widehat\bd_{a_2},\ldots,\widehat\bd_{a_r}\bigr) \Bigr) \leq 2Err. \]
		Our algorithm guarantees that
		\[ \begin{aligned} & \max_{1 \leq j \leq p} \dist\Bigl( \widehat\bd_j, \mathcal{S}\bigl(\widehat\bb_1,\widehat\bb_2,\ldots,\widehat\bb_r\bigr) \Bigr) \leq \max_{1 \leq j \leq p} \dist\Bigl( \widehat\bd_j, \mathcal{S}\bigl(\widehat\bd_{a_1},\widehat\bd_{a_2},\ldots,\widehat\bd{a_r}\bigr) \Bigr), \end{aligned} \]  therefore, we have \eqref{VertexHunting1}.
		
		Let $j_k \in [p]$ be the index such that $\widehat\bb_k = \widehat\bd_{j_k}$. We next consider
		\[ 1 - \max_{1 \leq l \leq r} \bw_{j_l}(k) \]
		for $k = 1,2,\ldots,r$. For any $\bmu \in \mathcal{S}_{r-1}$,
		\[ \begin{aligned} 1 - \max_{1 \leq l \leq r}\bw_{j_l}(k) \leq 1 - \sum_{l=1}^r \eta_l \bw_{j_l}(k) \leq \Bigl\| \be_k - \sum_{l=1}^r \eta_l \bw_{j_l} \Bigr\|_2, \end{aligned} \]
		According to property 1, $\mathcal{R}\bw_{j_l} = \bd_{j_l}$, $\mathcal{R}\be_k = \bb_k = \bd_{a_k}$,
		\[ \begin{aligned} & \Bigl\| \be_k - \sum_{l=1}^r \eta_l \bw_{j_l} \Bigr\|_2 = \Bigl\| \bw_{a_k} - \sum_{l=1}^r \eta_l \bw_{j_l} \Bigr\|_2 \\ \overset{\eqref{VH0}}{\leq} & (c^{\star})^{-1}r^{-\frac{1}{2}} \Bigl\| \mathcal{R}\bw_{a_k} - \sum_{l=1}^r \eta_l \mathcal{R}\bw_{j_l} \Bigr\|_2 = (c^{\star})^{-1}r^{-\frac{1}{2}} \Bigl\| \bd_{a_k} - \sum_{l=1}^r \eta_l \bd_{j_l} \Bigr\|_2 \\ \leq & (c^{\star})^{-1}r^{-\frac{1}{2}} \Bigl( \Bigl\| \widehat\bd_{a_k} - \sum_{l=1}^r \eta_l \widehat\bd_{j_l} \Bigr\|_2 + 2Err \Bigr) \\ = & (c^{\star})^{-1}r^{-\frac{1}{2}} \Bigl( \Bigl\| \widehat\bd_{a_k} - \sum_{l=1}^r \eta_l \widehat\bb_l \Bigr\|_2 + 2Err \Bigr). \end{aligned} \]
		we have
		\[ \begin{aligned} & 1 - \max_{1 \leq l \leq r}\bw_{j_l}(k) \leq (c^{\star})^{-1}r^{-\frac{1}{2}}\dist\Bigl(\widehat\bd_{a_k}, \mathcal{S}\bigl(\widehat\bb_1,\widehat\bb_2,\ldots,\widehat\bb_k\bigr)\Bigr) + 2(c^{\star})^{-1}r^{-\frac{1}{2}}Err. \end{aligned} \]
		It follows from \eqref{VertexHunting1} that
		\beq \label{VertexHunting2} 1 - \max_{1 \leq l \leq r}\bw_{j_l}(k) \leq 4(c^{\star})^{-1}r^{-\frac{1}{2}}Err. \eeq
		
		When $Err < 8^{-1}c^{\star}\sqrt{r}$, for each $k$, there is only one $l$ that attains the maximum in \eqref{VertexHunting2}.
		Based on \eqref{VertexHunting2}, for $k=1,2,\ldots,r$
		\[ \begin{aligned} & \min_{1 \leq l \leq r}\bigl\| \bOmega\widehat\bb_l - \bb_k \bigr\|_2 = \min_{1 \leq l \leq r}\bigl\| \bOmega\widehat\bd_{j_l} - \bb_k \bigr\|_2 \leq \min_{1 \leq l \leq r} \bigl\|\bd_{j_l} - \bb_k\bigr\|_2 + Err \\ = & \min_{1 \leq l \leq r} \bigl\|\mathcal{R}\bw_{j_l} - \mathcal{R}\be_k\bigr\|_2 + Err \overset{\eqref{VH0}}{\leq} C^{\star}\sqrt{r} \min_{1 \leq l \leq r} \bigl\|\bw_{j_l} - \be_k\bigr\|_2 + Err \\ \leq & C^{\star}\sqrt{r} \min_{1 \leq l \leq r} \bigl\|\bw_{j_l} - \be_k\bigr\|_1 + Err \leq C^{\star}\sqrt{r} \min_{1 \leq l \leq r} 2\bigl(1-\bw_{j_l}(k)\bigr) + Err \\ \leq & 2C^{\star}\sqrt{r} \bigl(1- \max_{1 \leq l \leq r}\bw_{j_l}(k)\bigr) + Err \leq \bigl(8C^{\star}(c^{\star})^{-1} + 1\bigr)Err.  \end{aligned} \]


		
		
		
	\end{proof}
	
	
	\subsection{Error Decomposition}
	
	\begin{theorem} \label{Deterministic_Decompose}
		Under assumptions \eqref{Assumption1} - \eqref{Assumption4}, there exist constants $c^*>0$ and $C^* > 0$ such that if $Err < c^*$,
		\beq \label{DW0} \max_{1 \leq j \leq p} \bigl\|\widehat\bw_j - \bw_j\bigr\|_1 \leq C^* \cdot Err. \eeq
		If we further have
		$\bigl\|n^{-1}\bm - \bpi\bigr\|_{\infty} \leq c_1 p^{-1}$, then
		\beq \label{0} \begin{aligned} & \frac{1}{r} \sum_{k=1}^p \bigl\|(\widehat\bV - \bV)\be_k\bigr\|_1 \leq C^*\Bigl( Err +  \sqrt{p} \cdot err + p \bigl\| n^{-1}\bm - \bpi \bigr\|_{\infty} \Bigr). \end{aligned} \eeq
	\end{theorem}
	
	\begin{proof} 
		Suppose that $\omega \in \{\pm 1\}$ and $\bOmega \in \mathbb{O}^{(r-1) \times (r-1)}$ achieve the minima in definitions \eqref{Def_err} and \eqref{Def_Err}.
		
		We first focus on the differences between $\widehat\bw_j^*$ and $\bw_j$ for $j=1,2,\ldots,p$. Note that
		\[  \widehat\bw_j^* = \widehat{\bB}^{-1} \left( \begin{array}{c} 1 \\ \bOmega \widehat\bd_j \end{array} \right), \quad \bw_j = \bB^{-1} \left( \begin{array}{c} 1 \\ \bd_j \end{array} \right), \]
		where
		\[ \begin{aligned}
		& \widehat\bB = \left( \begin{array}{ccc} 1 & \ldots & 1 \\ \bOmega \widehat\bb_1 & \ldots & \bOmega \widehat\bb_r \end{array} \right),  \\
		& \bB = \left( \begin{array}{ccc} 1 & \ldots & 1 \\ \bb_1 & \ldots & \bb_r \end{array} \right) = \bL^{\top}[\diag(\bl_1)]^{-1},  \\
		& \left( \begin{array}{c} 1 \\ \bd_j \end{array} \right) = [\bh_1(j)]^{-1}\bH^{\top}\be_j.
		\end{aligned} \]
		We have
		\beq \label{DW} \begin{aligned} \bigl\| \widehat\bw_j^* - \bw_j \bigr\|_2 \leq & \bigl\|\widehat\bB^{-1}\bigr\|_2 \biggl\| \left( \begin{array}{c} 1 \\ \bOmega \widehat\bd_j \end{array} \right)-\left( \begin{array}{c} 1 \\ \bd_j \end{array} \right) \biggr\|_2 + \biggl\| \widehat\bB^{-1} \left( \begin{array}{c} 1 \\ \bd_j \end{array} \right) - {\bf w}_j \biggr\|_2 \\ = & \bigl\|\widehat\bB^{-1}\bigr\|_2 \bigl\| \bOmega \widehat\bd_j - \bd_j \bigr\|_2 + \bigl\| \widehat\bB^{-1} \bB\bw_j - {\bf w}_j \bigr\|_2 \\ \leq & \bigl\|\widehat\bB^{-1}\bigr\|_2 \bigl\| \bOmega \widehat\bd_j - \bd_j \bigr\|_2 + \bigl\| \widehat\bB^{-1} \bigr\|_2 \bigl\| ({\bf B} - \widehat{\bf B}) {\bf w}_j \bigr\|_2, \end{aligned} \eeq
		where
		\[ \bigl\| \bOmega\widehat\bd_j - \bd_j \bigr\|_2 \leq Err,\] 	and
		\beq
		\begin{aligned}
			\bigl\| \bigl(\widehat\bB - \bB\bigr) \bw_j \bigr\|_2 = & \bigg\| \sum_{k=1}^r \bw_j(k) \bigl( \bOmega\widehat\bb_k - \bb_k \bigr) \bigg\|_2 \\ \leq & \sum_{k=1}^r \bw_j(k) \bigl\| \bOmega\widehat\bb_k - \bb_k \bigr\|_2 \leq \max_{1 \leq k \leq r} \bigl\| \bOmega\widehat\bb_k - \bb_k \bigr\|_2.
		\end{aligned}
		\eeq
		Here we used $\bw_j \geq 0$ and $\sum_{k=1}^r \bw_j(k) = 1$.
		

		We now derive an upper bound for $\bigl\|\widehat\bB^{-1}\bigr\|_1$.
		According to Lemma \ref{B},
		\[ \begin{aligned} \bigl\|\bB^{-1}\bigr\|_2 = & \bigl\| [\diag(\bl_1)] \bL^{-\top} \bigr\|_2 \leq \|\bl_1\|_{\infty} \|\bL^{-1}\|_2 \leq \widetilde{C} r^{-\frac{1}{2}} \end{aligned} \]
		for some constant $\widetilde{C} > 0$.
		In addition,
		\beq \label{aha} \begin{aligned} \bigl\| \widehat\bB^{-1} - \bB^{-1} \bigr\|_2 = & \bigl\| \widehat\bB^{-1}\bigl(\widehat\bB - \bB)\bB^{-1} \bigr\|_2 \leq \bigl\| \widehat\bB^{-1}\bigr\|_2 \bigl\|\bB^{-1} \bigr\|_2 \bigl\| \widehat\bB - \bB \bigr\|_2 \\ \leq & \Bigl( \bigl\| \bB^{-1}\bigr\|_2 + \bigl\| \widehat\bB^{-1} - \bB^{-1}\bigr\|_2 \Bigr) \bigl\|\bB^{-1} \bigr\|_2 \bigl\| \widehat\bB - \bB \bigr\|_2. \end{aligned} \eeq
		Lemma \ref{VertexHunting} shows that if $Err \leq \alpha' \sqrt{r}$, \[ \max_{1 \leq k \leq r}\bigl\| \bOmega\widehat\bb_k - \bb_k \bigr\|_2 \leq \alpha \cdot Err, \]
		therefore,
		\[ \bigl\| \widehat\bB - \bB \bigr\|_2 \leq \sqrt{r} \max_{1 \leq k \leq r} \bigl\| \bOmega\widehat\bb_k - \bb_k \big\|_2 \leq \alpha \sqrt{r} \cdot Err. \]
		The inequality \eqref{aha} can be reduced to
		\[ \begin{aligned} & \bigl\| \widehat\bB^{-1} - \bB^{-1} \bigr\|_2 
		\leq \bigl( \alpha \widetilde{C} \cdot Err \bigr)\bigl\| \widehat\bB^{-1} - \bB^{-1} \bigr\|_2 + \alpha \widetilde{C}^2 r^{-\frac{1}{2}} \cdot Err. \end{aligned} \]
		Under condition $Err \leq \frac{1}{2}\alpha^{-1}\widetilde{C}^{-1}$, we further have
		\beq \label{DB} \begin{aligned} & \bigl\| \widehat\bB^{-1} - \bB^{-1} \bigr\|_2 \leq 2\alpha \widetilde{C}^2r^{-\frac{1}{2}} \cdot Err, \\ & \bigl\|\widehat\bB^{-1}\bigr\|_2 \leq \bigl\|\bB^{-1}\bigr\|_2 + \bigl\| \widehat\bB^{-1} - \bB^{-1} \bigr\|_2 \leq 2 \widetilde{C} r^{-\frac{1}{2}}. \end{aligned} \eeq
		
		Plugging \eqref{DB} into \eqref{DW} gives
		\beq \label{w*} \bigl\| \widehat\bw_j^* - \bw_j \bigr\|_2 \leq \widetilde{C}' r^{-\frac{1}{2}} \cdot Err \eeq
		for some constant $\widetilde{C}' > 0$.
		
		
		Recall that
		\[ \widehat\bW = [\widehat\bw_1,\ldots,\widehat\bw_p]^{\top} \]
		with
		\[ \widehat\bw_j = \bigl\| \bigl[ \widehat\bw_j^* \bigr]_+ \bigr\|_1^{-1} \bigl[ \widehat\bw_j^* \bigr]_+.\]
		We have
		\beq \label{DW2} \begin{aligned} \bigl\| \widehat\bw_j - \bw_j \bigr\|_1 \leq & \bigl\| \widehat\bw_j - \bigl[\widehat\bw^*_j\bigr]_+ \bigr\|_1 + \bigl\| \bigl[\widehat\bw^*_j\bigr]_+ - \bw_j \bigr\|_1 \\ \leq & \Bigl| \bigl\| \bigl[ \widehat\bw_j^* \bigr]_+ \bigr\|_1 - 1 \Bigr| + \bigl\| \bigl[\widehat\bw^*_j\bigr]_+-\bw_j\bigr\|_1 \\ \leq & 2 \bigl\| \bigl[\widehat\bw^*_j\bigr]_+-\bw_j\bigr\|_1 \leq 2 \bigl\| \widehat\bw_j^*-\bw_j\bigr\|_1 \\ \leq & 2\sqrt{r} \bigl\| \widehat\bw_j^*-\bw_j\bigr\|_2 \leq 2\widetilde{C}' \cdot Err. \end{aligned} \eeq

		%
		%
		Define
		\[ \widehat\bV^{\circ} = [\diag(\omega\widehat\bh_1)][\diag(n^{-1}\bm)]^{\frac{1}{2}}\widehat\bW, \]
		and 
		\[ \bV^{\circ} = [\diag(\bh_1)][\diag(\bpi)]^{\frac{1}{2}}\bW = \bV[\diag(\bl_1)]. \]
		We calculate the row-wise $\ell_1$-distance between $\widehat\bV^{\circ}$ and $\bV^{\circ}$,
		\beq \label{DV0} \begin{aligned} 
			\bigl\| \be_j^{\top}\bigl(\widehat\bV^{\circ} - \bV^{\circ} \bigr) \bigr\|_1 = & \Bigl\|  \omega \widehat\bh_1(j) \sqrt{n^{-1}m_j} \cdot \widehat\bw_j - \bh_1(j)\sqrt{\pi_j} \cdot \bw_j \Bigr\|_1 \\
			\leq & \bigl| \omega\widehat\bh_1(j) - \bh_1(j) \bigr| \cdot \sqrt{n^{-1}m_j} \cdot \bigl\| \widehat\bw_j \bigr\|_1 \\ & + \bigl| \bh_1(j) \bigr| \cdot \Bigl| \sqrt{n^{-1}m_j} - \sqrt{\pi_j} \Bigr| \cdot \bigl\| \widehat\bw_j \bigr\|_1 \\ & + \bigl|\bh_1(j)\bigr| \cdot \sqrt{\pi_j} \cdot \bigl\| \widehat\bw_j - \bw_j \bigr\|_1.
		\end{aligned} \eeq
		Because $|\sqrt{x}-1| \leq |x-1|$ for $x \geq 0$,
		\[ \begin{aligned} & \Bigl| \sqrt{n^{-1}m_j} - \sqrt{\pi_j} \Bigr| = \sqrt{\pi_j}\Bigl| \sqrt{n^{-1}m_j/\pi_j} - 1 \Bigr| \\ \leq & \sqrt{\pi_j}\bigl| (n^{-1}m_j/\pi_j) - 1 \bigr| = \pi_j^{-\frac{1}{2}}\bigl| n^{-1}m_j-\pi_j \bigr|. \end{aligned} \]
		Under the condition $\bigl\|n^{-1}\bm - \bpi\bigr\|_{\infty} \leq c_1p^{-1}$,
		\[ \sqrt{n^{-1}m_j} \leq \sqrt{\pi_j} + \Bigl| \sqrt{n^{-1}m_j} - \sqrt{\pi_j} \Bigr| \leq 2\sqrt{\pi_j}. \]
		We apply $\bigl\| \widehat\bw_j \bigr\|_1 = 1$, $\bigl| \omega\widehat\bh_1(j) - \bh_1(j) \bigr| \leq err$, \eqref{h1} and \eqref{DW2} to \eqref{DV0}, and obtain
		\beq \label{DV0'} \begin{aligned} 
			\bigl\| \be_j^{\top}\bigl(\widehat\bV^{\circ} - \bV^{\circ} \bigr) \bigr\|_1 
			\leq & 2\sqrt{\pi_j} \cdot err + C^{\#} \bigl\| n^{-1}\bm - \bpi \bigr\|_{\infty} + 2\widetilde{C}' C^{\#}\pi_j \cdot Err.
		\end{aligned} \eeq
		Here, $C^{\#}>0$ is the constant in Lemma \ref{Pre}.
		
		In the last step, we normalize each column of $\widehat\bV^{\circ}$ to get $\widehat\bV$. We find that for $k=1,2,\ldots,r$,
		\[ \begin{aligned} & \bigl\| (\widehat\bV - \bV)\be_k \bigr\|_1 = \Bigl\| \bigl\| \widehat\bV^{\circ}\be_k \bigr\|_1^{-1}\widehat\bV^{\circ} \be_k- \bigl\| \bV^{\circ}\be_k \bigr\|_1^{-1}\bV^{\circ}\be_k \Bigr\|_1 \\ \leq & \Bigl| \bigl\| \widehat\bV^{\circ}\be_k \bigr\|_1^{-1}- \bigl\| \bV^{\circ}\be_k \bigr\|_1^{-1} \Bigr| \cdot \bigl\| \widehat\bV^{\circ}\be_k \bigr\|_1  + \bigl\| \bV^{\circ} \be_k \bigr\|_1^{-1} \cdot \bigl\| \bigl(\widehat\bV^{\circ} - \bV^{\circ} \bigr)\be_k \bigr\|_1 \\ = & \Bigl| \bigl\| \widehat\bV^{\circ}\be_k \bigr\|_1- \bigl\| \bV^{\circ}\be_k \bigr\|_1 \Bigr| \cdot \bigl\| \bV^{\circ}\be_k \bigr\|_1^{-1} + \bigl\| \bV^{\circ} \be_k \bigr\|_1^{-1} \cdot \bigl\| \bigl(\widehat\bV^{\circ} - \bV^{\circ} \bigr)\be_k \bigr\|_1 \\ \leq & \bigl\| \widehat\bV^{\circ}\be_k - \bV^{\circ}\be_k \bigr\|_1 \cdot \bigl\| \bV^{\circ}\be_k \bigr\|_1^{-1} + \bigl\| \bV^{\circ} \be_k \bigr\|_1^{-1} \cdot \bigl\| \bigl(\widehat\bV^{\circ} - \bV^{\circ} \bigr)\be_k \bigr\|_1 \\ = & 2 \bigl\| \bV^{\circ} \be_k \bigr\|_1^{-1} \cdot \bigl\| \bigl(\widehat\bV^{\circ} - \bV^{\circ} \bigr)\be_k \bigr\|_1. \end{aligned} \]
		Since by Lemma \ref{B}, $\bigl\|\bV^{\circ}\be_k\bigr\|_1 = \bl_1(k) \geq c^{\#}r^{-1}$ for some constant $c^{\#}>0$,
		\[ \bigl\| (\widehat\bV - \bV)\be_k \bigr\|_1 \leq 2(c^{\#})^{-1}r\cdot \bigl\| \bigl(\widehat\bV^{\circ} - \bV^{\circ} \bigr)\be_k \bigr\|_1.  \]
		We further derive from \eqref{DV0'} that
		\[ \begin{aligned} & \frac{1}{r} \sum_{k=1}^r \bigl\| (\widehat\bV - \bV) \be_k \bigr\|_1
		\leq 2(c^{\#})^{-1}\sum_{k=1}^r \bigl\| \bigl(\widehat\bV^{\circ} - \bV^{\circ} \bigr)\be_k \bigr\|_1 \\ = & 2(c^{\#})^{-1}\sum_{j=1}^p \bigl\| \be_j^{\top}\bigl(\widehat\bV^{\circ} - \bV^{\circ} \bigr)\bigr\|_1 \leq \widetilde{C}'' \Bigl( \sqrt{p} \cdot err + p \bigl\| n^{-1}\bm - \bpi \bigr\|_{\infty} + Err \Bigr) \end{aligned} \]
		for some constant $\widetilde{C}'' > 0$.
	\end{proof}
	
	\subsection{Main Results}
	
	\begin{theorem}[Statistical error bounds for $\bV$, Theorem 3 in paper] \label{thm:Vbound}
		Under assumptions \eqref{Assumption1} - \eqref{Assumption4}, for any $c_0 > 0$, there exists a constant $C > 0$ such that if $n \geq C \tau_*p^{\frac{3}{2}}r\bigl(\log^2(r) \vee 1\bigr)\log^2(n)$, then
		\[ \begin{aligned} \mathbb{P} \biggl( \frac{1}{r} \sum_{k=1}^r \bigl\| (\widehat\bV - \bV)\be_k \bigr\|_1 & \geq C \Bigl(
		n^{-1/2}\sqrt{\tau_*pr\bigl(\log^2(r) \vee 1\bigr)\log^2(n)} \\
		& \qquad + n^{-1}\tau_*p^{\frac{3}{2}}r^{\frac{1}{2}}\bigl(\log^2(r) \vee 1\bigr)\log^2(n)\Bigr) \biggr) \leq n^{-c_0}. \end{aligned} \]
	\end{theorem}
	
	\begin{proof}
		According to Lemma \ref{MC_Concentration}, \ref{Leading_Convergence} and \ref{Perturbation}, for a fixed $c_0 > 0$, there exists a constant $C_0 > 0$, $\omega \in \{\pm 1\}$ and $\bOmega \in \mathbb{O}^{(r-1) \times (r-1)}$ such that, when $n \geq C_0 \tau_* p^{\frac{3}{2}}\bigl(\log^2(r) \vee 1\bigr)\log^2(n)$, with probability at least $1-3n^{-c_0}$,
		\begin{align}
		& \begin{aligned}  Err \leq & \max_{1 \leq j \leq p} \bigl\| \bOmega\widehat\bd_j - \bd_j \bigr\|_2 \\ \leq & C_0 \Bigl(
		n^{-1/2}\sqrt{\tau_*pr\bigl(\log^2(r) \vee 1\bigr)\log^2(n)} + n^{-1}\tau_*p^{\frac{3}{2}}r^{\frac{1}{2}}\bigl(\log^2(r) \vee 1\bigr)\log^2(n)\Bigr), \end{aligned} \label{1}\\
		& \begin{aligned} err \leq & \max_{1 \leq j \leq p}\bigl| \omega\widehat\bh_1(j)-\bh_1(j) \bigr| \\ \leq & C_0 \Bigl( n^{-1/2}\sqrt{\tau_*\bigl(\log^2(r) \vee 1\bigr)\log^2(n)}\label{2} + n^{-1}\tau_*p\bigl(\log^2(r) \vee 1\bigr)\log^2(n)\Bigr), \end{aligned} \\
		& \|n^{-1}\bm - \bpi\|_{\infty} \leq C_0 n^{-1/2} \sqrt{\tau_*p^{-1} \log^2(n)}. \label{3}
		\end{align}
		We can take $\widetilde{C}_0 \geq C_0$ such that when $n \geq \widetilde{C}_0 \tau_* p^{\frac{3}{2}}r \bigl(\log^2(r) \vee 1\bigr) \log^2(n)$,
		\[ Err \leq c^*, \qquad \bigl\| n^{-1}\bm - \bpi \bigr\|_{\infty} \leq c^*p^{-1}. \]
		Here, $c^*$ is the constant in Theorem \ref{Deterministic_Decompose}. Then plugging \eqref{1}, \eqref{2} and \eqref{3} into \eqref{0}, we complete the proof.
		
	\end{proof}


	
	

	\begin{theorem}[Statistical error bounds for $\bU$, Theorem 4 in paper]
		Under assumptions \eqref{Assumption1} - \eqref{Assumption4}, for any $c_0 > 0$, there exists a constant $C > 0$ such that if $n \geq C \tau_*p^{\frac{3}{2}}r\bigl(\log^2(r) \vee 1\bigr)\log^2(n)$, then
		\beq \label{UpperBoundU} \begin{aligned} \mathbb{P} \biggl( \frac{1}{p} \sum_{j=1}^p \bigl\|\be_j^{\top} (\widehat\bU - \bU) \bigr\|_1 & \geq C r^{\frac{3}{2}}\Bigl(
			n^{-1/2}\sqrt{\tau_*pr\bigl(\log^2(r) \vee 1\bigr)\log^2(n)} \\
			& \qquad + n^{-1}\tau_*p^{\frac{3}{2}}r^{\frac{1}{2}}\bigl(\log^2(r) \vee 1\bigr)\log^2(n)\Bigr) \biggr) \leq n^{-c_0}. \end{aligned} \eeq
	\end{theorem}
	
	\begin{remark}
		\normalfont
		Write for short
		\begin{align*}
		\widetilde{err}_n & = C \Bigl(
		n^{-1/2}\sqrt{\tau_*pr\bigl(\log^2(r) \vee 1\bigr)\log^2(n)} + n^{-1}\tau_*p^{\frac{3}{2}}r^{\frac{1}{2}}\bigl(\log^2(r) \vee 1\bigr)\log^2(n)\Bigr). 
		\end{align*}
		Below are a few alternative expressions for \eqref{UpperBoundU}:
		
		With probability at least $1-n^{-c_0}$,
		\begin{itemize}
			\item $p^{-\frac{1}{2}} \bigl\|\widehat\bU - \bU\bigr\|_2 \leq \sqrt{r} \cdot err_n$,
			\item $\sqrt{p^{-1}\sum_{j=1}^p \bigl\| \widehat\bu_j - \bu_j \bigr\|_2^2} \leq r \cdot err_n$,
			\item $p^{-1} \sum_{j=1}^p \bigl\|\widehat\bu_j - \bu_j\bigr\|_1 \leq r^{\frac{3}{2}} \cdot err_n$.
		\end{itemize}
	\end{remark}
	
	\begin{proof}
		By definition,
		\begin{align*}
		\widehat{\bf U} &= [\diag(n^{-1}\bm)]^{-1}(n^{-1}\widehat{\bN})\widehat{\bV}(\widehat{\bV}^\top\widehat{\bV})^{-1}, \cr
		{\bf U} &= [\diag(\bpi)]^{-1}\bF\bV(\bV^\top\bV)^{-1}. 
		\end{align*}
		
		We need some preparations. First, consider the diagonal matrix $[\diag(n^{-1}\bm)]^{-1}$. The assumption \eqref{Assumption1} guarantees $\pi_j\geq c_1p^{-1}$. By \eqref{MC_Concentration_pi}, with probability $1-n^{-c_0}$, $\|n^{-1}\bm-\pi\|_\infty\leq Cp^{-1}r^{-\frac{1}{2}}\cdot\widetilde{err}_n$. It follows that $n^{-1}m_j\geq \pi_j-\|n^{-1}\bm-\pi\|_\infty\geq c_1p^{-1}/2$. Therefore,
		\beq \label{thmU-pi-1}
		\bigl\|[\diag(\bpi)]^{-1}\bigr\|_2\leq Cp, \quad \bigl\|[\diag(n^{-1}\widehat{\bN})]^{-1}\bigr\|_2\leq Cp,
		\eeq
		and
		\begin{align} \label{thmU-pi-2}
		& \bigl\|[\diag(n^{-1}\widehat{\bN})]^{-1}-[\diag(\bpi)]^{-1}\bigr\|_2\cr
		\leq & \bigl\|[\diag(\tfrac{1}{n}\widehat{\bN})]^{-1}\bigr\|_2\bigl\|\diag(\tfrac{1}{n}\widehat{\bN})-\diag(\bpi)\bigr\|_2\bigl\|[\diag(\bpi)]^{-1}\bigr\|\cr
		\leq & Cpr^{-\frac{1}{2}}\cdot \widetilde{err}_n. 
		\end{align} 
		
		Second, consider the matrix $n^{-1}\widehat{\bN}$. By \eqref{MC_Concentration_pi}, with probability $1-n^{-c_0}$, 
		\beq \label{thmU-N-1}
		\|n^{-1}\widehat{\bN}-\bF\|_2\leq Cp^{-1}r^{-\frac{1}{2}}\cdot\widetilde{err}_n. 
		\eeq
		Additionally, by \eqref{normF}, $\|\bF\|_2\leq Cp^{-1}$. Combining it with \eqref{thmU-N-1} gives
		\beq \label{thmU-N-2}
		\|\bF\|_2\leq Cp^{-1}, \qquad \|n^{-1}\widehat{\bN}\|_2\leq Cp^{-1}. 
		\eeq
		
		Next, consider the matrix $\widehat{\bV}$. By \eqref{rowV} and the assumption \eqref{Assumption1}, $\|\be_j^{\top}\bV\|_1\leq Cp^{-1}r$ for all $1\leq j\leq p$. It follows that $\|\bV\|_1=\max_{1\leq k\leq r}\|\bV\be_k\|_1=1$, and $\|\bV\|_\infty=\max_{1\leq j\leq p}\|\be_j^{\top}\bV\|_1\leq Cp^{-1}r$. As a result,
		\beq \label{thmU-V-1}
		\|\bV\|_2\leq \sqrt{\|\bV\|_1\|\bV\|_\infty}\leq Cp^{-\frac{1}{2}}r^{\frac{1}{2}}. 
		\eeq
		By Theorem~\ref{thm:Vbound}, with probability $1-n^{-c_0}$, $
		\sum_{k=1}^r \|(\widehat{\bV}-\bV)\be_k\|_1\leq Cr\cdot \widetilde{err}_n$.
		It immediately gives 
		\beq \label{thmU-temp1}
		\|\widehat{\bV}-\bV\|_1=\max_{1\leq k\leq r}\|(\widehat{\bV}-\bV)\be_k\|_1\leq Cr\cdot \widetilde{err}_n.
		\eeq
		We then bound $\|\widehat{\bV}-\bV\|_\infty$. Let $\bV^{\circ}$ and $\widehat{\bV}^{\circ}$ be the same as in \eqref{DV0'}. We have seen in \eqref{DV0'} that, with probability $1-n^{-c_0}$, $\|\be_j^{\top}(\widehat{\bV}^{\circ}-\bV^{\circ})\|_1\leq Cp^{-1}\widetilde{err}_n$. It follows that
		\begin{align} \label{thmU-2}
		\sum_{k=1}^r \bigl\| \widehat\bV^{\circ}\be_k - \bV^{\circ}\be_k \bigr\|_1
		= \sum_{j=1}^p\|\be_j^{\top}(\widehat{\bV}^{\circ}-\bV^{\circ})\|_1 
		\leq   C\widetilde{err}_n.
		\end{align}
		Additionally, by Lemma~\ref{B}, $\| \bV^{\circ} \be_k \|_1=\bl_1(k)\geq C^{-1}r^{-1}$; as a result,  $\|\widehat{\bV}^{\circ}\be_k\|_1\geq \| \bV^{\circ} \be_k \|_1 -C\widetilde{err}_n\geq C^{-1}r^{-1}$. 
		It is seen that, for each $1\leq j\leq p$,  
		\begin{align*} 
		& \bigl| \be_j^{\top}(\widehat\bV - \bV)\be_k \bigr| = \Bigl| \bigl\|\widehat\bV^{\circ}\be_k \bigr\|_1^{-1} \be_j^{\top}\widehat\bV^{\circ} \be_k- \bigl\| \bV^{\circ}\be_k \bigr\|_1^{-1} \be_j^{\top}\bV^{\circ}\be_k \Bigr| \cr 
		\leq & \Bigl| \bigl\| \widehat\bV^{\circ}\be_k \bigr\|_1^{-1}- \bigl\| \bV^{\circ}\be_k \bigr\|_1^{-1} \Bigr| \cdot \bigl|\be_j^{\top}\widehat\bV^{\circ}\be_k \bigr| + \bigl\| \bV^{\circ} \be_k \bigr\|_1^{-1} \cdot \bigl| \be_j^{\top}\bigl(\widehat\bV^{\circ} - \bV^{\circ} \bigr)\be_k \bigr| \cr 
		= & \bigl\| \bV^{\circ}\be_k \bigr\|_1^{-1} \Bigl| \bigl\| \widehat\bV^{\circ}\be_k \bigr\|_1- \bigl\| \bV^{\circ}\be_k \bigr\|_1 \Bigr|\cdot \bigl\| \widehat{\bV}^{\circ}\be_k \bigr\|_1^{-1}\bigl|\be_j^{\top}\widehat\bV^{\circ}\be_k \bigr|,\cr 
		& +\bigl\| \bV^{\circ} \be_k \bigr\|_1^{-1} \cdot \bigl| \be_j^{\top}\bigl(\widehat\bV^{\circ} - \bV^{\circ} \bigr)\be_k \bigr|  \cr 
		\leq	 & Cr\cdot \bigl\| \widehat\bV^{\circ}\be_k - \bV^{\circ}\be_k \bigr\|_1\cdot Cr\cdot \bigl\|\be_j^{\top}\widehat\bV^{\circ}\bigr\|_1 + Cr\cdot\bigl|\be_j^{\top}(\widehat\bV^{\circ}-\bV^{\circ})\be_k \bigr|.   
		\end{align*} 
		Summing over $k$ on both sides gives
		\begin{align*}
		\bigl\| \be_j^{\top}(\widehat\bV - \bV) \bigr\|_1
		\leq & Cr^2\bigl\|\be_j^{\top}\widehat\bV^{\circ}\bigr\|_1\cdot \sum_{k=1}^r \bigl\| \widehat\bV^{\circ}\be_k - \bV^{\circ}\be_k \bigr\|_1 +Cr\cdot\bigl\|\be_j^{\top}(\widehat\bV^{\circ}-\bV^{\circ})\bigr\|_1\cr
		\leq & Cr^2\cdot \bigl\|\be_j^{\top}\widehat\bV^{\circ}\bigr\|_1\cdot C\widetilde{err}_n + Cr\cdot p^{-1} \widetilde{err}_n,
		\end{align*}
		where the last inequality is from \eqref{DV0'} and \eqref{thmU-2}. 
		Since $\bV^{\circ}=\bV[\diag(\bl_1)]$, we have $\|\be_j^{\top}\bV^{\circ}\|_1\leq \|\be_j^{\top}\bV\|_1\|\bl_1\|_\infty$. By Lemma~\ref{B},  
		$\|\bl\|_\infty\leq Cr^{-1}$; by \eqref{rowV} and the assumption \eqref{Assumption1}, $\|\be_j^{\top}\bV\|_1\leq Cp^{-1}r$ . Hence, $\|\be_j^{\top}\bV^{\circ}\|_1\leq Cp^{-1}$. 
		Plugging it into the above inequality gives
		\[
		\bigl\| \be_j^{\top}(\widehat\bV - \bV) \bigr\|_1\leq Cp^{-1}r^2\cdot \widetilde{err}_n. 
		\]
		It follows that
		\beq \label{thmU-temp2}
		\|\widehat{\bV}-\bV\|_\infty\leq Cp^{-1}r^2\cdot \widetilde{err}_n. 
		\eeq
		Combing \eqref{thmU-temp1} and \eqref{thmU-temp2} gives
		\begin{align} \label{thmU-V-2}
		\|\widehat{\bV}-\bV\|_2 & \leq\sqrt{\|\widehat{\bV}-\bV\|_1\|\widehat{\bV}-\bV\|_\infty}\leq Cp^{-\frac{1}{2}}r^{\frac{3}{2}}\cdot\widetilde{err}_n. 
		\end{align}
		
		Last, we study the matrix $(\widehat{\bV}^{\top}\widehat{\bV})^{-1}$. Since $(\widehat{\bV}^{\top}\widehat{\bV}-\bV^{\top}\bV)$ is a symmetric matrix, 
		\begin{align*} 
		& \|\widehat{\bV}^{\top}\widehat{\bV}-\bV^{\top}\bV\|_2
		\leq \|\widehat{\bV}^{\top}\widehat{\bV}-\bV^{\top}\bV\|_1\cr
		\leq & \|\bV^{\top}(\widehat{\bV}-\bV)\|_1 + \|(\widehat{\bV}-\bV)^{\top}\bV\|_1 + \|(\widehat{\bV}-\bV)^{\top}(\widehat{\bV}-\bV)\|_1\cr
		\leq & \|\bV\|_\infty\|\widehat{\bV}-\bV\|_1 +\|\widehat{\bV}-\bV\|_\infty\|\bV\|_1 + \|\widehat{\bV}-\bV\|_\infty\|\widehat{\bV}-\bV\|_1\cr
		\leq & (Cp^{-1}r)(Cr\cdot\widetilde{err}_n) + (Cp^{-1}r^2\cdot\widetilde{err}_n)\cdot 1 + (Cr\cdot\widetilde{err}_n)(Cp^{-1}r^2\cdot\widetilde{err}_n)\cr
		\leq & Cp^{-1}r^2\cdot\widetilde{err}_n. 
		\end{align*}
		By the assumption \eqref{Assumption2}, $\lambda_{\min}(\bV^{\top}\bV)\geq C^{-1}p^{-1}r$. It further implies that $\lambda_{\min}(\widehat{\bV}^{\top}\widehat{\bV})\geq \lambda_{\min}(\bV^{\top}\bV)- \|\widehat{\bV}^{\top}\widehat{\bV}-\bV^{\top}\bV\|_2\geq C^{-1}p^{-1}r$. In other words,
		\beq \label{thmU-V'V-1}
		\bigl\|(\bV^{\top}\bV)^{-1}\bigr\|_2\leq Cpr^{-1}, \quad \bigl\|(\widehat{\bV}^{\top}\widehat{\bV})^{-1}\bigr\|_2\leq Cpr^{-1}.
		\eeq
		Furthermore,
		\begin{align} \label{thmU-V'V-2}
		& \bigl\|(\bV^{\top}\bV)^{-1}- (\widehat{\bV}^{\top}\widehat{\bV})^{-1}\bigr\|_2\cr
		\leq & \bigl\|(\widehat{\bV}^{\top}\widehat{\bV})^{-1}\bigr\|_2\|\widehat{\bV}^{\top}\widehat{\bV}-\bV^{\top}\bV\|_2\bigl\|(\bV^{\top}\bV)^{-1}\bigr\|_2\cr
		\leq & Cp\cdot \widetilde{err}_n. 
		\end{align}
		
		We now proceed to proving the claim. Using the triangular inequality,
		\begin{align*}
		& \|\widehat{\bU}-\bU\|_2\cr
		\leq & \bigl\|[\diag(n^{-1}\widehat{\bN})]^{-1}-[\diag(\bpi)]^{-1}\bigr\|_2\|n^{-1}\widehat{\bN}\|_2\|\widehat{\bV}\|_2\bigl\|(\widehat{\bV}^{\top}\widehat{\bV})^{-1}\bigr\|_2\cr
		& + \bigl\|[\diag(\bpi)]^{-1}\bigr\|_2\|n^{-1}\widehat{\bN}-\bF\|_2\|\widehat{\bV}\|_2\bigl\|(\widehat{\bV}^{\top}\widehat{\bV})^{-1}\bigr\|_2\cr
		& + \bigl\|[\diag(\bpi)]^{-1}\bigr\|_2\|\bF\|_2\|\widehat{\bV}-\bV\|_2\bigl\|(\widehat{\bV}^{\top}\widehat{\bV})^{-1}\bigr\|_2\cr
		& + \bigl\|[\diag(\bpi)]^{-1}\bigr\|_2\|\bF\|_2\|\bV\|_2\bigl\|(\widehat{\bV}^{\top}\widehat{\bV})^{-1}-(\bV^{\top}\bV)^{-1}\bigr\|_2\cr
		\leq &C(pr^{-\frac{1}{2}}\,\widetilde{err}_n)\cdot p^{-1}\cdot p^{-\frac{1}{2}}r^{\frac{1}{2}}\cdot pr^{-1}\cr
		& + Cp\cdot (p^{-1}r^{-\frac{1}{2}}\,\widetilde{err}_n)\cdot p^{-\frac{1}{2}}r^{\frac{1}{2}}\cdot pr^{-1}\cr
		& + Cp\cdot p^{-1}\cdot (p^{-\frac{1}{2}}r^{\frac{3}{2}}\widetilde{err}_n)\cdot pr^{-1}\cr
		& + Cp\cdot p^{-1}\cdot p^{-\frac{1}{2}}r^{\frac{1}{2}}\cdot (p\,\widetilde{err}_n)\cr
		\leq & C\sqrt{pr}\cdot \widetilde{err}_n. 
		\end{align*}
		It follows that
		\begin{align*}
		& p^{-1}\sum_{j=1}^p \|\be_j^{\top}(\widehat{\bU}-\bU)\|_2^2
		= p^{-1}\|\widehat{\bU}-\bU\|_F^2
		\leq p^{-1}\cdot (2r)\cdot \|\widehat{\bU}-\bU\|_2^2
		\leq Cr^2\cdot \widetilde{err}_n. 
		\end{align*}
	\end{proof}

	
	\begin{theorem}[Recovery of anchor states, Theorem 5 in paper]
		Let $\mathcal{N}$ be the set of non-anchor states. Denote \[ \phi = \max_{j \in \mathcal{N}} \max_{1 \leq k \leq r} \mathbb{P}_{X_0 \sim \bpi}\bigl(Z_t = k \mid X_{t+1} = j\bigr). \] Under conditions \eqref{Assumption1} - \eqref{Assumption4}, for a fixed $c_0 > 0$, there exist constants $c > 0$ and $C > 0$ such that when $\delta_0 \leq c(1-\phi)$ and \[ n \geq C\delta_0^{-2}\tau_*p^{\frac{3}{2}}r\bigl(\log^2(r) \vee 1\bigr) \log^2(n),\] we can successfully identify the anchor states with probability at least $1-n^{-c_0}$.
	\end{theorem}
	
	\begin{proof}
		\newcommand{\bzeta}{{\boldsymbol{\zeta}}}
		Define  \[ \bzeta_j = \pi_j^{-1} \bigl[\diag(\bU^{\top}\bpi)\bigr] \bV^{\top} \be_j, \quad j = 1,2,\ldots,p.\] Then for $k=1,2,\ldots,r$,
		\[ \bzeta_j(k) = \mathbb{P}_{X_0 \sim \bpi} \bigl(Z_t = k \mid X_{t+1} = j\bigr). \]
		We first present a useful fact that, for any two vectors $\bx, \bx'$ in the $r$-dimensional standard simplex $\mathcal{S}_{r-1}$, by triangle inequality,
		\beq \label{l1} \begin{aligned} & \bigl\| \bx' - \bx \bigr\|_1 = \bigl| \bx'(k) - \bx(k) \bigr| + \sum_{l \neq k} \bigl| \bx'(l) - \bx(l) \bigr| \\ \geq & \bigl| \bx'(k) - \bx(k) \bigr| + \Big| \sum_{l \neq k} \bigl( \bx'(l) - \bx(l) \bigr) \Big| = 2\bigl| \bx'(k) - \bx(k) \bigr|, \quad k = 1,2,\ldots,r. \end{aligned} \eeq
		If $\bx = \be_k$ then the equality holds.
		According to \eqref{l1}, the parameter $\phi$ in the theorem can be equivalently defined as
		\[ \begin{aligned} 1 - \phi = & \min_{j \in \mathcal{N}} \min_{1 \leq k \leq r} \bigl(1 - \bzeta_j(k)\bigl) = \frac{1}{2} \min_{j \in \mathcal{N}} \min_{1 \leq k \leq r} \bigl\| \bzeta_j - \be_k \bigl\|_1. \end{aligned} \]

		Denote \[ \bA = \bigl[\diag(\bU^{\top}\bpi)\bigr][\diag(\bl_1)]^{-1}. \] 
		Since
		\[ \bw_j = \pi_j^{-\frac{1}{2}}[\bh_1(j)]^{-1}[\diag(\bl_1)]\bV^{\top} \be_j, \]
		we have $\bzeta_j = \bigl\|\bA \bw_j\bigr\|_1^{-1} \bA \bw_j$ for $j = 1,2,\ldots,p$.
		By \eqref{xi_LowerBound}, Lemma \ref{B} and assumption \eqref{Assumption5}, there exist constants $\widetilde{c}>0$ and $\widetilde{C}>0$ such that the diagonal entries of $\bA$ satisfies
		\[ \widetilde{c} \leq  \bigl(\bU^{\top}\bpi\bigr)_k\bigl[\bl_1(k)\bigr]^{-1} \leq \widetilde{C}, \quad \text{for $k=1,2,\ldots,r$}. \]
		
		We first derive a lower bound for $\bigl\|\bw_j - \be_k\bigr\|_1$, $j \in \mathcal{N}$, using $\phi$.
		For any $\bw, \bw' \in \mathcal{S}_{r-1}$, let $\bzeta = \bigl\|\bA\bw\bigr\|_1^{-1}\bA\bw$, $\bzeta' = \bigl\|\bA\bw'\bigr\|_1^{-1}\bA\bw'$.
		Because $\bw,\bw' \in \mathcal{S}_{r-1}$,
		\[ \begin{aligned} & \bigl\| \bA\bw \bigr\|_1 \geq \widetilde{c}, \qquad \bigl\| \bA\bw' \bigr\|_1 \geq \widetilde{c}, \qquad \bigl\| \bA\bigl(\bw - \bw'\bigr)\bigr\|_1 \leq \widetilde{C} \bigl\|\bw - \bw'\bigr\|_1. \end{aligned} \]
		We find that
		\[ \begin{aligned} \bigl\| \bzeta - \bzeta' \bigr\|_1 = & \Bigl\| \bigl\| \bA\bw\bigr\|_1^{-1} \bA\bw - \bigl\| \bA\bw'\bigr\|_1^{-1} \bA\bw' \Bigr\|_1 \\ \leq & \bigl\| \bA \bw \bigr\|_1^{-1} \bigl\| \bA \bigl(\bw - \bw'\bigr) \bigr\|_1 + \Bigl| \bigl\| \bA \bw \bigr\|_1^{-1} - \bigl\| \bA \bw' \bigr\|_1^{-1} \Bigr| \cdot \bigl\|\bA\bw' \bigr\|_1 \end{aligned} \]
		Here, the second term
		\[ \begin{aligned} \Bigl| \bigl\| \bA \bw \bigr\|_1^{-1} - \bigl\| \bA \bw' \bigr\|_1^{-1} \Bigr| \cdot \bigl\| \bA\bw' \bigr\|_1 \leq & \Bigl| \bigl\| \bA \bw \bigr\|_1 - \bigl\| \bA \bw \bigr\|_1 \Bigr| \cdot \bigl\| \bA\bw \bigr\|_1^{-1} \\ \leq & \bigl\| \bA\bw \bigr\|_1^{-1} \bigl\| \bA \bigl(\bw -  \bw'\bigr) \bigr\|_1. \end{aligned} \]
		Therefore,
		\beq \label{VH1} \begin{aligned}  \bigl\| \bzeta - \bzeta' \bigr\|_1  \leq 2 \bigl\| \bA\bw \bigr\|_1^{-1} \bigl\| \bA \bigl(\bw -  \bw'\bigr) \bigr\|_1 \leq 2 \widetilde{C}\widetilde{c}^{-1} \bigl\|\bw - \bw' \bigr\|_1, \end{aligned} \eeq
		and
		\beq \label{c5} \begin{aligned} & \min_{j \in \mathcal{N}} \min_{1 \leq k \leq r} \bigl\| \bw_j - \be_k \bigr\|_1 \geq \frac{1}{2}\widetilde{C}^{-1}\widetilde{c}\min_{j \in \mathcal{N}} \min_{1 \leq k \leq r} \bigl\| \bzeta_j - \be_k \bigr\|_1 = \widetilde{C}^{-1}\widetilde{c} (1-\phi). \end{aligned} \eeq
		
		
		We now consider the perturbation bound for $\widehat\bw_j$. According to Lemma \ref{Perturbation}, for a fixed $c_0 > 0$, there exists a constant $C_0>0$ such that 	if $n \geq C_0 \tau_*p^{\frac{3}{2}}\bigl(\log^2(r) \vee 1)\log^2(n)$, with probability at least $1-n^{-c_0}$,
		\[ \begin{aligned} Err \leq C_0 \Bigl( &
		n^{-1/2}\sqrt{\tau_*pr\bigl(\log^2(r) \vee 1\bigr)\log^2(n)} \\
		& + n^{-1}\tau_*p^{\frac{3}{2}}r^{\frac{1}{2}}\bigl(\log^2(r) \vee 1\bigr)\log^2(n) \Bigr). \end{aligned} \]
		We further take $n > (c^*)^{-2}(C_0^{\frac{1}{2}}+1)^{-1}C_0 \tau_*p^{\frac{3}{2}}r\bigl(\log^2(r) \vee 1)\log^2(n)$,
		then $Err < c^*$, where $c^*$ is the constant in Theorem \ref{Deterministic_Decompose}. Theorem \ref{Deterministic_Decompose} then implies that
		\[ \max_{1 \leq j \leq p} \bigl\|\widehat\bw_j - \bw_j\bigr\|_1 \leq C^* \cdot Err. \]
		There exists a constant $\widetilde{C}_0 \geq (c^*)^{-2}(C_0^{\frac{1}{2}}+1)^{-1}C_0$ such that when \[ n \geq \widetilde{C}_0 \delta_0^{-2}\tau_*p^{\frac{3}{2}}r\bigl(\log^2(r) \vee 1\bigr) \log^2(n),\] we have
		\beq \label{kkk}  \max_{1 \leq j \leq p} \bigl\|\widehat\bw_j - \bw_j\bigr\|_1  \leq 2 \delta_0.  \eeq
		
		
		Suppose that $j$ is an anchor state for meta-state $k$. Then $\bw_j = \be_k$.
		Under \eqref{kkk}, we use \eqref{l1} and obtain
		\[ \widehat\bw_j(k) = 1 - \bigl( 1- \widehat\bw_j(k)\bigr) = 1 - \frac{1}{2}\bigl\|\widehat\bw_j - \bw_j \bigr\|_1 \geq 1 - \delta_0. \]
		
		Consider the case where $j \in \mathcal{N}$. Suppose that $\delta_0 \leq 4^{-1}\widetilde{C}^{-1}\widetilde{c}(1-\phi)$. Then by \eqref{c5}, for $k = 1,2,\ldots,r$,
		\[ 1 - \bw_j(k) = \frac{1}{2}\bigl\|\bw_j - \be_k \bigr\|_1 \geq \frac{1}{2}\widetilde{C}^{-1}\widetilde{c}(1-\phi) \geq 2\delta_0.  \]
		It follows from \eqref{kkk} that
		\[ \begin{aligned} \widehat\bw_j(k) \leq & 1 - \bigl( 1 - \bw_j(k) \bigr) + \bigl| \widehat\bw_j(k) - \bw_j(k) \bigr| \\ \leq & 1 - 2\delta_0 + \frac{1}{2}\bigl\| \widehat\bw_j - \bw_j\bigr\|_1 \leq 1 - \delta_0. \end{aligned} \]

	\end{proof}
	
	\section{Explanation of Main Algorithm}
	
	In this section, we explain the rationale of Algorithm 1, especially for:
	\begin{itemize}
		\item  Why the SCORE normalization \cite{SCORE} produces a simplex geometry.
		\item  How the simplex geometry is used for estimating $\bV$.
	\end{itemize}

	Without loss of genrality, we assume that all entires of the stationary distribution $\bpi \in \mathbb{R}^p$ are positive. The normalized data matrix 
	\beq \label{signalMat}
	\widetilde{\bN}\approx n^{\frac{1}{2}}\mathrm{diag}(\bpi)\bP[ \mathrm{diag}(\bpi)]^{-1/2}\equiv n^{\frac{1}{2}}\bQ. 
	\eeq
	The matrix $\bQ$ can be viewed as the ``signal'' part of $\widetilde{\bN}$. Let $\bh_1,\ldots,\bh_r$ be the right singular vectors of $\mathrm{diag}(\bpi)\bP[ \mathrm{diag}(\bpi)]^{-1/2}$. They can be viewed as the population counterpart of $\hbh_1,\ldots,\hbh_r$. We define a population counterpart of the matrix $\widehat{\bD}$ produced by SCORE:
	\beq
	\bD = [\mathrm{diag}(\bh_1)]^{-1}[\bh_2,\ldots,\bh_r] = \bigl[ \bd_1, \bd_2, \ldots, \bd_p \bigr]^{\top}. 
	\eeq
	From now on, we pretend that the matrix $\bQ$ is directly given and study the geometric structures associated with the singular vectors and the SCORE matrix $\bD$.   

	\subsection{The Simplex Geometry and Explanation of Steps of Algorithm 1}
	
	When $\bP=\bU\bV^\top$, the matrix $\bQ$, defined in \eqref{signalMat}, also admits a low-rank decomposition:
	\[
	\bQ=\bU^*(\bV^*)^\top, 
	\]
	where
	\[
	\bU^*=[\mathrm{diag}(\bpi)]\bU, \qquad \bV^*=[\mathrm{diag}(\bpi)]^{-1/2}\bV.
	\]
	The span of the right singular vectors $\bh_1,\ldots,\bh_r$ is the same as the column space of $\bV^*$. It implies there exists a linear transformation $\bL\in\mathbb{R}^{r\times r}$ such that
	\beq \label{HandVstar}
	\bH\equiv[\bh_1,\ldots,\bh_r] = \bV^*\bL. 
	\eeq
	Since $\bV^*$ is a nonnegative matrix, each row of $\bH$ is an affine combination of rows of $\bL$. Furthermore, if $j$ is an anchor state, then the $j$-th row of $\bV^*$ has exactly one nonzero entry, and so the $j$-th row of $\bH$ is proportional to one row of $\bL$. This gives rise to the following simplicial-cone geometry:
	\begin{Proposition}[Simplicial cone geometry] \label{prop:simplicialcone}
		Suppose $\bP=\bU\bV^\top$, each meta-state has an anchor state, and $\mathrm{rank}(\bU)=r$.  
		Let $\bH=[\bh_1,\ldots,\bh_r]$ contain the right singular vectors of $\bQ=\mathrm{diag}(\bpi) \bP [\mathrm{diag}(\bpi)]^{-1/2}$. There exists a simplicial cone in $\mathbb{R}^r$, which has $r$ extreme rays, such that all rows of $\bH$ are contained in this simplicial cone. Furthermore, for all anchor states $j$ of a meta-state, the $j$-th row of $\bH$ lies exactly on one extreme ray of this simplicial cone. 
	\end{Proposition}
	
	{\bf Remark}. Similar simplicial-cone geometry has been discovered in the literature of nonnegative matrix factorization \cite{donoho2004does}. The simplicial cone there is associated with rows of the matrix that admits a nonnegative factorization, but the simplicial cone here is associated with singular vectors of the matrix. Since SVD is a linear projection, it is not surprising that the simplicial cone structure is retained in singular vectors.

	However, in the real case, we have to apply SVD to the noisy matrix, then the simplicial cone is corrupted by noise and hardly visible. We hope to find a proper normalization of $\bH$, so that the normalized rows are all contained in a simplex, where all points on the extreme ray of the previous simplicial cone (these points do not overlap) fall onto one vertex of the current simplex (these points now overlap).  Such a simplex geometry is much more robust to noise corruption and is easier to estimate.

	How to normalize $\bH$ to obtain a simplex geometry is tricky. If all entries of $\bH$ are nonnegative, we can normalize each row of $\bH$ by the $\ell^1$-norm of that row, and rows of the resulting matrix are contained in a simplex. However, $\bH$ consists of singular vectors and often has negative entries, so such a normalization doesn't work. 
	
	By Perron-Frobenius theorem in linear algebra, the leading right singular vector $\bh_1$ have all positive coordinates. It turns out that normalizing each row of $\bH$ by the corresponding coordinate of $\bh_1$ is a proper normalizaiton that will produce a simplex geometry. This is the idea of SCORE \cite{SCORE,mixed-SCORE,Topic-SCORE}. See Figure 2 in the paper for illustration.
	
	\begin{Proposition}[Post-SCORE simplex geometry] \label{prop:simplex}
		In the setting of Proposition~\ref{prop:simplicialcone}, additionally, we assume $\bh_1$ have all positive coordinates ({\it e.g.}, $\bQ^\top \bQ$ is an irreducible matrix). Consider the $p\times (r-1)$ matrix ${\bf D}=[\mathrm{diag}(\bh_1)]^{-1}[\bh_2,\ldots,\bh_r]$. Then, there exists a simplex ${\cal S}_0^*\subset\mathbb{R}^{r-1}$, which has $r$ vertices ${\bf b}_1,\ldots,{\bf b}_r$, such that all rows of ${\bf D}$ are contained in this simplex. Furthermore, for all anchor states $j$ of a same meta-state, the $j$-th row of ${\bf D}$ falls exactly onto one vertex of this simplex. 
	\end{Proposition}

	Proposition~\ref{prop:simplex} explains the rationale of the vertex hunting step. The vertex hunting step we used was borrowed from \cite{mixed-SCORE,Topic-SCORE}; see explanations therein. 
	
	Let ${\bf b}_1,\ldots,{\bf b}_r$ be the vertices of the simplex ${\cal S}_0^*$. 
	By vertex hunting, we obtain estimates of these vertices. The next question is: 
	How can we recover $\bV$ from the simplex vertices ${\bf b}_1,\ldots,{\bf b}_r$?

	Let ${\bf d}_j^\top$ be the $j$-th row of $\bD$, for $j\in[p]$. 
	By the nature of a simplex, each point in it can be uniquely expressed as a convex combination of the vertices. This means, for each $j\in[p]$, there exists a weight vector ${\bf w}_j$ from the standard simplex such that
	\[
	{\bf d}_j = \sum_{k=1}^r \bw_j(k){\bf b}_k. 
	\]
	The next proposition shows that we can recover $\bV$ from ${\bf w}_1,\ldots,{\bf w}_p$.
	
	\begin{Proposition}[Relation of simplex and matrix $\bV$] \label{prop:weight}
		In the setting of Proposition~\ref{prop:simplex}, each row of $\bD$ is a convex combination of the vertices of ${\cal S}_0^*$, i.e., for each $j\in[p]$, there exists ${\bf w}_j$ in the standard simplex such that $\bd_j = \sum_{k=1}^r \bw_j(k){\bf b}_k$. Furthermore, consider the matrix $\bW\in\mathbb{R}^{p\times r}$, whose $j$-th row equals to ${\bf w}_j^\top$. Then, $\bW$ and $\bV$ are connected by 
		\[
		[\mathrm{diag}(\bh_1)][\mathrm{diag}(\bpi)]^{1/2}\bW = \bV[\mathrm{diag}({\bf l}_1)], 
		\]
		where ${\bf l}_1$ is the first column of $\bL$ as defined in \eqref{HandVstar}.  
	\end{Proposition}

	By Proposition~\ref{prop:weight}, each column of the matrix
	\[
	[\mathrm{diag}(\bh_1)][\mathrm{diag}(\boldsymbol{\pi})]^{1/2}\bW
	\]
	is proportional to the corresponding column of $\bV$. Since each column of $\bV$ has a unit $\ell^1$-norm, if we normalize each column of the above matrix by its $\ell^1$-norm, we can exactly recover $\bV$. 
	
	Once we have obtained $\bV$, we can immediately recover $\bU$ from $(\bP, \bV)$ by the relation:
	\[
	\bU =\bU(\bV^\top\bV)(\bV^\top\bV)^{-1}=\bP\bV(\bV^\top\bV)^{-1}. 
	\]
	The above gives the following theorem:
	\begin{Proposition}[Exact recovery of $\bU$ and $\bV$] \label{prop:noiseless}
		In the setting of Proposition~\ref{prop:simplex}, if we apply Algorithm~1 to the matrix $[\mathrm{diag}(\bpi)]\bP[ \mathrm{diag}(\bpi)]^{-1/2}$, it exactly outputs $\bU$ and $\bV$. 
	\end{Proposition}

	\subsection{Proof of propositions}
	Proposition~\ref{prop:simplicialcone} follows from \eqref{HandVstar} and definition of simplicial cone. Proposition~\ref{prop:noiseless} is proved in Section 1.1. We now prove Propositions~\ref{prop:simplex}-\ref{prop:weight}. Recall that by \eqref{HandVstar}, for $k\in[r]$,
	\[
	\bh_k = \bV^*{\bf l}_k, 
	\]
	where ${\bf l}_k$ is the $k$-th column of $\bL$. When all the coordinates of $\bh_1$ are strictly positive, the matrix $\bD$ is well-defined. Additionally, for an anchor state $j$ of the $k$-th meta state, ${\bf h}_1(j)=V^*_{jk}{\bf l}_1(k)$, where $V^*_{jk}>0$. Therefore, ${\bf l}_1(k)>0$ for $k\in[r]$. We define a matrix
	\[
	\bB = [\mathrm{diag}({\bf l}_1)]^{-1}[{\bf l}_2,\ldots,{\bf l}_r]. 
	\]
	By definition,
	\beq \label{prop-1}
	[{\bf 1}, \bD] =  [\mathrm{diag}(\bh_1)]^{-1}\bH, 
	\eeq
	and 
	\beq \label{prop-2}
	[{\bf 1}, \bB] = [\mathrm{diag}({\bf l}_1)]^{-1}\bL. 
	\eeq
	Combining them with \eqref{HandVstar} gives
	\beq \label{prop-3}
	[{\bf 1}, \bD] = [\mathrm{diag}(\bh_1)]^{-1}\bV^*[\mathrm{diag}({\bf l}_1)][{\bf 1}, \bB]. 
	\eeq
	Let 
	\[
	\bW = [\mathrm{diag}(\bh_1)]^{-1}\bV^*[\mathrm{diag}({\bf l}_1)].
	\]
	The \eqref{prop-3} implies
	\[
	{\bf 1} = \bW{\bf 1}, \quad \bD = \bW\bB. 
	\]
	Since $\bW$ is a nonnegative matrix, the first equation implies that each row of $\bW$ is from the standard simplex, and the second equation implies that each row of $\bD$ is a linear combination of the $r$ rows of $\bB$, where the combination coefficients come from the corresponding row of $\bW$. This has proved the simplex geometry stated in Proposition~\ref{prop:simplex}. 
	
	Note that the $j$-th row of $\bD$ is located on one vertex of the simplex if and only if the $j$-th row of $\bW$ is located on one vertex of the standard simplex. From the way we define $\bW$, its $j$-th row equal to  
	\[
	{\bf w}^\top_j = \frac{1}{{\bf h}_1(j)\sqrt{\pi_j}} [V_{j1}{\bf l}_1(1), V_{j2}{\bf l}_1(2), \ldots, V_{jr}{\bf l}_1(r)]. 
	\]
	Since ${\bf h}_1$, ${\bf l}_1$ and $\bpi$ are all positive vectors, ${\bf w}_j$ is located on one vertex of the standard simplex if and only if exactly one of $V_{j1},\ldots,V_{jr}$ is nonzero, where the latter is true if and only if $j$ is an anchor state. This has proved Proposition~\ref{prop:simplex}.  
	
	Furthermore, from the way $\bW$ is defined above, using the fact that $\bV^*=[\mathrm{diag}(\bpi)]^{-1/2}\bV$, we immediately find that 
	\[
	\bW = [\mathrm{diag}(\bh_1)]^{-1}[\mathrm{diag}(\bpi)]^{-1/2}\bV[\mathrm{diag}({\bf l}_1)], 
	\]
	which is equivalent to
	\[
	[\mathrm{diag}(\bh_1)][\mathrm{diag}(\bpi)]^{1/2}\bW = \bV[\mathrm{diag}({\bf l}_1)]. 
	\]
	This has proved Proposition~\ref{prop:weight}. 
	
	\section{Technical Proofs}\label{TechnicalProofs}
	
	\begin{proof}[Proof of Lemma \ref{sigma_r}]
		Since $\bpi$ is a stationary distribution, the frequency matrix $\bF$ satisfies \beq\label{begin1}  \sum_{i=1}^p F_{ij} = \sum_{i=1}^p \pi_i P_{ij} = \pi_j \eeq for $j=1,2,\ldots,p$.
		Because $\sum_{j=1}^p P_{ij} = 1$ for $i=1,2,\ldots,p$,
		\beq \label{begin2} \sum_{j=1}^p F_{ij} = \sum_{i=1}^p \pi_i  P_{ij} = \pi_i. \eeq It follows that \[ \|\bF\|_1 = \max_{1 \leq j \leq p}\sum_{i=1}^p F_{ij} = \max_{1 \leq j \leq p} \pi_j = \pi_{\max}, \quad \|\bF\|_{\infty} = \max_{1 \leq i \leq p}\sum_{j=1}^p F_{ij} = \max_{1 \leq i \leq p}\pi_i = \pi_{\max},\] which further implies \beq\label{normF} \| \bF \|_2 \leq \sqrt{\| \bF \|_{\infty}\| \bF \|_1} = \pi_{\max} \leq C_1 p^{-1}. \eeq
		Therefore, \[ \begin{aligned}
		\sigma_1 = \|\bQ\|_2 = \bigl\|\bF[\diag(\bpi)]^{-\frac{1}{2}}\bigr\|_2  \leq  \pi_{\min}^{-\frac{1}{2}}\|\bF\|_2 \leq C_1c_1^{-\frac{1}{2}}p^{-\frac{1}{2}}.\end{aligned}\]
		
		As for the smallest singular value $\sigma_r$, by definition,
		\[ \begin{aligned} \sigma_r = \min_{\bx\in\mathbb{S}^{p-1}} \bigl\| \bQ^{\top}\bx \bigr\|_2 =  \min_{\bx\in\mathbb{S}^{p-1}} \bigl\|[\diag(\bpi)]^{-\frac{1}{2}}\bV\bU^{\top}[\diag(\bpi)]\bx\bigr\|_2. \end{aligned} \]
		Since \[ \begin{aligned}  \bigl\| [\diag(\bpi)]^{-\frac{1}{2}}\bV\bU^{\top}[\diag(\bpi)]\bx \bigr\|_2 \geq  \sigma_{\min}\bigl([\diag(\bpi)]^{-\frac{1}{2}} \bV\bigr) \bigl\| \bU^{\top}[\diag(\bpi)]\bx \bigr\|_2 \end{aligned} \] and \[ \bigl\| \bU^{\top}[\diag(\bpi)]\bx \bigr\|_2 \geq \sigma_{\min}\bigl(\bU^{\top}[\diag(\bpi)]\bigr)\|\bx\|_2,\] we have
		\[ \begin{aligned}
		\sigma_r \geq & \sigma_{\min}\bigl([\diag(\bpi)]^{-\frac{1}{2}} \bV\bigr)  \sigma_{\min} \bigl(\bU^{\top}[\diag(\bpi)]\bigr) \\
		= & \lambda_{\min}^{\frac{1}{2}}\bigl( \bV^{\top}[\diag(\bpi)]^{-1} \bV\bigr) \lambda_{\min}^{\frac{1}{2}}\bigl(\bU^{\top}[\diag(\bpi)]^2\bU\bigr) \geq c_2p^{-\frac{1}{2}},
		\end{aligned} \]
		where the last inequality holds due to \eqref{Assumption2}
	\end{proof}

	\begin{proof}[Proof of Lemma \ref{RowNorm}]
		We first consider the rows of right singular matrix ${\bf H}$.
		The columns of $[\diag(\bpi)]^{-\frac{1}{2}}\bV$ and $\bH$ span the same linear space. Hence, there exists a nonsingular matrix $\bL \in \mathbb{R}^{r \times r}$ such that 
		\beq \label{Def_L} \bH = [\diag(\bpi)]^{-\frac{1}{2}} \bV \bL. \eeq
		We plug \eqref{Def_L} into $\bH^{\top}\bH = \bI_r$ and obtain $\bL^{\top}\bV^{\top}[\diag(\bpi)]^{-1} \bV\bL = \bI_r$. Multiplying $\bL$ on the left and $\bL^{\top}$ on the right gives $\bL\bL^{\top}\bV^{\top}[\diag(\bpi)]^{-1} \bV\bL\bL^{\top} = \bL\bL^{\top}$. Because $\bL\bL^{\top}$ is non-singular, \[ \bL\bL^{\top} = \bigl( \bV^{\top}[\diag(\bpi)]^{-1}\bV \bigr)^{-1}.\] As a result, \beq\label{normL}\|\bL\|_2 = \lambda_{\min}^{-1/2}\bigl( \bV^{\top}[\diag(\bpi)]^{-1}\bV \bigr) \leq c_2^{-\frac{1}{2}}r^{-\frac{1}{2}},\eeq
		which implies that for any $j = 1,2,\ldots,p$,
		\beq\label{H_Row-wiseNorm} \begin{aligned}
			& \bigl\| \be_j^{\top} \bH \bigr\|_2 = \bigl\| \be_j^{\top} [\diag(\bpi)]^{-\frac{1}{2}} \bV \bL \bigr\|_2 \leq \pi_j^{-\frac{1}{2}}\bigl\| \be_j^{\top} \bV \bigr\|_2 \|\bL\|_2 \\
			\leq & \pi_j^{-\frac{1}{2}}\bigl\| \be_j^{\top} \bV \bigr\|_1 \|\bL\|_2 \leq c_2^{-\frac{1}{2}} \pi_j^{-\frac{1}{2}} r^{-\frac{1}{2}} \cdot \bigl\|\be_j^{\top}\bV\bigr\|_1. 
		\end{aligned} \eeq
		It only remains to estimate $\big\| \be_j^{\top}\bV \big\|_1$.
		
		Note that the invariant distribution $\bpi$ satisfies $\bpi^{\top}\bP = \bpi^{\top}$. For $ j = 1,2,\ldots,p$,
		\beq\label{temp_1}
		\pi_j = \bigl(\bpi^{\top} \bP\bigr) \be_j = \bpi^{\top} \bU \bV^{\top} \be_j = \sum_{k=1}^r (\bU^{\top}\bpi)_k(\bV^{\top}\be_j)_k \geq \big\| \be_j^{\top} \V \big\|_1 \cdot \min_{1 \leq k \leq r} \big( \bU^{\top}\bpi \big)_k.
		\eeq
		Under assumption \eqref{Assumption2},
		\[ \sum_{i=1}^p \pi_i^2 U_{ik}^2 = \be_k^{\top}\bigl(\bU^{\top}[\diag(\bpi)]^2\bU\bigr)\be_k \geq c_2p^{-1}r^{-1}, \qquad \text{for $k=1,2,\ldots,r$.} \]
		It follows that 
		\beq \label{xi_LowerBound} \bigl(\bU^{\top}\bpi\bigr)_k = \sum_{i=1}^p \pi_i U_{ik} \geq \pi^{-1}_{\max}\sum_{i=1}^p \pi_i^2 U_{ik}^2 \geq C_1^{-1}c_2r^{-1}. \eeq 
		Plugging \eqref{xi_LowerBound} into \eqref{temp_1} yields
		\beq \label{rowV} \bigl\|\be_j^{\top}\bV\bigr\|_1 \leq C_1 c_2^{-1} \pi_j r. \eeq
		We can further derive from \eqref{H_Row-wiseNorm} an upper bound for $\bigl\|\be_j^{\top}\bH\bigr\|_2$. 
		
		
		As for the left singular matrix ${\bf G}$, we can estimate $\bigl\|\be_j^{\top}\bG\big\|_2$ in a similar way.
		Analogous to the definition of ${\bf L}$, there exists a nonsingular matrix ${\bf R} \in \mathbb{R}^{r \times r}$ such that $\bG = [\diag(\bpi)]\bU{\bf R}$ and $\|\bR\|_2 = \lambda_{\min}^{-1/2}\bigl( \bU^{\top}[\diag(\bpi)]^2\bU \bigr) \leq c_2^{-\frac{1}{2}}\sqrt{pr}$. It follows that
		\[ \begin{aligned} & \bigl\| \be_j^{\top}\bG \bigr\|_2 = \bigl\| \be_j^{\top}[\diag(\bpi)]\bU {\bf R} \bigr\|_2 = \pi_j \bigl\| \be_j^{\top} \bU \bR \bigr\|_2 \\ \leq & \pi_j \bigl\| \be_j^{\top} \bU \bigr\|_2 \| \bR \|_2 \leq \pi_j \bigl\|\be_j^{\top}\bU\bigr\|_1 \|\bR\|_2 = c_2^{-\frac{1}{2}}\pi_j\sqrt{pr}, \end{aligned} \]
		where we used $\bigl\| \be_j^{\top} \bU \bigr\|_1 = 1$.
		
	\end{proof}
	
	\begin{proof}[Proof of Lemma \ref{B}]
		We first show that $\bl_1$ is the leading eigen vector of matrix \[ \boldsymbol{\Theta} = \bigl(\bU^{\top}[\diag(\bpi)]^2\bU\bigr)\bigl(\bV^{\top}[\diag(\bpi)]^{-1}\bV\bigr).\]
		Note that by definition,
		$ \bQ = [\diag(\bpi)]\bU\bV^{\top}[\diag(\bpi)]^{-\frac{1}{2}} $, 
		thus $\bQ^{\top}\bQ$ and $\boldsymbol{\Theta}$ share the same eigen values.
		Recall that $\bh_1$ is the leading right singular vector of $\bQ$,
		\beq \label{final1} \begin{aligned} \sigma_1^2\bh_1 = \bQ^{\top}\bQ \bh_1 =  [\diag(\bpi)]^{-\frac{1}{2}}\bV\bU^{\top}[\diag(\bpi)]^2\bU\bV^{\top}[\diag(\bpi)]^{-\frac{1}{2}}\bh_1. \end{aligned} \eeq
		Plugging $\bh_1 = [\diag(\bpi)]^{-\frac{1}{2}}\bV\bl_1$ into \eqref{final1} and multiplying $\bV^{\top}[\diag(\bpi)]^{-\frac{1}{2}}$ on the left, we have
		\[ \begin{aligned} \sigma_1^2\bV^{\top}[\diag(\bpi)]^{-1}\bV\bl_1 = \bV^{\top}[\diag(\bpi)]^{-1}\bV\bigl(\bU^{\top}[\diag(\bpi)]^2\bU\bigr) \bV^{\top}[\diag(\bpi)]^{-1}\bV\bl_1. \end{aligned} \]
		It can be reduced to
		\[ \boldsymbol{\Theta} \bl_1 = \bigl(\bU^{\top}[\diag(\bpi)]^2\bU\bigr)\bigl(\bV^{\top}[\diag(\bpi)]^{-1}\bV\bigr) \bl_1 = \sigma_1^2 \bl_1. \]
		
		The entries of $\boldsymbol{\Theta}$ are lower bounded by
		\[ \pi_{\min}^2\pi_{\max}^{-1} \min_{k,l}\bigl[(\bU^{\top}\bU)(\bV^{\top}\bV)\bigr]_{kl} = \pi_{\min}^2\pi_{\max}^{-1} \min_{k,l}\bigl(\bU^{\top}\bP\bV\bigr)_{kl},\] and upper bounded by
		\[ \pi_{\max}^2\pi_{\min}^{-1} \max_{k,l}\bigl[(\bU^{\top}\bU)(\bV^{\top}\bV)\bigr]_{kl} = \pi_{\max}^2\pi_{\min}^{-1} \max_{k,l}\bigl(\bU^{\top}\bP\bV\bigr)_{kl}.\]
		Condition \eqref{Assumption4} ensures that $\bU^{\top}\bP\bV$ is a positive matrix, hence $\boldsymbol{\Theta}$ is also positive. According to Perron-Frobenius Theorem, all components of $\bl_1$ are non-zero and have the same sign. Without loss of generality, we assume that the entries of $\bl_1$ are all positive. Accoring to Theorem 3.1 in \cite{minc1988nonnegative},
		\beq \label{ratio} \begin{aligned} & \frac{\max_{1 \leq k \leq r} \bl_1(k)}{\min_{1 \leq k \leq r} \bl_1(k)} \leq \max_{s,t,k} \biggl\{ \frac{\boldsymbol{\Theta}_{sk}}{\boldsymbol{\Theta}_{tk}} \biggr\} \leq \frac{\max_{k,l}\boldsymbol{\Theta}_{kl}}{\min_{k,l}\boldsymbol{\Theta}_{kl}} \leq \frac{\pi_{\max}^2\pi_{\min}^{-1} \max_{k,l}\bigl(\bU^{\top}\bP\bV\bigr)_{kl}}{\pi_{\min}^2\pi_{\max}^{-1} \min_{k,l}\bigl(\bU^{\top}\bP\bV\bigr)_{kl}} \leq C_1^3c_1^{-3}C_4, \end{aligned} \eeq
		where we used assumptions \eqref{Assumption1} and \eqref{Assumption4}.
		Recall that in \eqref{normL}, $ \|\bl_1\|_2 \leq \|\bL\|_2 \leq c_2^{-\frac{1}{2}}r^{-\frac{1}{2}}$,
		therefore,
		$\min_{1 \leq k \leq r} \bl_1(k) \leq r^{-\frac{1}{2}} \|\bl_1\|_2 \leq c_2^{-\frac{1}{2}}r^{-1}$. \eqref{ratio} then implies
		\[ \max_{1 \leq k \leq r} \bl_1(k) \leq C_1^3c_1^{-3}C_4 \min_{1 \leq k \leq r} \bl_1(k) \leq \widetilde{C}r^{-1} \]
		for some constant $\widetilde{C} > 0$.
		
		Consider $\bigl\|\bL^{-1}\bigr\|_2$. Since $[\diag(\bpi)]^{-\frac{1}{2}}\bV = \bH\bL^{-1}$ and $\bH$ is orthonormal,
		\[ \begin{aligned} & \bigl\|\bL^{-1}\bigr\|_2 = \max_{\bx \in \mathbb{S}^{r-1}} \bigl\|\bL^{-1}\bx\bigr\|_2 = \max_{\bx \in \mathbb{S}^{r-1}} \bigl\|\bH\bL^{-1}\bx\bigr\|_2 \\ = & \max_{\bx \in \mathbb{S}^{r-1}} \bigl\|[\diag(\bpi)]^{-\frac{1}{2}}\bV\bx\bigr\|_2 = \bigl\|[\diag(\bpi)]^{-\frac{1}{2}}\bV\bigr\|_2. \end{aligned} \]
		By \eqref{rowV}, $0 \leq [\diag(\bpi)]^{-1}\bV{\bf 1}_r \leq C_1c_2^{-1}r$, thus
		\[ \begin{aligned} & \bigl\|\bV^{\top}[\diag(\bpi)]^{-1}\bV\bigr\|_1 = \bigl\|\bV^{\top}[\diag(\bpi)]^{-1}\bV{\bf 1}_r\bigr\|_{\infty} \\ \leq & C_1c_2^{-1}r \bigl\|\bV^{\top}{\bf 1}_p\bigr\|_{\infty} = C_1c_2^{-1} r \|{\bf 1}_r\|_{\infty} = C_1c_2^{-1}r. \end{aligned} \]
		We have
		\[ \begin{aligned}  \bigl\| \bL^{-1} \bigr\|_2 = \bigl\|[\diag(\bpi)]^{-\frac{1}{2}}\bV\bigr\|_2 \leq  \bigl\|\bV^{\top}[\diag(\bpi)]^{-1}\bV\bigr\|_1^{\frac{1}{2}} \leq  C_1^{\frac{1}{2}}c_2^{-\frac{1}{2}}\sqrt{r}. \end{aligned} \]
		
		Therefore, $\|\bl_1\|_2 \geq \bigl\| \bL^{-1} \bigr\|_2^{-1} \geq C_1^{-\frac{1}{2}}c_2^{\frac{1}{2}}r^{-\frac{1}{2}}$ and
		\[ \max_{1 \leq k \leq r} \bl_1(k) \geq r^{-\frac{1}{2}}\|\bl_1\|_2 \geq C_1^{-\frac{1}{2}}c_2^{\frac{1}{2}}r^{-1}. \]
		We can conclude from \eqref{ratio} that
		\[ \min_{1 \leq k \leq r} \bl_1(k) \geq C_1^{-3}c_1^3C_4^{-1} \max_{1 \leq k \leq r} \bl_1(k) \geq \widetilde{c}r^{-1} \]
		for some $\widetilde{c}>0$.
	\end{proof}

	\begin{proof}[Proof of Lemma \ref{Pre}]
		Recall that by definition
		\[ \bh_1 = [\diag(\bpi)]^{-\frac{1}{2}}\bV\bl_1, \]
		and Lemma \ref{B} provides an estimate of the entries in $\bl_1$. Therefore,
		\[ cr^{-1} \cdot \pi_j^{-\frac{1}{2}} \bigl\|\be_j^{\top}\bV\bigr\|_1 \leq \bh_1(j) \leq Cr^{-1} \cdot \pi_j^{-\frac{1}{2}} \bigl\|\be_j^{\top}\bV\bigr\|_1, \]
		where $c,C > 0$ are the constants in Lemma \ref{B}.
		
		Note that in \eqref{rowV}, there is an upper bound for $\bigl\|\be_j^{\top}\bV\bigr\|_1$. Hence,
		\beq \label{hUpper} \bh_1(j) \leq C \pi_j^{-\frac{1}{2}}r^{-1} \cdot C_1 c_2^{-1} \pi_j r = CC_1c_2^{-1}\sqrt{\pi_j}. \eeq
		
		We now derive a lower bound for $\bigl\|\be_j^{\top}\bV\bigr\|_1$ using assumption \eqref{Assumption5}. Because the stationary distribution $\bpi$ satisfies $\bpi^{\top} \bP = \bpi^{\top}$, for each $j=1,2,\ldots,p$ we have
		\[ \begin{aligned} & \pi_j = \bpi^{\top}\bP\be_j = \bpi^{\top}\bU\bV^{\top} \be_j = \sum_{k=1}^r \bigl(\bU^{\top}\bpi\bigr)_k\bigl(\be_j^{\top}\bV\bigr)_k \\ \leq & \bar{C}_1r^{-1} \sum_{k=1}^r \bigl(\be_j^{\top}\bV\bigr)_k = \bar{C}_1r^{-1} \bigl\|\be_j^{\top}\bV\bigr\|_1. \end{aligned} \]
		It follows that
		\beq \label{hLower} \begin{aligned} \bh_1(j) \geq c\pi_j^{-\frac{1}{2}}r^{-1} \bigl\|\be_j^{\top}\bV\bigr\|_1 \geq c\pi_j^{-\frac{1}{2}}r^{-1} \cdot \bar{C}_1^{-1}\pi_jr = \bar{C}_1^{-1}c\sqrt{\pi_j}. \end{aligned} \eeq
		
		As for $\bg_1$, by the definition of singular value decomposition,
		\[ \bg_1 = \sigma_1^{-1} \bQ\bh_1, \]
		thus $\bg_1$ is nonnegative.
		For any $j=1,2,\ldots,p$,
		\[ \bg_1(j) = \sigma_1^{-1} \be_j^{\top} \bQ \bh_1 = \sigma_1^{-1} \be_j^{\top} \bF [\diag(\bpi)]^{-\frac{1}{2}} \bh_1. \]
		By \eqref{hUpper}, $[\diag(\bpi)]^{-\frac{1}{2}}\bh_1 \leq CC_1c_2^{-1}{\bf 1}_p$, thus
		\beq \label{g1} \bg_1(j) \leq CC_1c_2^{-1}\sigma_1^{-1} \be_j^{\top} \bF {\bf 1}_p = CC_1c_2^{-1} \sigma_1^{-1} \pi_j. \eeq
		Here, we used $\bF{\bf 1}_p = \bpi$. Since $\bigl\|[\diag(\bpi)]^{\frac{1}{2}}{\bf 1}\bigr\|_2 = \sqrt{\sum_{i=1}^p \pi_i} = 1$,
		\beq \label{Sigma1_LowerBound} \begin{aligned} \sigma_1 =  \max_{\bx \in \mathbb{R}^p} \bigl\|\bQ\bx\bigr\|_2 \geq \bigl\|\bQ[\diag(\bpi)]^{\frac{1}{2}}{\bf 1}_p\bigr\|_2, = \bigl\|\bF{\bf 1}_p\bigr\|_2 = \|\bpi\|_2 \geq p^{-\frac{1}{2}}\|\bpi\|_1 = p^{-\frac{1}{2}}. \end{aligned} \eeq
		Plugging \eqref{Sigma1_LowerBound} into \eqref{g1}, we obtain an upper bound for $\bg_1(j)$.
	\end{proof}
	
	\section{Numerical Experiments}
	
	
	\subsection{Explanation of Simulation Settings}
	
	We test our new approach on simulated sample transitions. For a $p$-state Markov chain with $r$ meta-states, we first randomly create two matrices ${\bf U}, {\bf V} \in \mathbb{R}_+^{p \times r}$ such that each meta-state has the same number of anchor states. 
	After assembling a transition matrix ${\bf P} = {\bf U}{\bf V}^{\top}$, we generate random walk data $\{X_0, X_1,\ldots,X_n\}$. For each data point in the figures, we conduct $5$ independent experiements and plot their mean and standard deviation.
	
	In Figure~3 (a), we run experiments with $p=1000$, $r = 6$ and the number of anchor states equal to $25,50,75,100$ for each meta-state. When $p$ is fixed and $n$ varies, the log-total-variation error in $\widehat{\bf V}$ scales linearly with $\log(n)$, with a fitted slope $\approx -0.5$. This is consistent with conclusion of Theorem 3 in paper which indicates that the error bound decreases with $n$ at the speed of $n^{-1/2}$. 
	In Figure~3 (b), we carry out experiments with $n/p=1000$, $r = 6$ and the number of anchor states equal to $\lfloor0.025p\rfloor,\lfloor0.050p\rfloor,\lfloor0.075p\rfloor,\lfloor0.100p\rfloor$ for each meta-state. When both $(n,p)$ vary while $n/p$ is fixed, the the log-total-variation error in $\widehat{\bf V}$ remains almost constant, with a fitted slope $\approx 0$. This validates the scaling of $\sqrt{p/n}$ in the error bound of $\widehat{\bf V}$. In both figures, we 
	observe that having multiple anchor states per each metastate makes the estimation error slightly smaller.
	
	In Figure~3 (c), we consider estimating the transition matrix $\bP$ by $\widehat{\bU}\widehat{\bV}^{\top}$, and compare it with the the spectral estimator in \cite{zhang2018spectral}. Our method has a slightly better performance. Note that our method not only estimates $\bP$ but also estimates $(\bU,\bV)$, while the spectral method cannot estimate $(\bU,\bV)$. 
	
	\subsection{More Results on Manhattan Taxi-trip Data}
	
	\vspace{-0.2cm}
	
	\paragraph{Distributions of Pick-up and Drop-off Locations.}~
	
	\vspace{-0.2cm}
	\begin{figure}[H]
		\centering
		\begin{minipage}{0.32\linewidth}
			\centering
			\includegraphics[width=\linewidth]{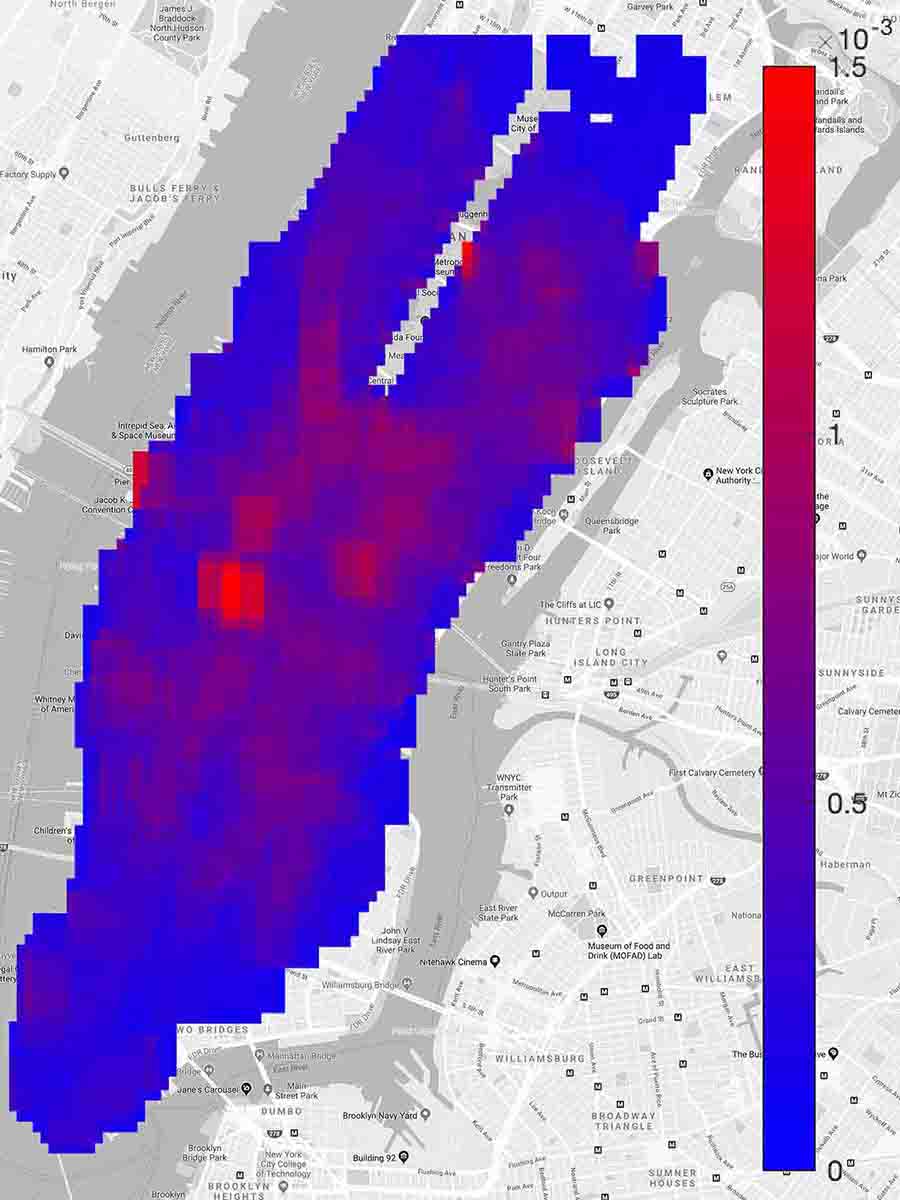}
		\end{minipage}
		\hspace{1cm}
		\begin{minipage}{0.32\linewidth}
			\centering
			\includegraphics[width=\linewidth]{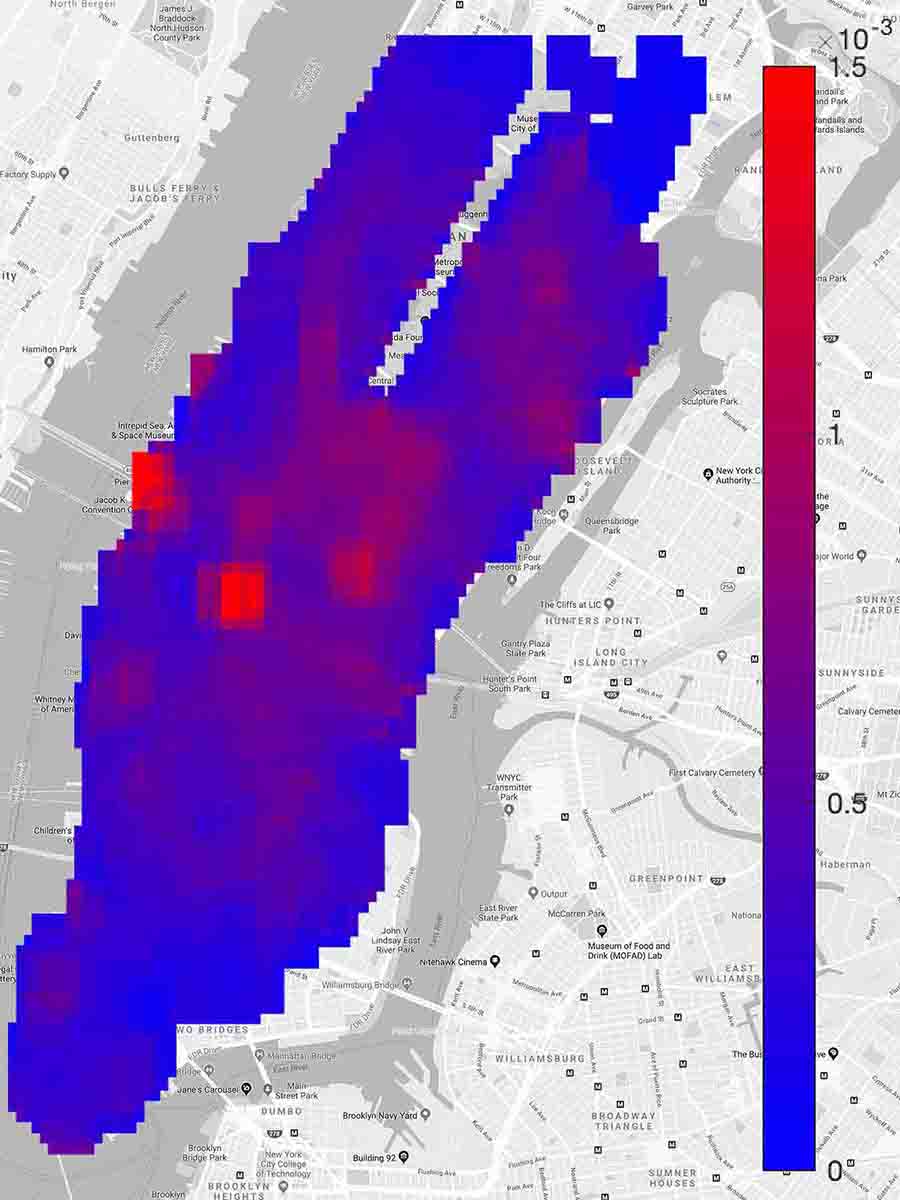}
		\end{minipage} \\ \vspace{0.3cm}
		{Figure F.1: Distributions of pick-up (L) and drop-off (R) location, illustrated as heatmaps.}
		\vspace{-0.5cm}
	\end{figure}
	
	\paragraph{Columns of Singular Vectors.}~ \vspace{-0.2cm}
	
	We conduct SVD to matrix $\widetilde{\bf N} = {\bf N} \big[ {\rm diag}({\bf N}^{\top}{\bf 1}_p) \big]^{-\frac{1}{2}}$. Denote the right singular vectors as $\widehat{\bf h}_1, \widehat{\bf h}_2, \ldots, \widehat{\bf h}_r$ and singular values $\widehat{\sigma}_1, \widehat{\sigma}_2, \ldots, \widehat{\sigma}_r$. In the following, we illustrate $\widehat{\bf h}_1, \widehat{\bf h}_2, \ldots, \widehat{\bf h}_r$ with heat maps. It turns out that the figures do not have clear patterns.
	
	\vspace{-0.4cm}
	\begin{figure}[H]
		\centering
		\hspace{-0.6cm}
		\begin{minipage}{0.32\linewidth}
			\centering
			\includegraphics[width=\linewidth]{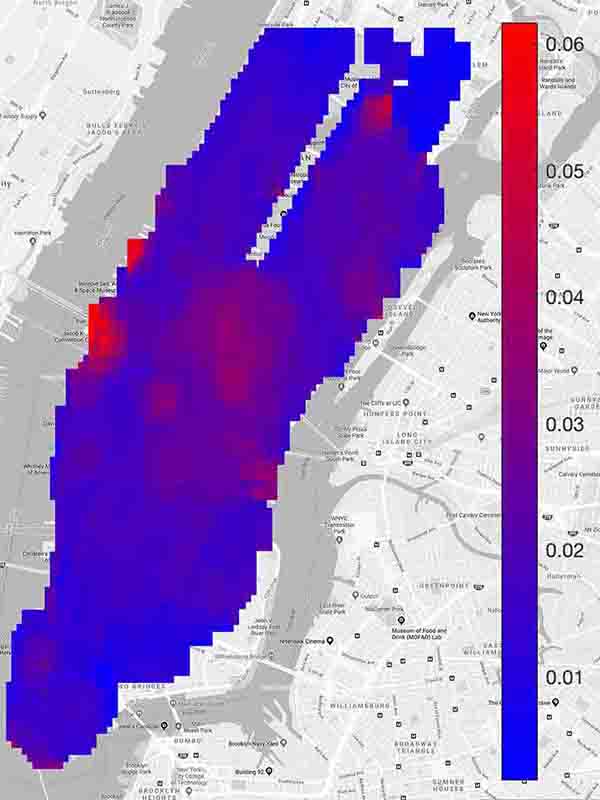}
			{$\widehat{\sigma}_1 = 2.3787$}
		\end{minipage}
		\hspace{0.3cm}
		\begin{minipage}{0.32\linewidth}
			\centering
			\includegraphics[width=\linewidth]{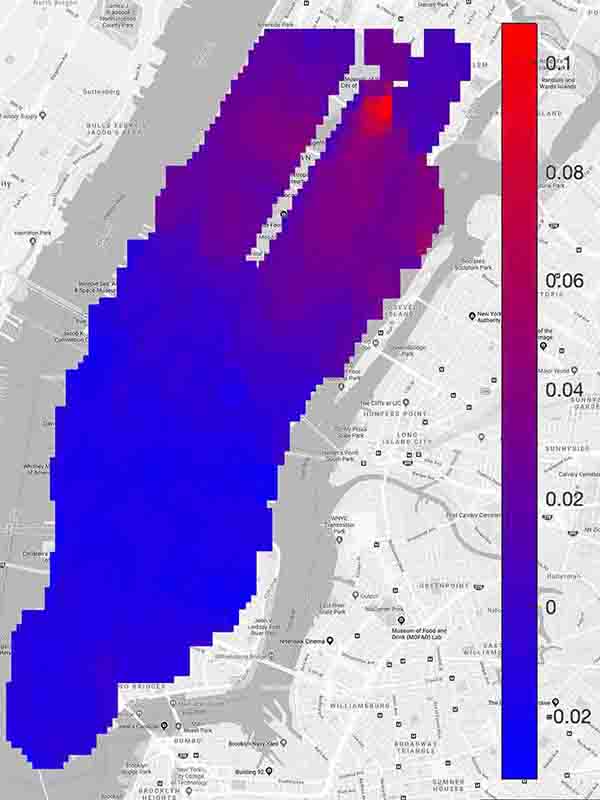}
			{$\widehat{\sigma}_2 = 1.1028$}
		\end{minipage}
		\hspace{0.3cm}
		\begin{minipage}{0.32\linewidth}
			\centering
			\includegraphics[width=\linewidth]{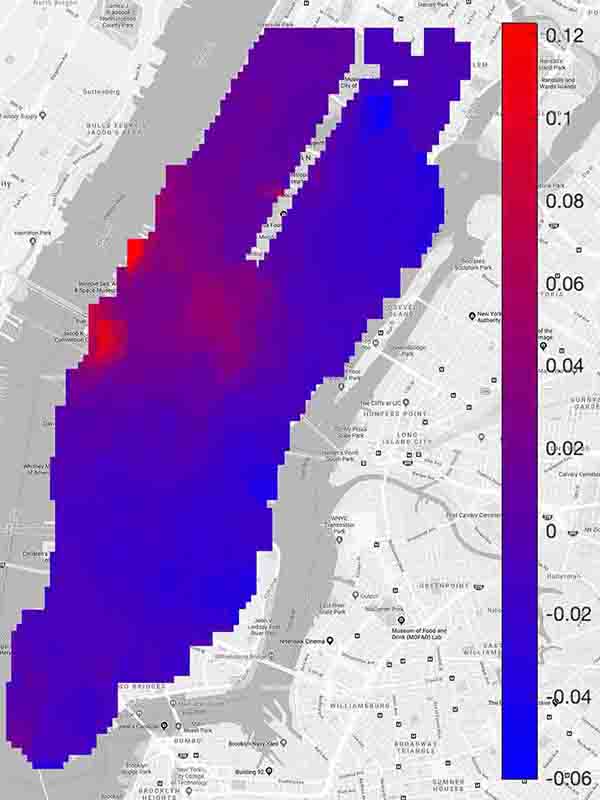}
			{$\widehat{\sigma}_3 = 0.9148$}
		\end{minipage} 
		\vspace{-1cm}
	\end{figure}
	
	\begin{figure}[H]
		\centering
		\hspace{-0.6cm}
		\begin{minipage}{0.32\linewidth}
			\centering
			\includegraphics[width=\linewidth]{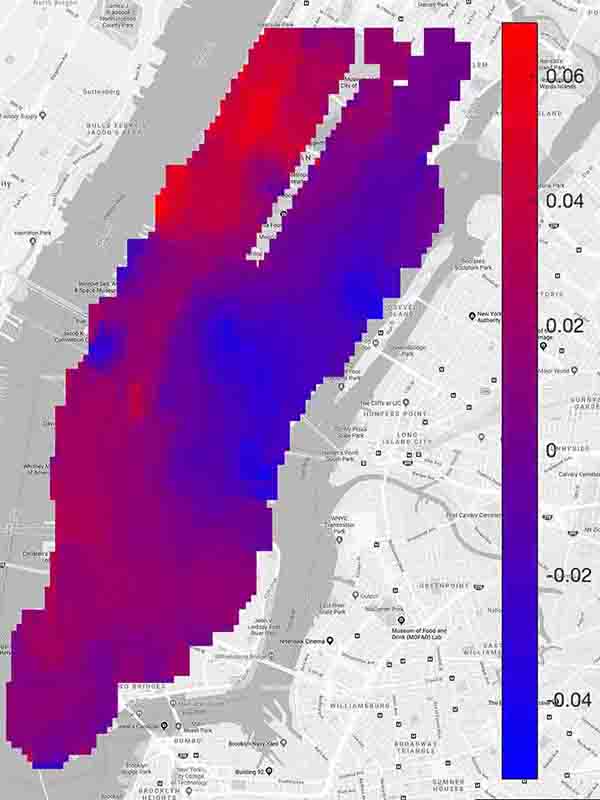}
			{$\widehat{\sigma}_4 = 0.8154$}
		\end{minipage}
		\hspace{0.3cm}
		\begin{minipage}{0.32\linewidth}
			\centering
			\includegraphics[width=\linewidth]{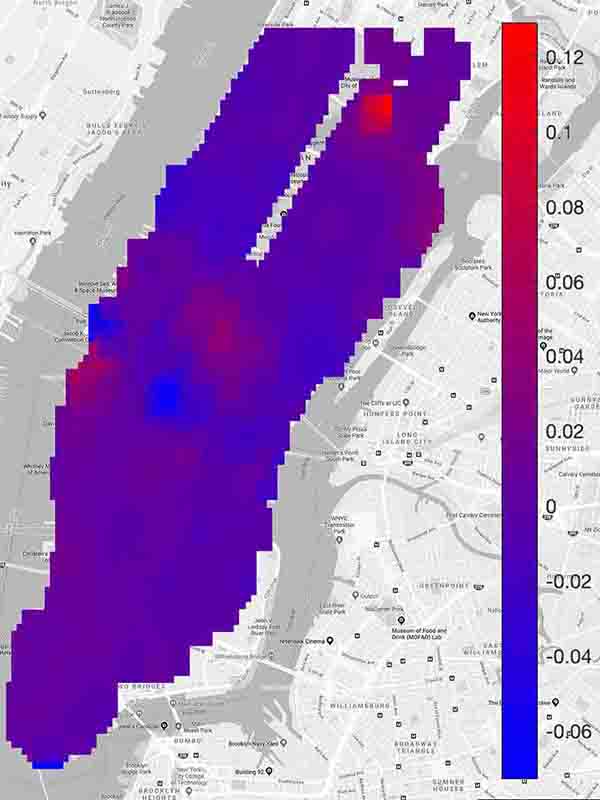}
			{$\widehat{\sigma}_5 = 0.6880$}
		\end{minipage}
		\hspace{0.3cm}
		\begin{minipage}{0.32\linewidth}
			\centering
			\includegraphics[width=\linewidth]{V_svd6.jpg}
			{$\widehat{\sigma}_6 = 0.5076$}
		\end{minipage} 
		\vspace{-0.5cm}
	\end{figure}
	\begin{figure}[H]
		\centering
		\hspace{-0.6cm}
		\begin{minipage}{0.32\linewidth}
			\centering
			\includegraphics[width=\linewidth]{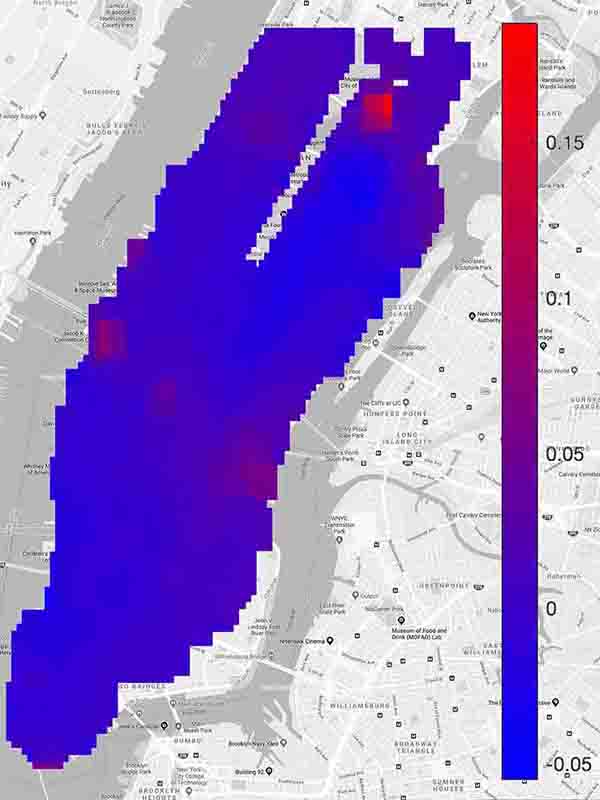}
			{$\widehat{\sigma}_7 = 0.4557$}
		\end{minipage}
		\hspace{0.3cm}
		\begin{minipage}{0.32\linewidth}
			\centering
			\includegraphics[width=\linewidth]{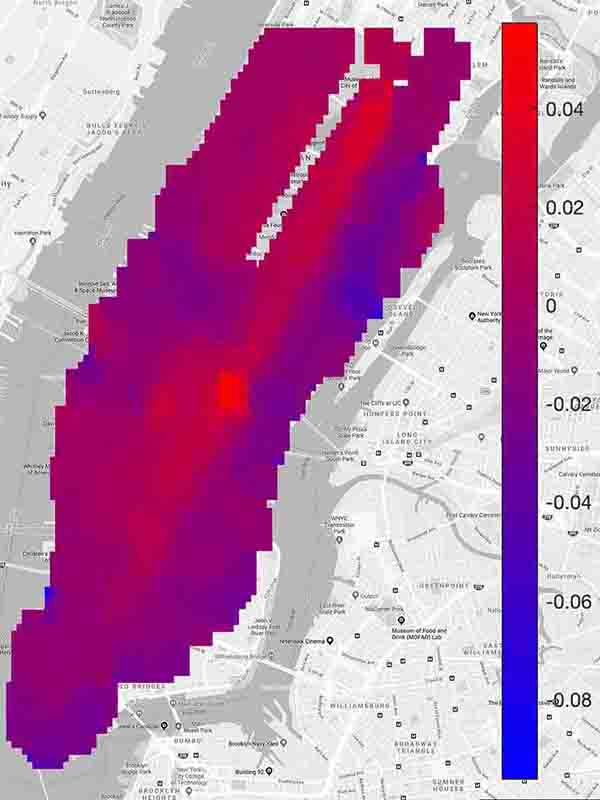}
			{$\widehat{\sigma}_8 = 0.4352.$}
		\end{minipage}
		\hspace{0.3cm}
		\begin{minipage}{0.32\linewidth}
			\centering
			\includegraphics[width=\linewidth]{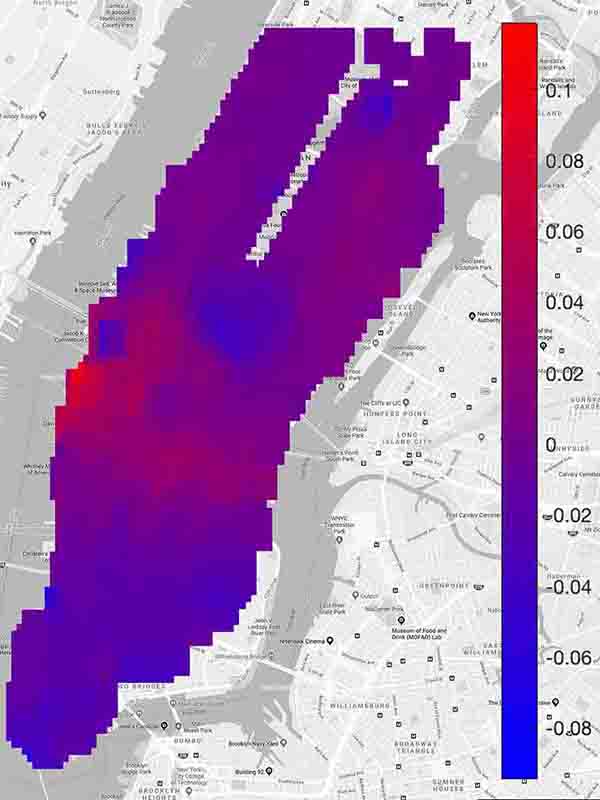}
			{$\widehat{\sigma}_9 = 0.4076$}
		\end{minipage}
		\vspace{-0.8cm}
	\end{figure}
	
	\begin{figure}[H]
		\centering
		\hspace{-0.6cm}
		\begin{minipage}{0.32\linewidth}
			\centering
			\includegraphics[width=\linewidth]{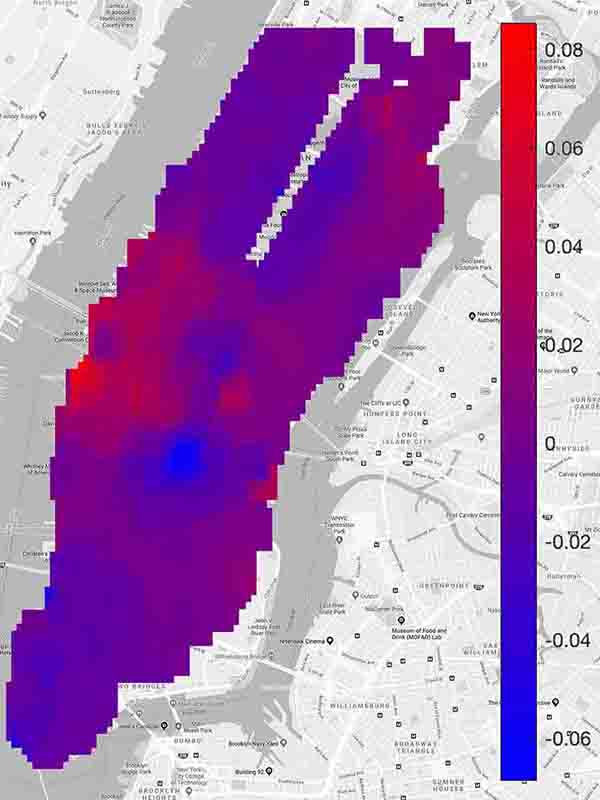}
			{$\widehat{\sigma}_{10} = 0.3818$}
		\end{minipage} \vspace{0.2cm}
		\\ {Figure F.2: Singular vectors of $\widetilde{\bf N}$ illustrated as heat maps.}
		\vspace{-0.4cm}
	\end{figure}
	
	\paragraph{Aggregation and Disaggregation Distributions.}~ \vspace{-0.1cm}
	
	We apply state aggregation learning to ${\bf N}$ with the number of meta-states equal to $r = 10$. The columns of estimated $\widehat{\bf U}$ and $\widehat{\bf V}$ are illustrated as heat maps. 
	We can tell from the figures that aggregation likelihoods and disaggregation distributions have practical meanings in real life.
	For example, the heat map of $\widehat{\bf V}$ has two red points which correspond to New York Penn. Station \& New York Ferry Waterway. It reveals the behavior of a certain group of passengers.
	
	\begin{figure}[H]
		\vspace{-.3cm}
		\centering
		\begin{minipage}{0.32\linewidth}
			\centering
			\includegraphics[width=\linewidth]{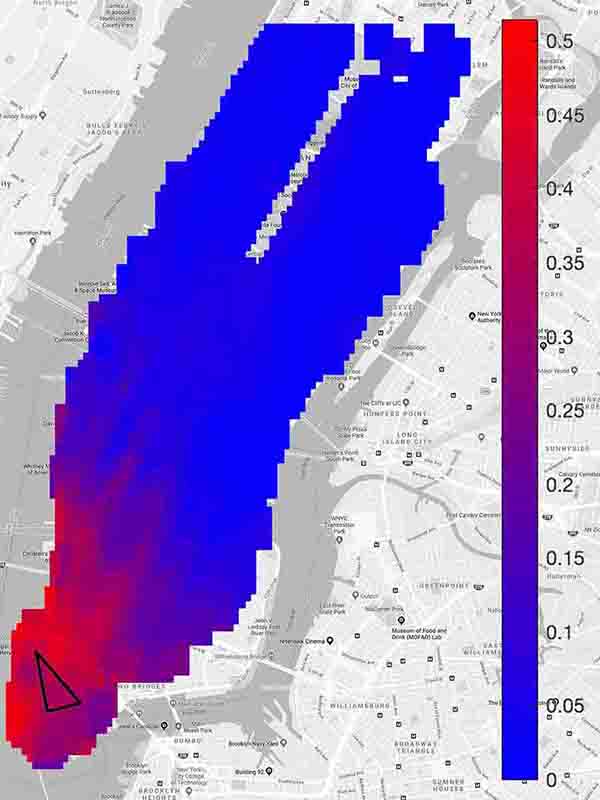}
			{$\widehat{\bf U}_1$}
		\end{minipage}
		\hspace{1cm}
		\begin{minipage}{0.32\linewidth}
			\centering
			\includegraphics[width=\linewidth]{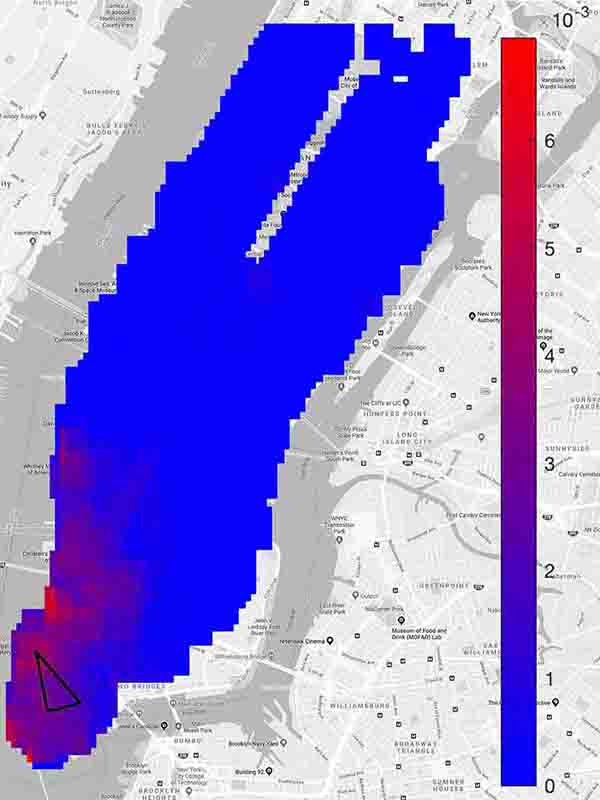}
			{$\widehat{\bf V}_1$}
		\end{minipage} \\ \vspace{0.05cm}
		{$\big( \widehat{\boldsymbol{\pi}}^{\top}\widehat{\bf U} \big)_1 = 0.0770$}\vspace{-0.5cm}
	\end{figure}
	
	\begin{figure}[H]
		\centering
		\begin{minipage}{0.32\linewidth}
			\centering
			\includegraphics[width=\linewidth]{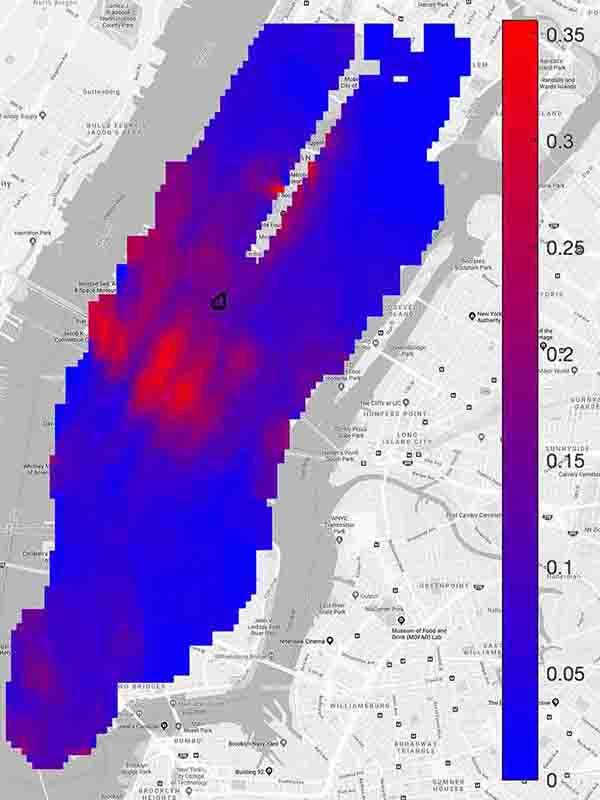}
			{$\widehat{\bf U}_2$}
		\end{minipage}
		\hspace{1cm}
		\begin{minipage}{0.32\linewidth}
			\centering
			\includegraphics[width=\linewidth]{V2.jpg}
			{$\widehat{\bf V}_2$}
		\end{minipage} \\ \vspace{0.05cm}
		{$\big( \widehat{\boldsymbol{\pi}}^{\top}\widehat{\bf U} \big)_2 = 0.1133$}\vspace{-0.5cm}
	\end{figure}
	
	\begin{figure}[H]
		\centering	\begin{minipage}{0.32\linewidth}
			\centering
			\includegraphics[width=\linewidth]{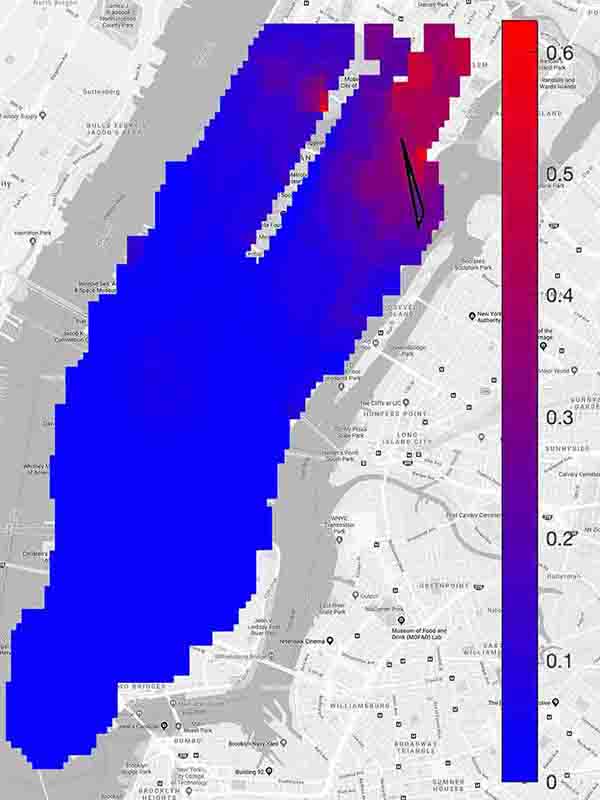}
			{$\widehat{\bf U}_3$}
		\end{minipage}
		\hspace{1cm}
		\begin{minipage}{0.32\linewidth}
			\centering
			\includegraphics[width=\linewidth]{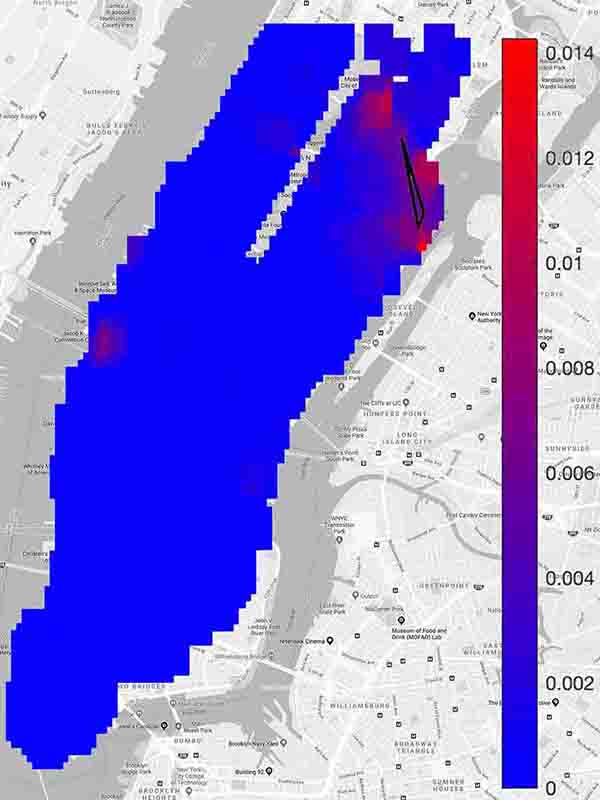}
			{$\widehat{\bf V}_3$}
		\end{minipage} \\ \vspace{0.15cm}
		{$\big( \widehat{\boldsymbol{\pi}}^{\top}\widehat{\bf U} \big)_3 = 0.0503$}\vspace{-0.5cm}
	\end{figure}
	
	\begin{figure}[H]
		\centering
		\begin{minipage}{0.32\linewidth}
			\centering
			\includegraphics[width=\linewidth]{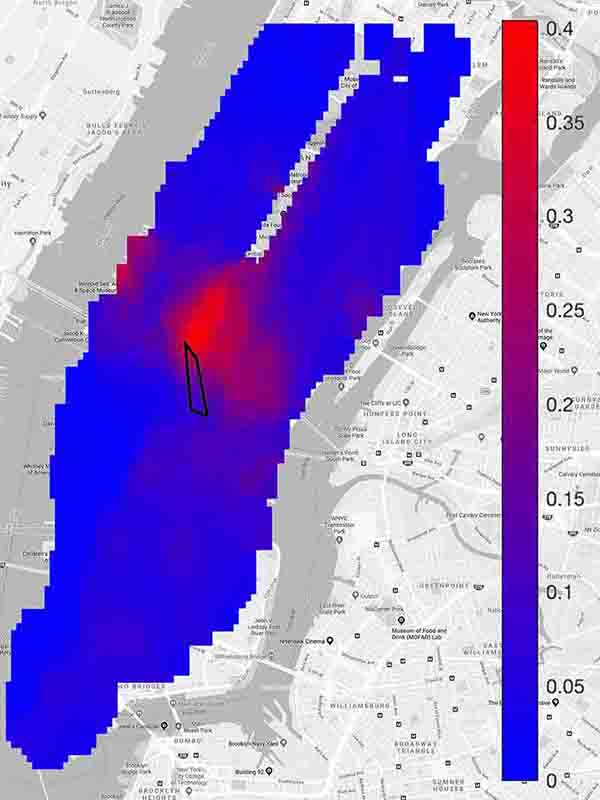}
			{$\widehat{\bf U}_4$}
		\end{minipage}
		\hspace{1cm}
		\begin{minipage}{0.32\linewidth}
			\centering
			\includegraphics[width=\linewidth]{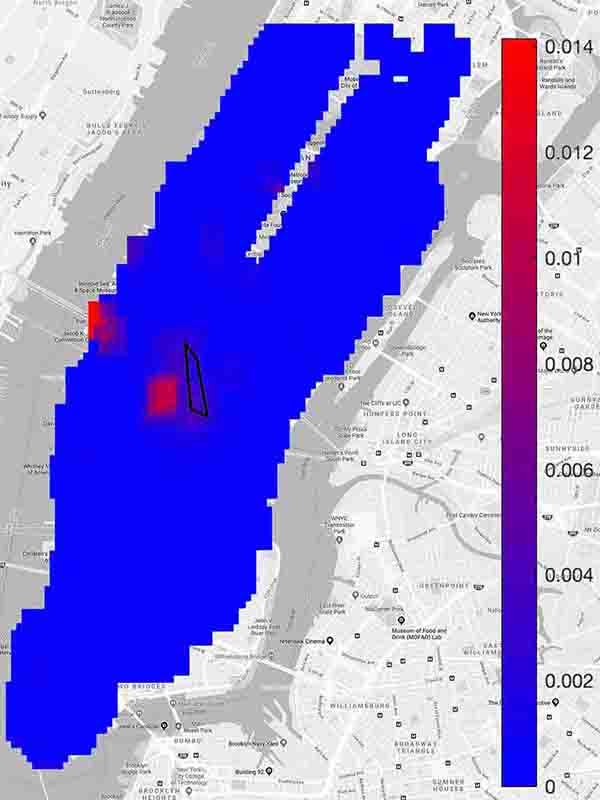}
			{$\widehat{\bf V}_4$}
		\end{minipage} \\ \vspace{0.15cm}
		{$\big( \widehat{\boldsymbol{\pi}}^{\top}\widehat{\bf U} \big)_4 = 0.1022$ \\ ~ \\ Right: The two red points correspond to New York Penn. Station \& New York Ferry Waterway.}\vspace{-0.5cm}
	\end{figure}
	
	\begin{figure}[H]
		\centering
		\begin{minipage}{0.32\linewidth}
			\centering
			\includegraphics[width=\linewidth]{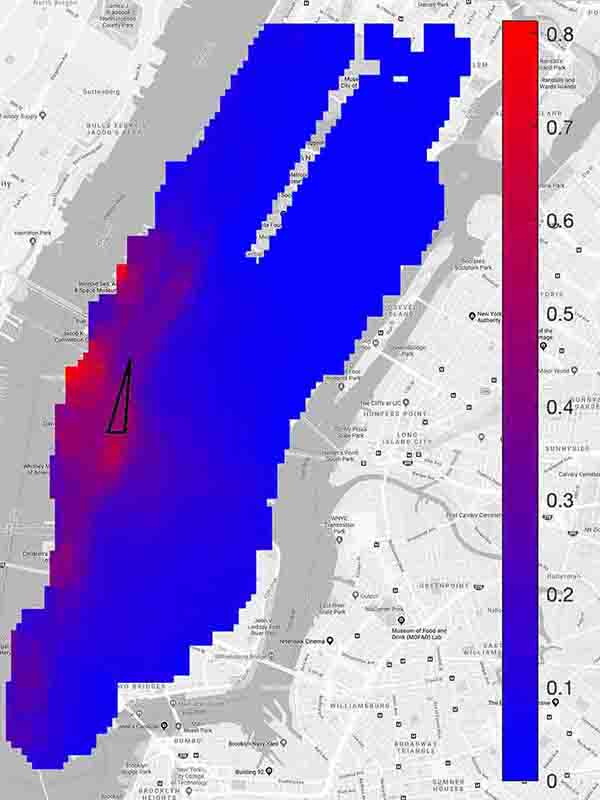}
			{$\widehat{\bf U}_5$}
		\end{minipage}
		\hspace{1cm}
		\begin{minipage}{0.32\linewidth}
			\centering
			\includegraphics[width=\linewidth]{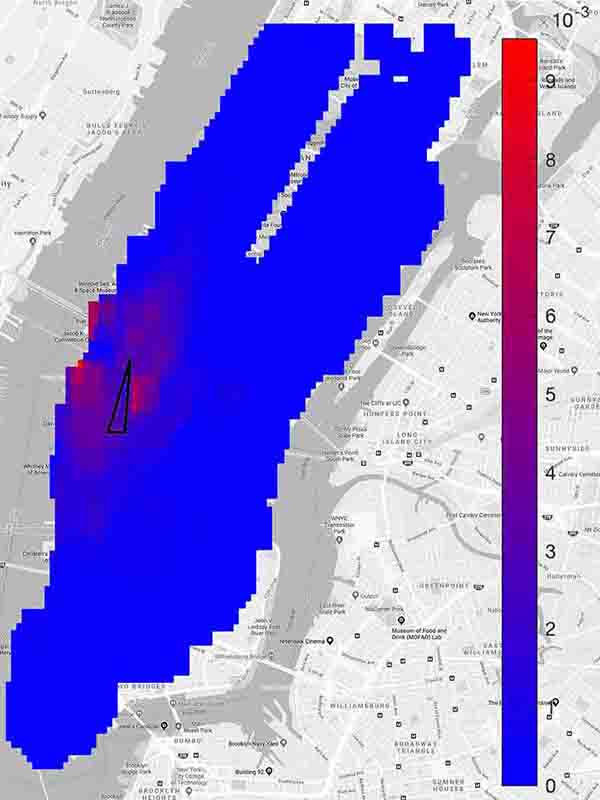}
			{$\widehat{\bf V}_5$}
		\end{minipage} \\ \vspace{0.15cm}
		{$\big( \widehat{\boldsymbol{\pi}}^{\top}\widehat{\bf U} \big)_5 = 0.1135$}\vspace{-0.5cm}
	\end{figure}
	
	\begin{figure}[H]
		\centering
		\begin{minipage}{0.32\linewidth}
			\centering
			\includegraphics[width=\linewidth]{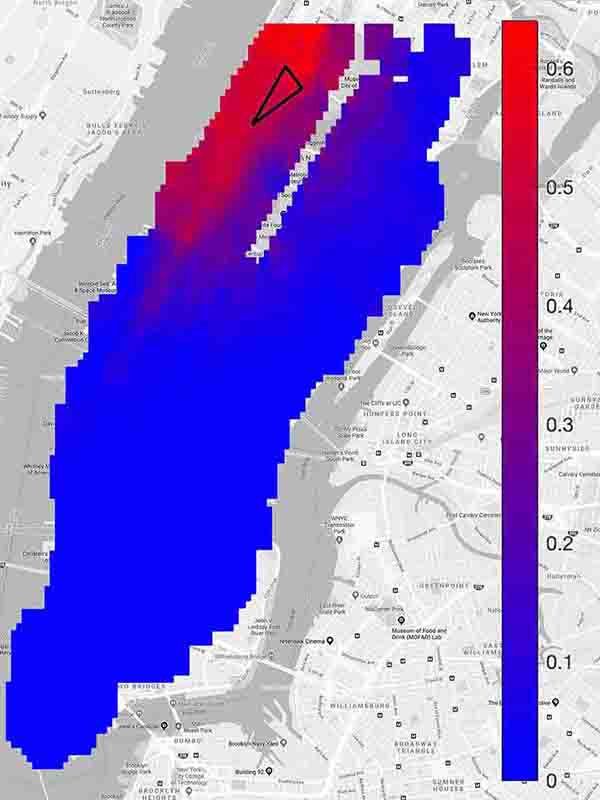}
			{$\widehat{\bf U}_6$}
		\end{minipage}
		\hspace{1cm}
		\begin{minipage}{0.32\linewidth}
			\centering
			\includegraphics[width=\linewidth]{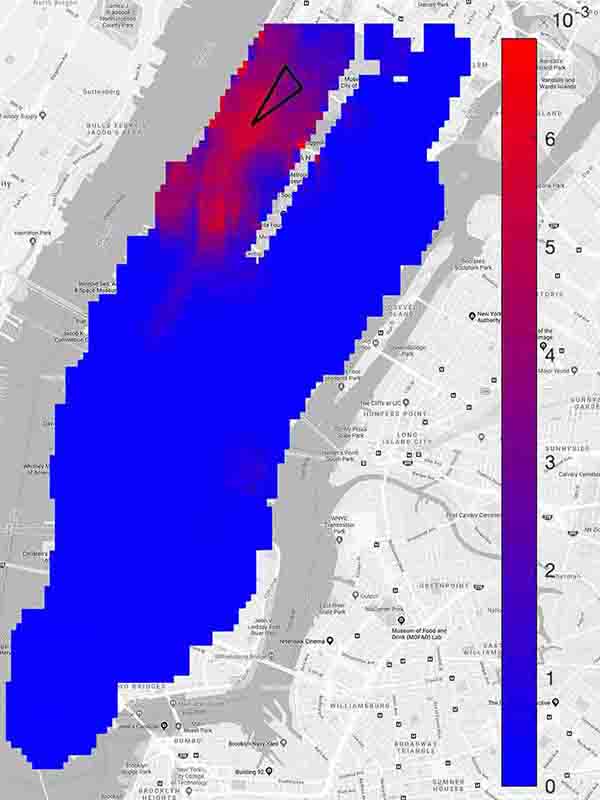}
			{$\widehat{\bf V}_6$}
		\end{minipage} \\ \vspace{0.15cm}
		{$\big( \widehat{\boldsymbol{\pi}}^{\top}\widehat{\bf U} \big)_6 = 0.0757$}\vspace{-0.5cm}
	\end{figure}
	
	\begin{figure}[H]
		\centering
		\begin{minipage}{0.32\linewidth}
			\centering
			\includegraphics[width=\linewidth]{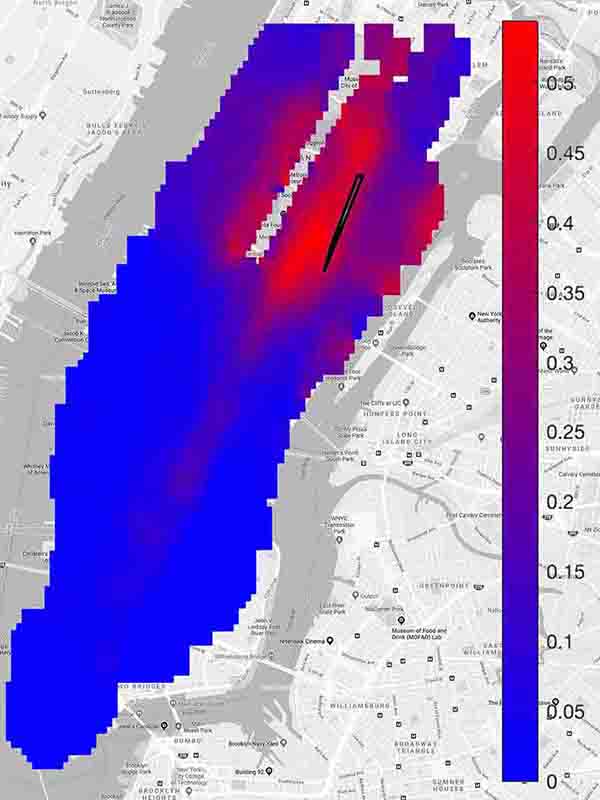}
			{$\widehat{\bf U}_7$}
		\end{minipage}
		\hspace{1cm}
		\begin{minipage}{0.32\linewidth}
			\centering
			\includegraphics[width=\linewidth]{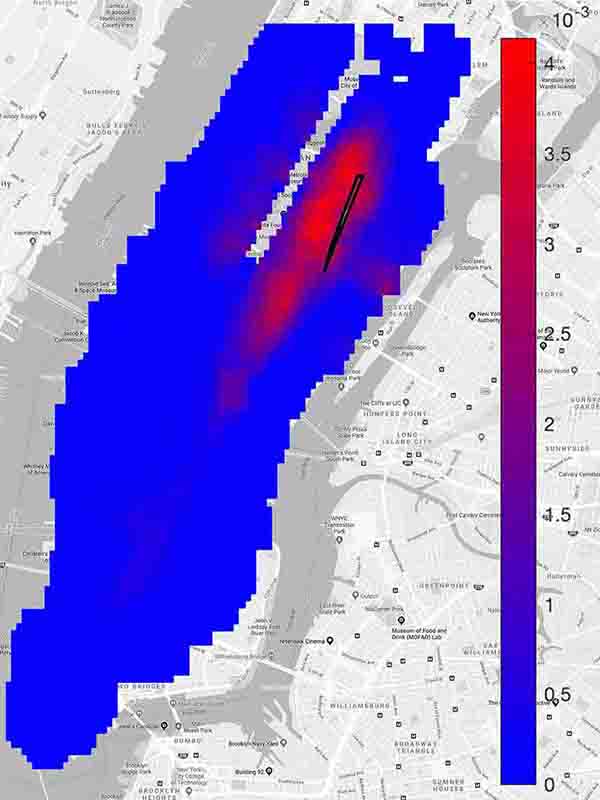}
			{$\widehat{\bf V}_7$}
		\end{minipage} \\ \vspace{0.15cm}
		{$\big( \widehat{\boldsymbol{\pi}}^{\top}\widehat{\bf U} \big)_7 = 0.1288$}\vspace{-0.5cm}
	\end{figure}
	\begin{figure}[H]
		\centering
		\begin{minipage}{0.32\linewidth}
			\centering
			\includegraphics[width=\linewidth]{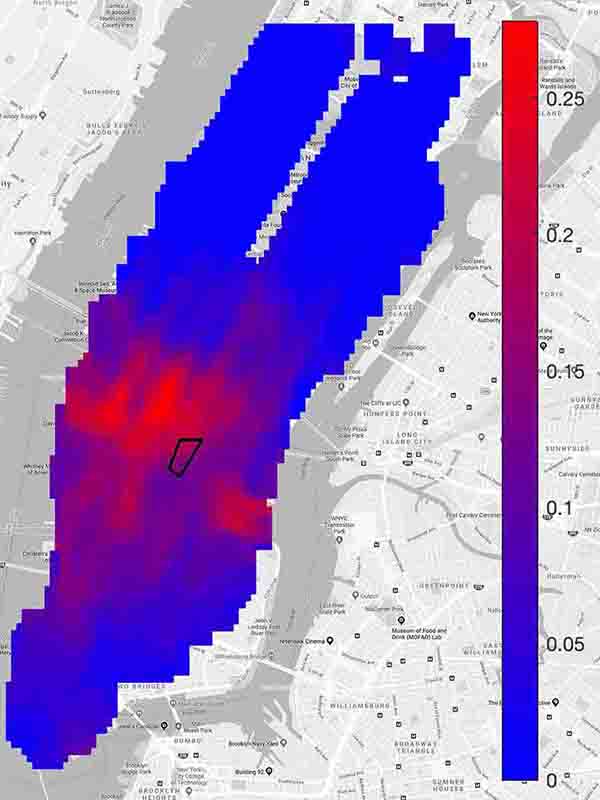}
			{$\widehat{\bf U}_8$}
		\end{minipage}
		\hspace{1cm}
		\begin{minipage}{0.32\linewidth}
			\centering
			\includegraphics[width=\linewidth]{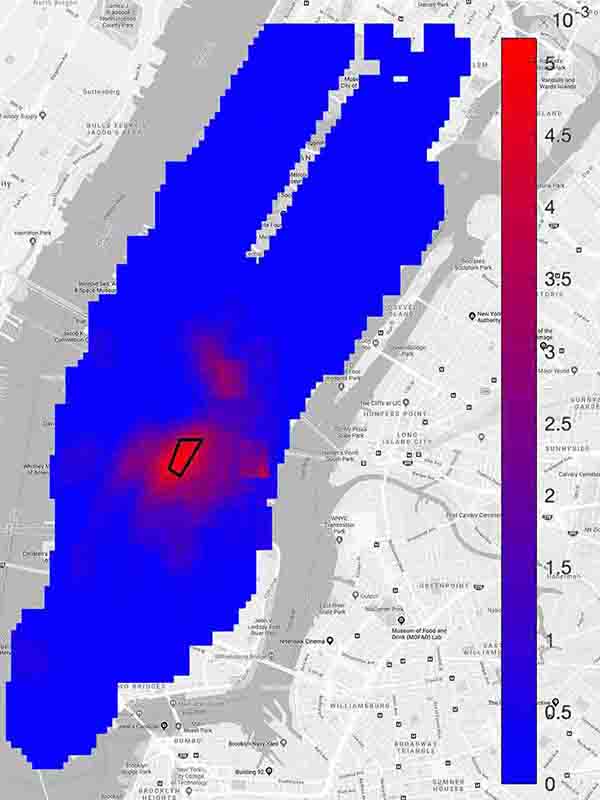}
			{$\widehat{\bf V}_8$}
		\end{minipage} \\ \vspace{0.15cm}
		{$\big( \widehat{\boldsymbol{\pi}}^{\top}\widehat{\bf U} \big)_8 = 0.0899$}\vspace{-0.5cm}
	\end{figure}
	\begin{figure}[H]
		\centering
		\begin{minipage}{0.32\linewidth}
			\centering
			\includegraphics[width=\linewidth]{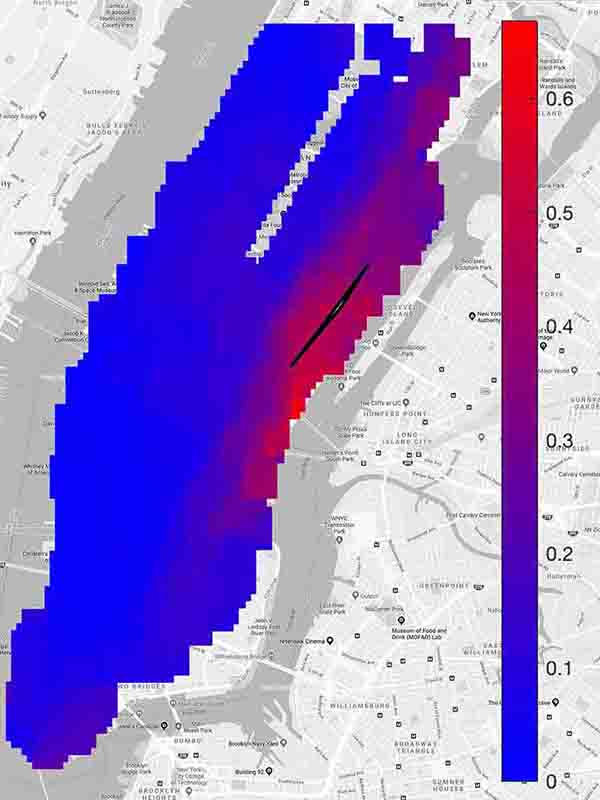}
			{$\widehat{\bf U}_9$}
		\end{minipage}
		\hspace{1cm}
		\begin{minipage}{0.32\linewidth}
			\centering
			\includegraphics[width=\linewidth]{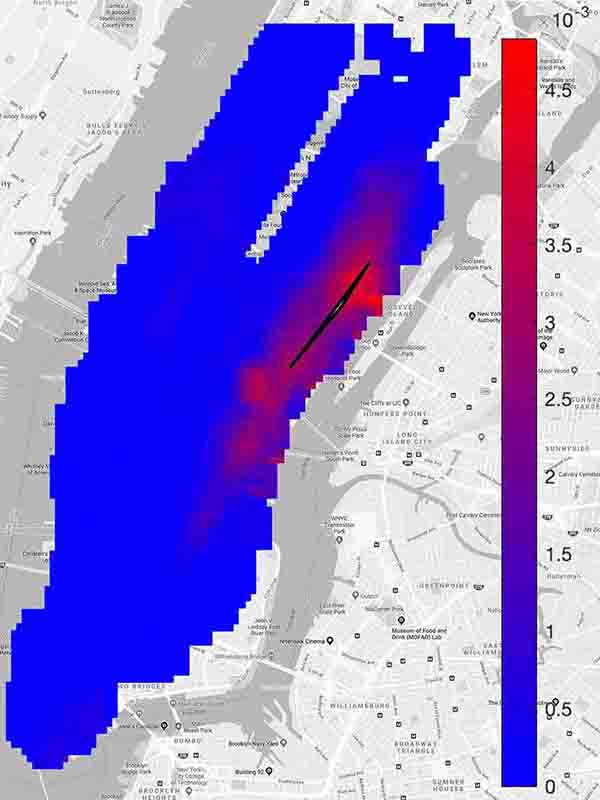}
			{$\widehat{\bf V}_9$}
		\end{minipage} \\ \vspace{0.15cm}
		{$\big( \widehat{\boldsymbol{\pi}}^{\top}\widehat{\bf U} \big)_9 = 0.1385$}\vspace{-0.5cm}
	\end{figure}
	
	\begin{figure}[H]
		\centering
		\begin{minipage}{0.32\linewidth}
			\centering
			\includegraphics[width=\linewidth]{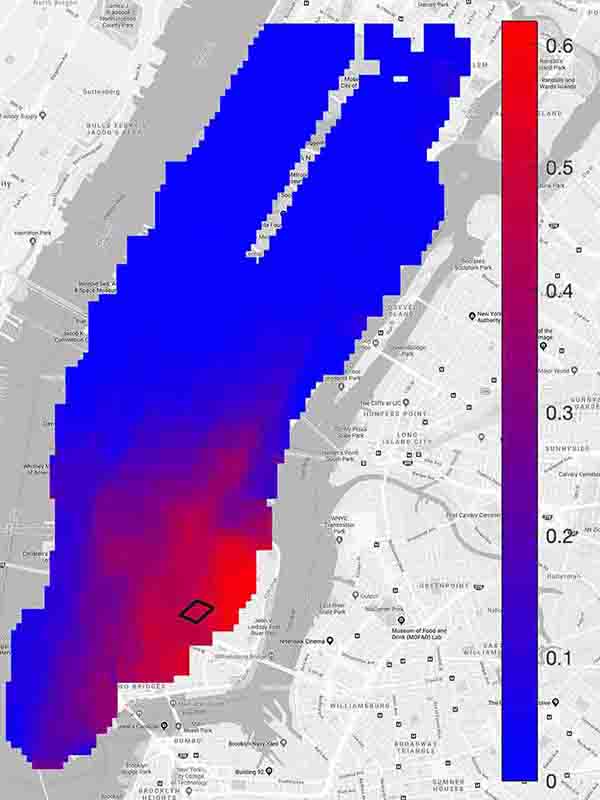}
			{$\widehat{\bf U}_{10}$}
		\end{minipage}
		\hspace{1cm}
		\begin{minipage}{0.32\linewidth}
			\centering
			\includegraphics[width=\linewidth]{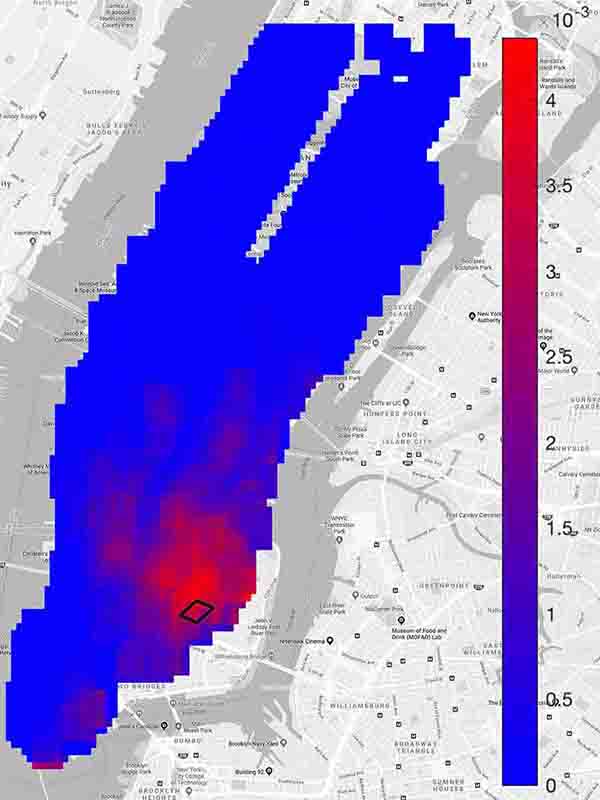}
			{$\widehat{\bf V}_{10}$}
		\end{minipage} \\ \vspace{0.15cm}
		{$\big( \widehat{\boldsymbol{\pi}}^{\top}\widehat{\bf U} \big)_{10} = 0.1106$} \vspace{1cm}
		\\ {Figure F.3: Columns of $\widehat{\bf U}$ and $\widehat{\bf V}$ illustrated as heat maps.}
		\vspace{-0.5cm}
	\end{figure}

	\paragraph{Anchor Regions and Partitions for Different $r$.}~
	
	\begin{figure}[H]
		\centering
		\begin{minipage}{0.32\linewidth}
			\centering
			\includegraphics[width = \linewidth]{Cluster4-Partition.jpg}
			$r = 4$
		\end{minipage}
		\hspace{1cm}    
		\begin{minipage}{0.32\linewidth}
			\centering
			\includegraphics[width = \linewidth]{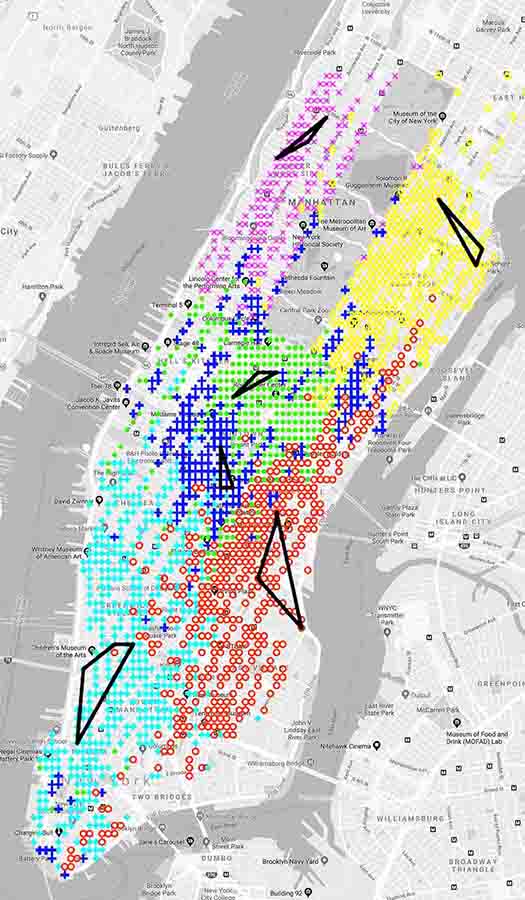}
			$r = 6$
		\end{minipage} \\ ~ \\ ~ \\
		\begin{minipage}{0.32\linewidth}
			\centering
			\includegraphics[width = \linewidth]{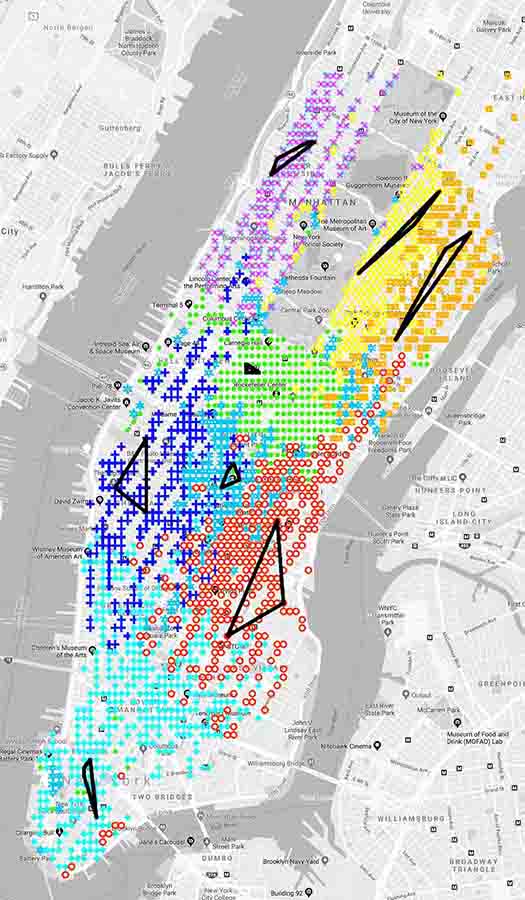}
			$r = 8$
		\end{minipage}
		\hspace{1cm}
		\begin{minipage}{0.32\linewidth}
			\centering
			\includegraphics[width = \linewidth]{Cluster10-Partition.jpg}
			$r = 10$
		\end{minipage}\\ \vspace{1cm}
		{Figure F.4: Partition of New York City by rounding the estimated disaggregation distributions to the closest vertices in the low-dimensional simplex.}\label{Partition}
	\end{figure}

	
	
	

\end{document}